\documentclass[twoside,11pt]{article}

\usepackage[preprint]{jmlr2e}

\usepackage{amsmath,amsfonts,bm,mathtools}
\usepackage[colorlinks=true, citecolor=teal, linkcolor=blue, urlcolor=cyan, allbordercolors={1 1 1}]{hyperref}
\usepackage{url}
\usepackage{subcaption}
\usepackage{algorithm}
\usepackage{algpseudocode}
\usepackage{float}
\usepackage{physics}
\usepackage{booktabs}

\def\gI{{\mathcal{I}}}
\def\gM{{\mathcal{M}}}
\def\gR{{\mathcal{R}}}
\def\gX{{\mathcal{X}}}
\def\gY{{\mathcal{Y}}}
\def\gZ{{\mathcal{Z}}}

\def\sE{{\mathbb{E}}} 
\def\sP{{\mathbb{P}}} 
\def\sR{{\mathbb{R}}} 
\def\vw{{\bm{w}}}
\def\bz{{\bm{z}}}

\def\bve{{\boldsymbol{\varepsilon}}}
\def\vz{{\varepsilon_z}}
\def\vy{{\varepsilon_y}}
\def\ve{{(\boldsymbol{\varepsilon})}}
\def\vez{{(\varepsilon_z)}}
\def\vey{{(\varepsilon_y)}}
\newcommand{\validb}{\mathcal{D}^{(b)}_{\mathrm{eval}}}

\newcommand{\trainb}{\mathcal{D}^{(b)}_{\mathrm{train}}}
\newcommand{\train}{\mathcal{D}_{\mathrm{train}}}
\newcommand{\valid}{\mathcal{D_{\mathrm{eval}}}}

\usepackage{xcolor}
\newcommand{\blue}[1]{\textcolor{black}{#1}}
\newcommand{\subfig}[4]{
  \begin{minipage}[b]{#4\textwidth}
    \centering
    \includegraphics[width=\textwidth]{#1}
    \text{(#2)} #3
  \end{minipage}
}

\usepackage{lastpage}
\jmlrheading{23}{2022}{1-\pageref{LastPage}}{1/21; Revised 5/22}{9/22}{21-0000}{Caihong Qin and Yang Bai}

\ShortHeadings{Classification Under LDP}{Qin and Bai}
\firstpageno{1}

\begin{document}

\title{Classification Under Local Differential Privacy with Model Reversal and Model Averaging}

\author{\name Caihong Qin \email caihqin@iu.edu \\
       \addr Department of Epidemiology and Biostatistics\\
       School of Public Health, Indiana University\\
       Indiana, USA
       \AND
       \name Yang Bai \email statbyang@mail.shufe.edu.cn \\
       \addr School of Statistics and Management\\
       Shanghai University of Finance and Economics\\
       Shanghai, China}

\editor{}

\maketitle

\begin{abstract}
Local differential privacy (LDP) has become a central topic in data privacy research, offering strong privacy guarantees by perturbing user data at the source and removing the need for a trusted curator. However, the noise introduced by LDP often significantly reduces data utility. To address this issue, we reinterpret private learning under LDP as a transfer learning problem, where the noisy data serve as the source domain and the unobserved clean data as the target. We propose novel techniques specifically designed for LDP to improve classification performance without compromising privacy: (1) a noised binary feedback-based evaluation mechanism for estimating dataset utility; (2) model reversal, which salvages underperforming classifiers by inverting their decision boundaries; and (3) model averaging, which assigns weights to multiple reversed classifiers based on their estimated utility. We provide theoretical excess risk bounds under LDP and demonstrate how our methods reduce this risk. Empirical results on both simulated and real-world datasets show substantial improvements in classification accuracy.
\end{abstract}


\section{Introduction}\label{sec:intro}
Data privacy has become a paramount concern as technological innovations continue to permeate daily life. Differential Privacy \citep[DP;][]{dwork2006calibrating} has emerged as a leading framework for privacy-preserving statistical analyses, providing a rigorous definition that limits the information an adversary can learn from released summaries. However, classical DP often assumes a {trusted data curator} who receives the raw data and applies privacy mechanisms. This assumption often fails in scenarios involving untrusted data collectors or highly sensitive information \citep{liu2022gdp}. 
Alternative privacy-preserving approaches such as cryptographic methods \citep{salami2023cryptographic} and Homomorphic Encryption \citep{marcolla2022survey} have been considered. While these methods can protect raw data through encryption and secure computation, they typically introduce high computational overhead and do not inherently address all privacy risks from published statistics \citep{gong2024practical,peng2019danger}. By contrast, Local Differential Privacy \citep[LDP;][]{kasiviswanathan2011can} addresses the problem at the data source. Under LDP, each individual perturbs their own data before sending any information to the collector, thereby removing the need for a trusted curator. This decentralized approach can significantly mitigate privacy leakage and is relatively scalable in real-world systems. Currently, multiple variants of LDP exist \citep{geumlek2019profile,erlingsson2020encode,meehan2022privacy,huang2022lemmas,qin2023survey}, with pure $\varepsilon$-LDP being the most widely used and well-established. Its formal definition is given as follows:
\begin{definition}[$\varepsilon$-LDP, \citealt{kasiviswanathan2011can}]
A randomized mechanism $\gM$ is said to satisfy $\varepsilon$-local differential privacy (LDP) if, for any pair of input values $u_1,u_2$ in the domain of $\gM$ and any possible output $v$, it holds that
\[
  \sP\bigl(\gM(u_1) = v \bigr)
  \;\le\;
  e^{\,\varepsilon}
  \;\sP\bigl(\gM(u_2) = v \bigr).
\]
\end{definition}

LDP has already been deployed by major technology companies such as Apple \citep{appleprivacy}, Google \citep{erlingsson2014rappor}, and Microsoft \citep{ding2017collecting}. While it offers stronger privacy protects against untrusted data collectors, LDP often adds more noise per user than standard DP, potentially harming data utility. Balancing the privacy-utility trade-off is a key challenge, particularly in private learning scenarios, which involve training machine learning models under LDP constraints \citep{yang2020local}. This contrasts with much of the existing LDP literature, which focuses on simpler statistical queries (e.g., frequency counts or sums) \citep{wang2017locally}.

Private learning under LDP faces two major difficulties \citep{ye2020local}. First, preserving {correlations} among features \citep{wang2019locally} and between features and labels \citep{yilmaz2020naive} is critical for accurate model training. The noise required by LDP can disrupt these relationships, reducing model accuracy. Second, when data are {high-dimensional}, one may split the privacy budget $\varepsilon$ across dimensions or randomly select a single dimension to report \citep{nguyen2016collecting,arcolezi2021random,arcolezi2022improving}. As the dimensionality $d$ grows, the effective privacy budget per dimension drops drastically, deteriorating utility even further. Meanwhile, omitting certain dimensions altogether may cause incomplete information for learning.

These challenges motivate our work. We focus on the classification problem, aiming to fully exploit data utility and enhance classifier performance while maintaining the same level of LDP protection. By better leveraging noisy data, we aim to boost predictive accuracy on the true (unperturbed) distribution without sacrificing privacy. Equivalently, for a desired performance target, stronger privacy guarantees can be offered to data owners.
We draw inspiration from {transfer learning} \citep{weiss2016survey,hosna2022transfer}, which typically uses {source} data to aid learning on a {target} domain. Here, the {noised} (LDP-perturbed) data play the role of source data, and the unobserved true data serve as the target. However, two key distinctions arise:
\begin{itemize}
    \item Unlike standard transfer learning, we lack direct observations from the target distribution, as only noised data are available under LDP.
    \item Negative transfer scenarios  \citep{li2022transfer} (where some source datasets could degrade performance) may be more frequent under LDP, because the noise injection can substantially distort the correlation information.
\end{itemize}
Although existing transfer learning techniques cannot be applied directly, we adapt several ideas to measure and exploit the transferability of noised datasets. Doing so enables us to assess dataset utility and design robust private-learning techniques. Below, we outline how our methods address the classic transfer learning questions \citep{pan2009survey}—{what} to transfer, {when} to transfer, and {how} to transfer—under LDP settings:
\begin{enumerate}
    \item What to transfer. Our goal is to learn the relationship between the response and covariates. One standard approach under LDP collects noised feature-label pairs from clients for model training. In addition to this, we introduce a novel LDP mechanism in which a subset of clients provides privatized binary evaluations of model performance. This mechanism preserves correlation structure and achieves higher utility by more effectively leveraging the underlying relationship between features and responses.
    
    \item When to transfer. We quantify the transferability, i.e., the utility of an LDP-perturbed dataset, to assess whether a source is beneficial or detrimental. Our proposed LDP-tailored evaluation process provides an unbiased estimate of this utility, enabling more effective identification of negative datasets or classifiers.
    
    \item How to transfer. We introduce two new techniques, guided by our utility evaluation, to improve performance of classifiers obtained under LDP:
    \begin{enumerate}
        \item Model Reversal (MR): When a noised classifier performs worse than random guessing, we reverse its decision boundary. This effectively corrects negative datasets without discarding them entirely.
        \item Model Averaging (MA): We adapt ensemble-learning ideas to the LDP context, assigning weight to each reversed classifier based on its estimated utility. The resulting averaged classifier tends to outperform any single classifier, especially in the presence of high noise.
    \end{enumerate}
\end{enumerate}
We address key challenges in private learning under LDP by introducing a novel framework inspired by transfer learning. Our main contributions are summarized below:
\begin{enumerate}
    \item \textbf{Linking LDP and Transfer Learning.} We recast private learning under LDP as a transfer learning problem, interpreting the noisy dataset as a source and the true data as a target. Building on \cite{qin2024adaptive}, we use their proposed definition of transferability to measure dataset utility under a given privacy budget.
    
    \item \textbf{Techniques for Private Learning.} We develop three {LDP-tailored} tools:
    \begin{enumerate}
        \item An evaluation scheme that requests privatized binary feedback on a classifier’s performance, based on which we provide unbiased accuracy estimates, effectively addressing the lack of target data under LDP.
        \item Model reversal, which salvages weak classifiers with below-50\% accuracy by inverting their decision boundaries, effectively leveraging negative datasets.
        \item Model averaging, which averages multiple reversed weak classifiers, weighting them by their estimated utility under LDP.
    \end{enumerate}
    
    \item \textbf{Theoretical Guarantees.} We provide excess risk bounds for our proposed methods, illustrating how noise affects classification performance under LDP and demonstrating how MR and MA help reduce this bound.
\end{enumerate}
Beyond these contributions, our proposed techniques are broadly applicable across various data structures and classification algorithms. The framework can be easily adapted to different privacy protection mechanisms by modifying the noise injection strategy. As a demonstration, we present a functional data classification example that, to our knowledge, represents the first LDP-based approach for modeling functional covariates.

In the following sections, we formalize these ideas and demonstrate their effectiveness both theoretically and empirically. Section~\ref{sec:refs} reviews related work. Section~\ref{sec:meth} introduces our utility evaluation process and the proposed techniques tailored for the LDP setting. Section~\ref{sec:thm} establishes excess risk bounds for the proposed classifiers. Section~\ref{sec:func} presents the complete process of building a functional classifier under LDP using our framework. Section~\ref{sec:exp} reports experimental results that demonstrate the effectiveness of the proposed techniques. Section~\ref{sec:real} presents real-data applications, and Section~\ref{sec:conc} concludes the paper.

\section{Related work}\label{sec:refs}
In this section, we review the literature on supervised learning under LDP, with a particular focus on classification methods. We then discuss key concepts from transfer learning that motivate our approach, followed by a brief overview of general classification techniques. Finally, we highlight relevant developments in functional data analysis that inform our application setting.

\textbf{Supervised Learning Under LDP.}
Given challenges arising from data correlations and high dimensionality, existing research on supervised learning under LDP is limited. Although a few studies have explored learning under LDP, preserving both privacy and utility remains a significant challenge. 
Some research focuses on empirical risk minimization (ERM), treating the learning process as an optimization problem solved through defining a series of objective functions. \cite{wang2019collecting} constructed a class of machine learning models under LDP that can be expressed as ERM, and solved by stochastic gradient descent (SGD). To address the high dimensionality issue, \cite{liu2020fedsel} privately selected the top-$k$ dimensions according to their contributions in each iteration of federated SGD. Deep learning models under LDP have also been studied; for example, \cite{sun2020ldp} employed adaptive data perturbation and parameter shuffling to mitigate privacy loss. Nevertheless, balancing privacy guarantees with competitive performance is still an open problem, particularly in large-scale, high-dimensional settings.

\textbf{Classification Under LDP.}
Currently, only a few classification algorithms have been developed under LDP. Among the most relevant are the works of \cite{yilmaz2020naive, berrett2019classification, ma2024optimal}.
\cite{yilmaz2020naive} proposed a naive Bayes classifier using LDP-based frequency and mean estimators. Their method assumes independence among features, which may limit its applicability in real-world settings. While they suggest dimensionality reduction techniques such as principal component analysis and discriminant component analysis for high-dimensional data, they do not detail how these can be implemented in the LDP context, where raw data are not directly accessible.
\cite{berrett2019classification} introduced a histogram-based classifier and established minimax rates for excess risk. Their approach applies Laplace noise with a scale parameter that grows exponentially with the number of features, potentially introducing substantial distortion, even in moderately low-dimensional settings, which may severely degrade classification performance.
\cite{ma2024optimal} considered the setting where additional public data are available and proposed an LDP decision tree classifier that achieves minimax-optimal rates under this setup. As discussed in Section~\ref{sec:intro}, leveraging public data effectively transforms the problem into a transfer learning task, aligning with the framework in \cite{cai2021transfer}. 
\blue{Their algorithm also remains implementable in the absence of public data, in which case the procedure reduces to a standard locally private tree classifier.}

\textbf{Transfer Learning.}
Transfer learning \citep{torrey2010transfer} enhances a target task by leveraging knowledge acquired from related source tasks. Its success depends on accurately quantifying the relationship between source and target domains to assess how informative the source is and to detect negative sources. Negative transfer occurs when a poorly matched source degrades target performance \citep{weiss2016survey}. The identification of informative or negative sources critically depends on the chosen metric for transferability.
In a series of parameter-based transfer learning works, it is common to measure the transferability of a source dataset by quantifying the distance between the parameters of interest in the source and target datasets; see  \cite{li2022transfer,lin2022transfer, tian2022transfer}. 
Yet, parameter-based measures alone can be misleading in classification settings, as smaller parameter distances do not necessarily imply higher transferability \citep{qin2024adaptive}. Alternative approaches assess the distributional similarity through a relative signal exponent \citep{cai2021transfer} or label-matching probabilities \citep{qin2024adaptive}.
In the context of LDP, available noised datasets, which serve as source data, are often prone to negative transfer because severe perturbations can distort key information and hinder reliable transferability assessment.

\textbf{Standard Classification Methods.}
Classification tasks appear across numerous real-world scenarios, ranging from medical diagnoses to spam detection. Classic methods include discriminant analysis \citep{hall2001functional}, Naïve Bayes classification \citep{dai2017optimal}, and logistic regression \citep{leng2006classification}, while contemporary research often employs more advanced approaches such as support vector machines \citep[SVM,][]{svm2008} or neural-network-based classifier \citep{saravanan2014review}. Comprehensive overviews of classification methods can be found in \cite{kotsiantis2007supervised,kumari2017machine}. From a theoretical perspective, many works investigate excess risk bounds to establish statistical guarantees. For instance, \cite{svm2008} derived non-asymptotic bounds for SVM-based classifiers, \cite{sang2022reproducing} developed similar results for distance-weighted discrimination (DWD), and \cite{ko2023excessriskconvergencerates} studied the excess risk convergence rates of neural network-based classifiers. In transfer learning contexts, \cite{cai2021transfer} and \cite{qin2024adaptive} analyzed the excess risk of transfer-learning-based classifiers. 
However, results on classification under LDP remain limited, particularly in understanding how noise injection impacts both model training and theoretical guarantees.

\textbf{Functional Data Analysis.}
Functional data, such as curves or surfaces defined over continuous domains, are inherently infinite-dimensional \citep{ramsay2005,horvath2012inference}. Common examples include time-continuous measurements collected by wearable devices, which are often associated with sensitive health information. To facilitate statistical analysis, recent research has focused on functional projection techniques that map functional observations into finite-dimensional representations \citep{horvath2015,pomann2016,kraus2019classification,Bai2023}. 
Several differential privacy methods have been developed for functional data \citep{hall2013differential,mirshani2019formal,jiang2023functional,lin2023differentially}, but most are based on the global DP setting, where a trusted data curator has access to the entire dataset. In contrast, LDP requires that data be perturbed at the individual level prior to collection, introducing unique challenges for functional data. As a result, training predictive models on functional data under LDP remains an open and underexplored problem, leaving a gap in privacy-preserving techniques for infinite-dimensional data.

\section{Privacy-Preserving Classification under LDP}\label{sec:meth}
Under local differential privacy, observed data are perturbed to protect sensitive information, which often leads to a loss in utility. In this section, we introduce a utility measure to evaluate the quality of perturbed data and propose novel techniques to improve classifier performance under LDP without compromising privacy.

\subsection{Data and Notation}\label{ssec:nota}
In practice, original data collected from clients can be of diverse forms, including high-dimensional data \citep{araveeporn2021higher}, functional data \citep{kraus2019classification}, images \citep{rawat2017deep}, and text \citep{blitzer2006domain}. To derive usable features from these data under privacy constraints, one often applies suitable representation methods and then adds random noise to the resulting features or labels. In this subsection, we detail the representation process and introduce the notations used throughout.

\blue{Let $\gX$ be an instance set, i.e., the space of original inputs (e.g., the diverse data types mentioned above). Let $\gZ$ be a feature space, with $\gZ=\sR^d$ as a common example, where elements in $\gZ$ correspond to finite-dimensional representations of the original input. Let $\gY=\{-1, 1\}$ denote the label set for binary classification.}
A mapping function is characterized by a distribution over $\gX$ and a (stochastic) mapping function $\zeta:\gX\to\gY$. For an instance $X \in \gX$, the value $\zeta(X)$ corresponds to the probability that its label $Y$ is $1$. Let $Q$ denote the joint distribution of $\{X,Y\}$, and let $Q_X$ be the marginal distribution of $X$. We write $Q_{Y|X}(x) = \zeta(x)$ to represent the conditional distribution $\sP(Y=1\mid X=x)$.

A representation function $\gR: \gX \to \gZ$ maps instances to features, thereby inducing a distribution over $\gZ$ and a (stochastic) mapping function $\eta:\gZ\to\gY$. Specifically, for any Borel set $B\subseteq\gZ$ such that $\gR^{-1}(B)$ is $Q_X$-measurable,
\begin{equation*}
    \sP(Z\in B) \;=\; Q_X\bigl(\gR^{-1}(B)\bigr).
\end{equation*}
Let $P$ denote the joint distribution of $\{Z,Y\}$, and let $P_Z$ be the marginal distribution of $Z$. 
We denote by \(\eta(z)\triangleq \sP\bigl(Y=1 \mid Z=z\bigr)\) the conditional distribution of $Y$ given $Z$, where
\begin{equation*}
    \eta(Z) \;=\; \sE_{Q_X}\bigl[\zeta(X) \mid \gR(X) = Z \bigr].
\end{equation*}
Note that $\eta(Z)$ may be a stochastic function even if $\zeta(X)$ is not. This is because the function $\gR$ can map two instances with different labels to the same feature representation \citep{ben2006analysis}.

Next, consider an $\varepsilon$-LDP mechanism that transforms $(Z,Y)$ into $(Z^{\vez},Y^{\vey})$, where $\bve = (\vz,\vy)$ denotes the privacy budget allocation with $\vz,\vy > 0$. When both $Z$ and $Y$ require protection, the budget is split such that $\vz + \vy = \varepsilon$. Notable cases include protecting only covariates ($\vz = \varepsilon$, $\vy = \infty$, meaning no privacy constraint on $Y$), or only labels ($\vy = \varepsilon$, $\vz = \infty$), where the summation constraint does not apply as the entire budget is assigned to one component. One way to achieve $\varepsilon$-LDP is by setting:
\begin{align}\label{eq:eLDP}
\begin{split}
    Z^{\vez} &= Z \;+\; \delta^{\vez},\\
    \sP\bigl(Y^{\vey} = Y\bigr) &= q^{\vey},
\end{split}
\end{align}
\blue{where the noise $\delta^{\vez}$ follows a distribution $D^{\vez}$ whose scale is calibrated to the sensitivity of $Z$ under the chosen $\vz$-LDP mechanism, and smaller $\vz$ corresponds to a larger noise scale. For example, if $Z\in[a,b]^d$, then under the Laplace mechanism \citep{dwork2006calibrating} one can add i.i.d.\ $\mathrm{Laplace}(0,\lambda_z)$ noise to each coordinate with $\lambda_z=d(b-a)/\vz$.}
A common choice for $q^{\vey}$ is $e^{\vy}/(1 + e^{\vy}) \in [1/2, 1)$, corresponding to the classical randomized response mechanism \citep{warner1965randomized}. This added noise weakens the dependency between $Z$ and $Y$, thereby degrading data utility. In Section~\ref{ssec:utlt}, we show how utility is affected by the allocation of the privacy budget.

Parallel to the original distribution $P$, let $P^\ve$ be the joint distribution of $(Z^{\vez},Y^{\vey})$. We denote its marginal distribution by $P_Z^\ve$ and the conditional distribution by $P^\ve_{Y|Z}(z) = \sP(Y^\vey=1 \mid Z^\vez = z) := \eta^\ve(z)$.
Under LDP, we only observe the noised data $(Z^{\vez},Y^{\vey})$, which defines the source distribution $P^\ve$. Our goal, however, is to learn a model that performs well on the unobserved target distribution $P$ (and ultimately $Q$), in order to gain insights into the real world. This discrepancy between the source and target distributions highlights a fundamental challenge in the LDP setting.

In this work, we focus on developing techniques that improve classification performance under a given LDP mechanism. The methods proposed in the following sections can be adapted to different variants of LDP \citep{qin2023survey} by appropriately modifying the noise injection procedure in Equation~\eqref{eq:eLDP}. For concreteness, we use the standard $\varepsilon$-LDP framework to illustrate our approach throughout this paper.

\subsection{Measure of Utility}\label{ssec:utlt}
In this subsection, we introduce a utility measure for the perturbed dataset $(Z^{\vez},Y^{\vey})$ by adopting the concept of dataset transferability proposed by \citet{qin2024adaptive}. 
\blue{For any feature point $z_0 \in \gZ$, we define the pointwise utility function using two independent draws: $(Y,Z)\sim P$ from the original distribution and $(Y^{\vey},Z^{\vez})\sim P^{\ve}$ from the privatized distribution. Specifically,
\begin{equation}\label{eq:gz}
    g^{\ve}(z_0)
    \;=\;
    \sP\!\bigl(Y^{\vey} = Y \,\big|\, Z = z_0,\ Z^{\vez} = z_0\bigr)
    \in [0,1],
\end{equation}
which represents the probability that the label from the original distribution with $Z=z_0$ agrees with the label from the privatized distribution with $Z^{\vez}=z_0$.}
Proposition \ref{prop:gfun} below clarifies how $g^{\ve}(z_0)$ relates to the mapping functions $\eta$ and $\eta^{\ve}$, and how it is affected by the privacy budget allocation $(\vz, \vy)$.

\begin{proposition}\label{prop:gfun}
With the mapping functions $\eta$ and $\eta^{\ve}$ defined above, we have
\begin{equation}\label{eq:gfun}
    g^{\ve}(z_0) \;=\; 2\Bigl(\eta(z_0) - \tfrac12\Bigr)\Bigl(\eta^{\ve}(z_0) - \tfrac12\Bigr)\;+\;\tfrac12,
\end{equation}
where
\begin{equation}\label{eq:etainfo}
    \eta^{\ve}(z_0) - \tfrac12
    \;=\;
    2\Bigl(q^{\vey} - \tfrac12\Bigr)\;
    \sE_{P_Z}\!\Bigl[\omega\bigl(z \mid z_0,\vz\bigr)\,\Bigl(\eta(z)-\tfrac12\Bigr)\Bigr]
\end{equation}
and
\begin{equation}\label{eq:wz}
    \omega\bigl(z \mid z_0,\vz\bigr)
    \;=\;
    \frac{D^{\vez}\bigl(\delta_z^{\vez} = z_0 - z\bigr)}{\sE_{P_Z}\!\Bigl\{D^{\vez}\bigl(\delta_z^{\vez}=z_0 - z\bigr)\Bigr\}}.
\end{equation}
\end{proposition}

In Proposition \ref{prop:gfun}, $\omega\bigl(z \mid z_0,\vz\bigr)$ is analogous to a continuous weighting scheme, satisfying $\sE_{P_Z}\bigl(\omega(z\mid z_0,\vz)\bigr)=1$. As discussed in \cite{qin2024adaptive}, Equation \eqref{eq:gfun} indicates that $g^{\ve}(z_0)$ depends on how $\eta$ and $\eta^{\ve}$ deviate from $1/2$. When they deviate in the same direction, $g^{\ve}(z_0)\ge 1/2$, making the noised dataset informative; otherwise, it can be {negative}. 
Equation~\eqref{eq:etainfo} illustrates how the privacy budget parameters $(\vz,\vy)$ jointly influence $g^{\ve}(z_0)$. To gain clearer insight into their roles, we consider the following two notable cases:
\begin{enumerate}
    \item Protecting only $Y$: Let $\vy = \varepsilon$ and $\vz=\infty$. We then have $P_Z^\ve = P_Z$ but $P^\ve_{Y|Z}\neq P_{Y|Z}$, which matches the classic {posterior drift} scenario \citep{qin2024adaptive}. In this setting,
    \[
        \eta^{\ve}(z_0)-\tfrac12 \;=\; 2\Bigl(q^{\vey}-\tfrac12\Bigr)\Bigl(\eta(z_0)-\tfrac12\Bigr),
    \]
    and the dataset remains informative when $q^{\vey}\ge 1/2$, meaning the perturbed data and the original data share the same Bayes classifier; otherwise, it becomes negative.
    
    \item Protecting only $Z$: Let $\vz = \varepsilon$ and $q^{\vey}=1$. Then $\eta^\ve(z_0)-1/2$ primarily depends on the weight $\omega\bigl(z\mid z_0,\vz\bigr)$ defined in Equation \eqref{eq:wz}. 
    To examine the impact of $\vz$ on this weight, we consider a special case where \( z \sim U(-1, 1) \), the uniform distribution on the interval \((-1, 1)\). The noise \( \delta_z^{\vez} \) is drawn from a Laplace distribution with mean zero and scale parameter \( 2/\vz \), which satisfies \( \varepsilon \)-LDP.
    We then compute the weight for \( z_0 \in [-2,2] \) with \( \vz \in \{10, 1, 0.1, 0.01\} \).
    Figure~\ref{fig:wz} illustrates how the weight function \( \omega\bigl(z \mid z_0, \vz\bigr) \) varies with \( z \) for different \( z_0 \) and \( \vz \), where the range of weight becomes increasingly concentrated around $1$ as \( \vz \) approaches 0.
    Given that $\int_z \omega\bigl(z \mid z_0, \vz\bigr) dz=1$,
    these heatmaps indicate that as $\vz$ decreases (i.e., stronger privacy constraints and more noise), the weight function becomes highly diffuse, spreading nearly uniformly over $z \in [-1,1]$ rather than concentrated around the diagonal $z = z_0$ for $z_0 \in [-1,1]$ or at the boundaries $z = -1$ and $z = 1$ when $z_0 < -1$ and $z_0 > 1$, respectively.
    Consequently, as privacy constraints tighten and the weight function grows more diffuse, the deviation between $\eta^{(\varepsilon)}(z) - \tfrac{1}{2}$ and $\eta(z) - \tfrac{1}{2}$ becomes larger.
\end{enumerate}

\begin{figure}[htbp]
    \centering
    \includegraphics[width=1\textwidth]{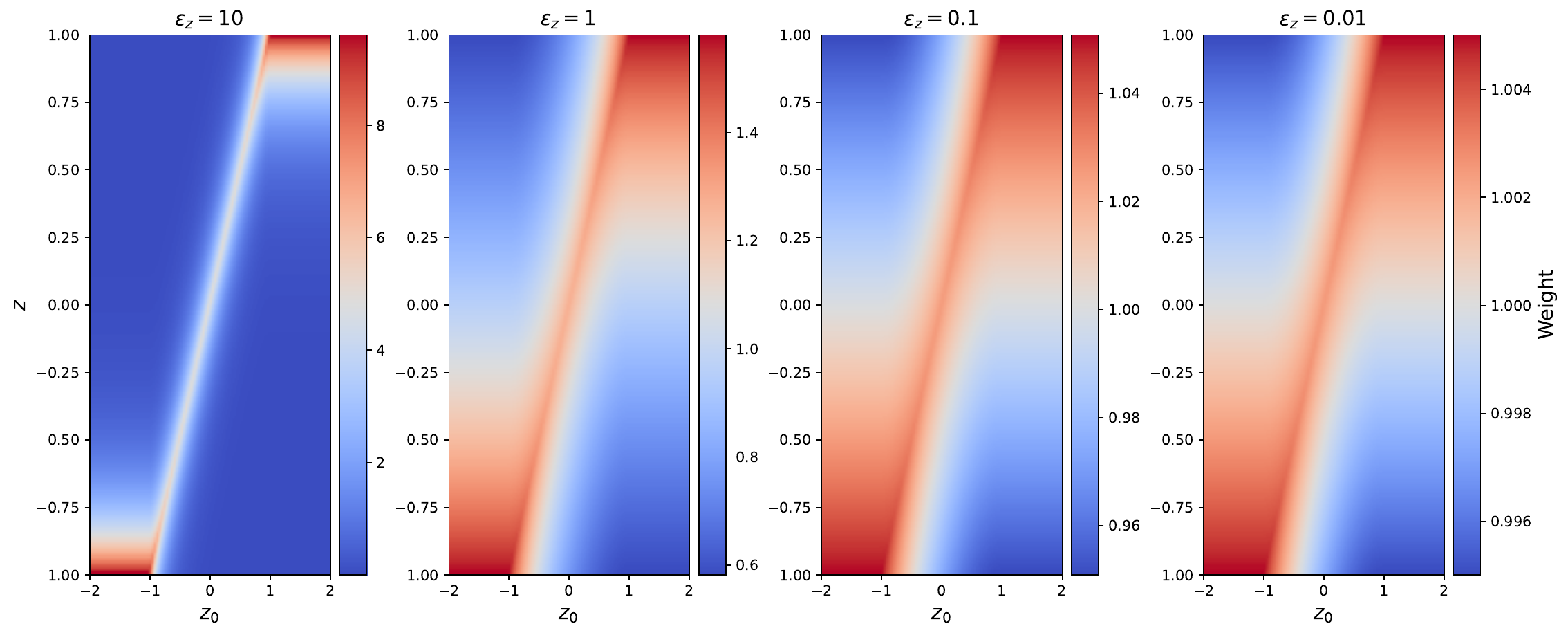}
    \caption{Heatmaps of the weight function $\omega\bigl(z \mid z_0, \vz\bigr)$ changing with $z$ for various values of $z_0$ and $\vz$, where we assume $z \sim U(-1,1)$ and $z_0 = z + \delta_z^{\vez}$ with $\delta_z^{\vez}$ drawn from a Laplace distribution having mean zero and scale $2/\vz$. As \(\vz\) decreases, the value intervals of the color bars become increasingly concentrated around 1.}
    \label{fig:wz}
\end{figure}

\subsection{Evaluation of Utility}\label{ssec:eval}
Utility assessment is crucial when integrating information from multiple data sources. Unlike in transfer learning, where one can directly evaluate source-trained models on target data, the LDP setting lacks access to unperturbed target data. To address this, we propose a novel alternative approach that requires only a noised binary evaluation from each client. Based on these privatized responses, we provide an unbiased estimate of the utility of a model trained on perturbed data.

Given a perturbed dataset $\mathcal{D}_0=\{(z_i^{\vez},y_i^{\vey})\}_{i=1}^{n_0}$, our goal is to estimate its utility $g^\ve(z_0)$, as defined in Equation \eqref{eq:gz}, for any given $z_0$ drawn from the distribution $P_Z$. A key challenge is that both $y(z_0)$ and $y^\vey(z_0)$ are unobservable. To overcome this, instead of requesting noisy feature-label pairs from each client, we apply the classifier $f^\ve(z_0)$, trained on $\mathcal{D}_0$, to each client in evaluation set and request a privatized binary response indicating whether $y(z_0)=\hat y^\vey(z_0)$, where $\hat y^\vey(z_0)=\text{sign}(f^\ve(z_0))$. This approach enables a more efficient form of querying, as it preserves the correlation between the response and the functional covariate in the target distribution while injecting significantly less noise, since only a single binary response is privatized per client. Furthermore, measuring the utility of $\mathcal{D}_0$ in this way is equivalent to estimating the classification accuracy of $f^\ve(z_0)$ trained on it. This interpretation aligns with Proposition~\ref{prop:gfun}, which indicates that higher dataset utility corresponds to better Bayes classification performance achievable from that dataset.

In what follows, we present the detailed procedures for classifier training and dataset utility evaluation. These involve two distinct types of client-side information access: collecting noised observation pairs and obtaining privatized binary evaluations. We assume a total of $N$ clients, partitioned into a training set $\train$ of size $N_0$ and an evaluation set $\valid$ of size $N_1$, where $N_0 + N_1 = N$.
\begin{enumerate}
    \item \textbf{Classifier Training.} The server collects perturbed data $\{(z_i^{\vez},y_i^{\vey})\}_{i \in \train}$ from the training clients to fit or construct multiple weak classifiers. Specifically, for each $b=1,\dots,B$, we randomly sample $n_0 (<N_0)$ pairs from $\train$, denoted as $\trainb$. Using these samples, we construct $B$ weak classifiers $\{f^{(b)}\}_{b=1}^B$. (Here and after, we ignore the $\bve$ in $f^{(b)}$ for simple representation.) Each client only uploads once perturbed observation, thus preserving their privacy guarantee.
    \item \textbf{Utility Evaluation.} For the $N_1$ clients in the evaluation set, we do not collect their perturbed observations. Instead, each evaluation client provides a noised binary indicator of misclassification, which is used to evaluate a classifier’s performance. Specifically, we split the evaluation set into $B$ subsets: $\validb$, $b=1,\dots,B$. For a client $i \in \validb$ with true observation $(x_i,y_i)$, we compute
    \begin{equation}\label{eq:yeval}
        \hat{y}_i := I\bigl[f^{(b)}\{\gR(x_i)\} > 0\bigr]\quad \text{and} \quad r_i := I(\hat{y}_i = y_i).
    \end{equation}
    Each client then reports a privatized value \( r_i' \) using randomized response, where
    \begin{equation}\label{eq:rprime}
    \sP(r_i' = r_i) = q = \frac{e^{\varepsilon_v}}{1 + e^{\varepsilon_v}}\quad \text{and} \quad \sP(r_i' \ne r_i) = 1 - q.
    \end{equation}
\end{enumerate}
Theorem \ref{thm:ldp2} shows that this evaluation mechanism satisfies $\varepsilon$-LDP when $\varepsilon_v=\varepsilon$, and provides an unbiased estimate of the accuracy of the classifier $f^{(b)}$. This estimate can be taken as an evaluation of the utility of the dataset $\trainb$, for $b=1,\dots,B$.

\begin{theorem}\label{thm:ldp2}
    Let $\gM$ be the privacy mechanism mapping $(x,y)$ to $r'$ as described above. 
    Then $\gM$ satisfies $\varepsilon_v$-LDP. Furthermore, let $r^{(b)}$ be the true classification accuracy of $f^{(b)}$ under the target distribution $P$, and let $n_1^{(b)} = |\validb|$. Define
    \begin{equation}\label{eq:r}
        \tilde r^{(b)} \;=\;\frac{\hat r^{(b)} + q - 1}{2q - 1}, 
        \quad\text{where}\quad
        \hat r^{(b)} \;=\;\frac{\sum_{i\in\validb} r'_i}{n_1^{(b)}}.
    \end{equation}
    Then $\sE\bigl(\tilde r^{(b)}\bigr)=r^{(b)}$, and 
    \(\mathrm{Var}\bigl(\tilde r^{(b)}\bigr) \;\le\; \bigl(\tfrac{e^{\varepsilon_v}+1}{e^{\varepsilon_v}-1}\bigr)^{2}(4\,n_1^{(b)})^{-1}.\)
\end{theorem}

Theorem~\ref{thm:ldp2} indicates that as $n_1^{(b)}$ increases, $\tilde{r}^{(b)}$ converges more closely to the true classification accuracy of $f^{(b)}$ under LDP. Consequently, enlarging the evaluation set substantially enhances the precision of accuracy estimates and assists in selecting the better-performing weak classifiers. Since heavily perturbed data may limit the gains from training on more samples, these findings suggest allocating a greater proportion of clients to evaluation rather than training. Simulation results in Appendix~\ref{appsim:size} provide evidence supporting this strategy. In addition, during the evaluation step, each client in \( \validb \) evaluates only a single classifier, ensuring that their privacy budget remains \( \varepsilon \) by setting \( \varepsilon_v = \varepsilon \). If a client were to evaluate multiple classifiers, the budget would need to be divided among them, potentially reducing the precision of each evaluation. Specifically, when evaluating \( k \) classifiers, the privacy budget must be allocated as \( \varepsilon_v = \varepsilon / k \) for each evaluation, leading to an increase in \( \mathrm{Var}(\tilde{r}^{(b)}) \) as \( k \) grows.

\subsection{Model Reversal and Model Averaging}\label{ssec:mrma}
Under LDP, the added noise can be substantial, causing the weak classifiers \( \{f^{(b)}\}_{b=1}^B \) to perform poorly. As shown in Proposition~\ref{prop:gfun}, it is possible for \( g^{\ve}(z) \) to fall below \( 1/2 \), indicating a negative dataset that performs worse than random guessing. We propose model reversal and model averaging to extract useful information from such seemingly uninformative data and improve classification performance. The implementation of these techniques builds on the unbiased utility evaluation established in Theorem \ref{thm:ldp2}.

\textbf{Model Reversal.} We first propose a novel procedure to enhance the performance of each weak classifier. Formally, for a weak classifier \( f^{(b)} \), we define its reversed classifier as
\begin{equation}\label{eq:mr}
   f^{*(b)} =
   \begin{cases}
       -f^{(b)}, & \text{if its estimated accuracy } \tilde{r}^{(b)} < 0.5,\\
       f^{(b)}, & \text{otherwise}.
   \end{cases}
\end{equation}
The model reversal process defined in Equation \eqref{eq:mr} means that for a given classifier $f^{(b)}$, its sign is reversed if its estimated accuracy $\tilde r^{(b)} < 0.5$; otherwise, it remains unchanged. Consequently, the estimated accuracy of the reversed classifier $f^{*(b)}$ becomes $\tilde r^{*(b)} = \max\{\tilde r^{(b)}, 1 - \tilde r^{(b)}\}$, ensuring it is larger than $0.5$.
This idea exploits the fact that classification primarily depends on the direction of the coefficient vector relative to the decision boundary, rather than its magnitude. By identifying whether a classifier is less than random guessing, we can flip its decision boundary to ensure it surpasses $50\%$ accuracy.
\blue{Here we focus on binary classification, where model reversal reduces to a simple ``flip or keep'' decision. Extending the MR idea to multi-class problems is an interesting direction for future work. For example, one could reduce a $K$-class problem to $K$ one-vs-rest binary subproblems, or apply permutation-based corrections to the predicted class labels of each weak classifier, guided by confusion matrices estimated from suitably designed LDP-protected queries in the evaluation step.}

\textbf{Model Averaging.} Next, given $B$ reversed weak classifiers $\{f^{*(b)}\}_{b=1}^B$ with estimated accuracies $\{\tilde r^{*(b)}\}_{b=1}^B$, we combine them by weighted averaging:
\begin{equation}\label{eq:wsingle}
    w_b \;=\; \frac{\max\{\tilde r^{*(b)} - r_0,\,0\}}{\sum_{j=1}^B \max\{\tilde r^{*(j)} - r_0,\,0\}},
\end{equation}
where $r_0 \in (0.5, 1)$ is a server-specified cutoff. Classifiers with $\tilde r^{*(b)} \le r_0$ receive weight $0$, and the final averaged classifier is
\begin{equation}\label{eq:fsingle}
    f^\dagger(\cdot) \;=\;\sum_{b=1}^B w_b\,f^{*(b)}(\cdot).
\end{equation}
\blue{In Equation~\eqref{eq:fsingle}, we aggregate reversed weak classifiers by a weighted average.
For parametric classifiers where $f^{*(b)}(\cdot)$ is linear in the model parameters, Equation~\eqref{eq:fsingle} is equivalent to averaging the corresponding parameter vectors.
More generally, when the weak classifiers are nonlinear in their parameters or do not share a common parameterization, the same utility-based weights $\{w_b\}$ from Equation~\eqref{eq:wsingle} can be combined with alternative aggregation rules, such as weighted voting or averaging predicted probabilities or decision scores.}

The effectiveness of MR and MA relies on accurate utility estimation. Theorems~\ref{thm:ermr} and \ref{thm:erma} show that these techniques improve excess risk bounds. Our experiments further confirm that MR and MA significantly enhance classification accuracy.

\subsection{Single-Server Classification with MRMA}\label{ssec:single}

This subsection integrates the previously introduced components, privacy-preserving data collection, utility evaluation, model reversal, and model averaging, into a unified procedure for classifier construction on a single server. 
The full process is presented in Algorithm~\ref{alg:single}, which takes as input the randomized mechanism, privacy budget, and learning parameters, and returns a final classifier trained under LDP.

\begin{algorithm}[ht]
\caption{Single-Server Classification with MRMA}\label{alg:single}
\begin{algorithmic}[1]
\Procedure{Server}{$\gR,N_0,N_1,n_0,B,\varepsilon,\vz,\vy,r_0$}

\State Divide clients into a training set $\train$ and an evaluation set $\valid$ with $|\train|=N_0, |\valid|=N_1$, and spilt the evaluation set into $B$ subsets, $\{\validb\}_{b=1}^B$.

\For{each client in the training set}
    \State Obtain \( (z_i^{\vez}, y_i^{\vey}) \) by Equation~\eqref{eq:eLDP}. \Comment{privacy-preserving data collection}
\EndFor

\For{$b=1,\dots,B$}
    \State Randomly draw $n_0$ samples from $\train$ without replacement and denote as $\trainb$.
    \State Build a classifier $f^{(b)}$ based on $\{(z_i^\vez,y_i^\vey)\}_{i\in\trainb}$. \Comment{server-specified classifier}
    \For{each client in the evaluation set $\validb$}
        \State Compute $r_i'$ by Equations \eqref{eq:yeval} and \eqref{eq:rprime}.
    \EndFor
    \State Estimate $\tilde r^{(b)}$ by Equation \eqref{eq:r}. \Comment{utility evaluation}
    \State Compute $f^{*(b)}$ by Equation \eqref{eq:mr}. \Comment{model reversal}
\EndFor

\State Estimate $\vw=(w_1,\dots,w_B)^{\top}$ by Equation \eqref{eq:wsingle} with cutoff value $r_0$.

\State \Return The final estimated classifier $f^\dagger$ by Equation \eqref{eq:fsingle}. \Comment{model averaging}
\EndProcedure
\end{algorithmic}
\end{algorithm}

\textbf{Sample Size Balancing.} 
In practice, when the total sample size \( N \) is large, increasing the training size \( N_0 \), the evaluation size \( N_1 \), and the number of classifiers \( B \) can improve performance. However, for a fixed \( N \), it is important to carefully choose \( N_0,N_1 \) and \( B \). Larger values of \( N_0 \) and \( B \) help train more effective weak classifiers, while smaller values ensure sufficient data in each evaluation subset. Section~\ref{sec:thm} provides theoretical guidance on setting these parameters. We also empirically study their impact in Appendix~\ref{appsim:paras} and~\ref{appsim:size}, offering practical recommendations.

\textbf{Broad Applicability.}
While this work focuses on a specific algorithmic structure, the proposed evaluation process and MRMA techniques are broadly applicable. Regarding data types, our framework accommodates vectors, functional data, images, or text, depending on the representation function and classifier specified at line 8 in Algorithm~\ref{alg:single}. Regarding privacy mechanisms, MRMA is compatible with various LDP variants by adjusting how noise is added in the training and evaluation phases, as implemented in lines 4 and 10 of Algorithm~\ref{alg:single}, respectively. This flexibility enables the method to adapt to a wide range of practical applications with different data modalities and privacy requirements.

\textbf{Online data.} 
Although our primary focus is on protecting static information such as individual profiles, many real-world systems produce continuous data streams from sensors or digital interactions. Our proposed procedure naturally extends to such online settings. On one hand, accuracy estimates for existing weak classifiers can be updated incrementally, allowing dynamic adjustment of their weights in the aggregated model. On the other hand, if performance deteriorates on new data, additional weak classifiers can be trained on newly collected, privatized samples. As a result, the final classifier can adapt over time by incorporating fresh models and reweighting existing ones, effectively tracking changes in the underlying data distribution.

\subsection{Multi-Server Classification with MRMA under Heterogeneity}\label{ssec:multi}
In many real-world scenarios, data are distributed across multiple servers. However, these servers may exhibit varying degrees of heterogeneity in their local data, which existing methods do not always handle effectively. Our proposed techniques extend naturally to the multi-server setting under LDP. 
Assume there are \( K \) servers, each with a locally trained classifier \( f_k(x) \). These servers share their models with one another. Each server can then re-evaluate the shared classifiers and apply MRMA, as described in Sections~\ref{ssec:eval} and \ref{ssec:mrma}, treating classifiers from different servers as a collection of weak classifiers. Details of model evaluation and MRMA in this setting are provided in the following two paragraphs.

\textbf{Perturbation and Evaluation.} A server with sufficient clients can split its evaluation set into $B+K$ parts, using $B$ parts for evaluating its internal weak classifiers and $K$ parts for evaluating the externally received classifiers $\{f_k\}_{k=1}^K$. Each evaluation is performed via randomized response at privacy level $\varepsilon$, as in Section~\ref{ssec:eval}. When data are limited, one can sequentially evaluate internal and external classifiers, dividing the privacy budget accordingly (by splitting $\varepsilon$ appropriately).

\textbf{MRMA.}
Upon receiving performance assessments for all $K$ classifiers from multiple servers, each server can independently apply MR and MA with a cutoff $r_0^*$ (see Equation \eqref{eq:wsingle} in Section~\ref{ssec:mrma}) and obtain the final classifier $f_k^{\dagger}$ for the $k$-th server. Because the servers and their data distributions may be highly heterogeneous, employing our MRMA approach can mitigate negative transfer effects by assigning small or zero weights to less relevant models. Overall, by integrating LDP with MR and MA, each server can capitalize on the shared knowledge from others, achieving improved classification performance while still honoring the local privacy constraints.

\blue{Our discussion of the multi-server setting focuses on how locally trained $\varepsilon$-LDP classifiers can be re-evaluated and then combined via MRMA to improve each server's final decision rule.
In practice, distributed learning across multiple servers can involve additional considerations that are widely studied in the federated learning literature \citep{truex2020ldp,sun2020ldp}.
Our evaluation and MRMA procedures can be used in such settings, as they only require a collection of candidate $\varepsilon$-LDP classifiers and their utility estimates.}

\section{Theoretical Guarantees for Classification under LDP}\label{sec:thm}
In this section, we analyze how the LDP mechanism affects classification performance and demonstrate the effectiveness of our proposed techniques through their improvements in excess risk bounds. We begin with necessary definitions for Bayes classifiers and proceed to establish our theoretical guarantees.

Recall that $\eta(z)$ denotes the conditional probability function, and the corresponding Bayes classifier under the original (unperturbed) distribution $P$ is given by $h^*(z) = \text{sign}\{\eta(z) \ge 1/2\}$. The expected classification accuracy of $h^*$ is $\frac{1}{2} + \sE_{P_Z}(|\eta(z) - \frac{1}{2}|)$. 
We measure the performance of a classifier $h(z)$ under $P$ relative to $h^*(z)$ using the excess risk, defined as:
\begin{equation*}
    L_{P}(h,h^*) \;=\; \sE_{P}\bigl[\ell(h) \;-\; \ell(h^*)\bigr],
\end{equation*}
where $\ell(h) = I\{y\,h(z)\le 0\}$ is the $0$--$1$ loss and $I\{\cdot\}$ denotes the indicator function. Similarly, for the perturbed distribution $P^{\ve}$ with conditional probability function $\eta^{\ve}(z)$, the corresponding Bayes classifier is $h^{*\ve}(z)=\text{sign}\{\eta^{\ve}(z)\ge1/2\}$. The excess risk of a classifier $h(z)$ under $P^{\ve}$ is defined as:
\begin{equation*}
    L_{P^{\ve}}\!\bigl(h,h^{*\ve}\bigr)
    \;=\;
    \sE_{P^{\ve}}\bigl[\ell(h) \;-\; \ell(h^{*\ve})\bigr].
\end{equation*}
Let $f^{(b)}(z)$ be a weak classifier trained from the noised data set $\{(z^{\vez}_i,y^{\vey}_i)\}_{i=1}^{n_0}$. The next theorem bounds the excess risk of $f^{(b)}(z)$ under the unperturbed distribution $P$.

\begin{theorem}\label{thm:erini}
    Let $\gR$ be a fixed representation function from $\gX$ to $\gZ$. Then for any weak classifier $f^{(b)}$ trained with $n_0$ samples from $P^{\ve}$,
    \begin{align*}
        L_P\bigl(f^{(b)},h^*\bigr)
        \;\le\;&
        L_{P^{\ve}}\!\bigl(f^{(b)},h^{*\ve}\bigr)
        \;+\;
        \frac{3}{4}\,d_{\mathrm{TV}}(P_Z,P_Z^{\ve})
        \;+\;
        \sE_{P_Z}\Bigl[\bigl|\eta(z)-\eta^{\ve}(z)\bigr|\Bigr]\\
        &\;+\;
        \sE_{P_Z}\bigl[\bigl|\eta(z)-\tfrac12\bigr| \;-\;\bigl|\eta^{\ve}(z)-\tfrac12\bigr|\bigr],
    \end{align*}
    where 
    \[
        d_{\mathrm{TV}}(P_Z,P_Z^{\ve})
        \;=\;
        2\,\sup_{B\subseteq \gZ}\Bigl|P_Z(B)-P_Z^{\ve}(B)\Bigr|
    \]
    is the total variation distance between $P_Z$ and $P_Z^{\ve}$.
\end{theorem}

In Theorem \ref{thm:erini}, the term $L_{P^{\ve}}\!\bigl(f^{(b)},h^{*\ve}\bigr)$ captures the excess risk of $f^{(b)}$ under the privatized distribution $P^{\ve}$. 
Since $f^{(b)}$ is trained on data from $P^{\ve}$, this term reflects the inherent difficulty of the learning task under $P^{\ve}$
\blue{and depends on the representation function $\gR$, the choice of classification algorithm, and the training sample size $n_0$;}
see related theoretical studies in \cite{svm2008,sang2022reproducing,ko2023excessriskconvergencerates}. 
\blue{The following three terms quantify discrepancies between the original and privatized distributions and do not depend on the training sample size~$n_0$.} 
\blue{The total variation distance term $d_{\mathrm{TV}}(P_Z,P_Z^{\ve})$ measures the discrepancy between the marginal distributions of $Z$ and $Z^{\vez}$. Similar excess risk bounds involving a total variation distance term can be found in \cite{ben2006analysis}. To further illustrate the magnitude of this term, Appendix~\ref{appsim:dtv_visualization} provides a numerical illustration showing how $d_{\mathrm{TV}}(P_Z,P_Z^{\ve})$ varies with the dimension $d$ and the privacy level~$\varepsilon_z$.}
The term $\sE_{P_Z}\bigl[|\eta(z)-\eta^{\ve}(z)|\bigr]$ quantifies how much the conditional distributions drift from $\eta$ to $\eta^{\ve}$, which can be seen as the discrepancy between the conditional distributions of $Y|Z$ and $Y^\vey|Z$. 
The last term highlights the effect of deviation around the decision boundary $1/2$, quantifying the increased uncertainty from the original distribution to perturbed distribution, where a narrower distance from $1/2$ means more uncertainty. A similar term appears in the excess risk bound derived by \cite{qin2024adaptive}.

Recall from Section~\ref{ssec:mrma} that model reversal flips a classifier’s sign if it has estimated accuracy less than $0.5$ under the original distribution $P$. 
Let $n_1$ be the sample size used for performance evaluation and $r^{(b)}$ be the {true} classification accuracy of $f^{(b)}$ under the distribution $P$. The following theorem shows how model reversal reduces the excess risk with high probability.

\begin{theorem}\label{thm:ermr}
    Let $\gR$ be a fixed representation function from $\gX$ to $\gZ$. Suppose $f^{*(b)}$ is obtained by selecting between $\{f^{(b)},-f^{(b)}\}$ via model reversal using $n_1$ evaluation samples. Then, with probability at least 
    $\Phi\!\Bigl(\sqrt{n_1}\,\frac{|r^{(b)}-1/2|}{\sqrt{r^{(b)}(1-r^{(b)})}}\Bigr)$ 
    (where $\Phi$ is the cumulative distribution function of normal distribution),
    \begin{align*}
        L_P\bigl(f^{*(b)},h^*\bigr)
        \;\le\;&
        L_{P^{\ve}}\!\bigl(f^{(b)},h^{*\ve}\bigr)
        \;+\;
        \frac{3}{4}\,d_{\mathrm{TV}}(P_Z,P_Z^{\ve})\\
        &\;+\;
        2\,\sE_{P_Z}\Bigl[\bigl|\eta(z)-\tfrac12\bigr|\;-\;\bigl|\eta^{\ve}(z)-\tfrac12\bigr|\Bigr].
    \end{align*}
\end{theorem}

Compared with the bound in Theorem~\ref{thm:erini}, the term $\sE_{P_Z}\bigl[|\eta(z)-\eta^{\ve}(z)|\bigr]\in[0,1]$ is replaced in Theorem~\ref{thm:ermr} by $\sE_{P_Z}\Bigl[\bigl|\eta(z)-\tfrac12\bigr|\;-\;\bigl|\eta^{\ve}(z)-\tfrac12\bigr|\Bigr]\in[0,1/2]$. When all perturbed datasets are non-negative (i.e., those with $g^{\ve}(z)<1/2$), they share the same Bayes classifier as the original distribution, and $|\eta(z)-\eta^{\ve}(z)|=|\eta(z)-\tfrac12|-|\eta^{\ve}(z)-\tfrac12|$. In this case, the excess risk bounds in Theorems~\ref{thm:erini} and \ref{thm:ermr} coincide. In contrast, when some datasets are negative, model reversal leads to a tighter bound, demonstrating its added benefit in such scenarios. Moreover, since dataset utility is typically unknown in practice, Theorem~\ref{thm:ermr} guarantees that with a sufficiently large evaluation sample size \( n_1 \), model reversal can effectively mitigate the influence of less-informative data and improve overall model performance.

Next, we analyze the final averaged classifier $f^\dagger$ produced by combining $B$ reversed weak classifiers $\{f^{*(b)}\}_{b=1}^B$ with weights determined by Equation \eqref{eq:wsingle}. The following theorem provides a bound for its excess risk under $P$ as $B\to\infty$.

\begin{theorem}\label{thm:erma}
    Let $\gR$ be a fixed representation function from $\gX$ to $\gZ$. 
    For any \((z,y)\sim P\), let \( F_z^\ve \) denote the distribution of classification accuracy \( \tilde r_z^{*(b)} = \sP\{\text{sign}[f^{*(b)}(z)] = y\}\) under the random sampling of training and evaluation sets. Assume this distribution has support \([r_{z,0}, r_{z,1}]\) with non-degenerate mass near \( r_{z,1} \). Let \( f^\dagger \) be the model-averaged classifier defined in Equation~\eqref{eq:fsingle}, and let $B_0$ be the number of classifiers exceeding the cutoff $r_0$. 
    If $r_0$ is chosen such that $B_0/B \to 0$ as $B\to\infty$, then
    \begin{align*}
        \limsup_{B \to \infty} L_P\bigl(f^\dagger,h^*\bigr)
        \;\le\;&
        \blue{L_{P^{\ve}}\!\bigl(f^\dagger,h^{*\ve}\bigr)}
        \;+\;
        \frac{3}{4}\,d_{\mathrm{TV}}(P_Z,P_Z^{\ve})\\
        &\;+\; 2\,\sE_{P_Z}\Bigl[\bigl|\eta(z)-\tfrac12\bigr|+\tfrac12-r_{z,1}\Bigr],
        \quad \text{in probability}.
    \end{align*}
\end{theorem}

This result shows that the excess risk of the model-averaged classifier is asymptotically bounded in probability bounded by a quantity that depends on the best-performing classifiers, i.e., those near the upper end \( r_{z,1} \) of the accuracy distribution. As the number of classifiers \( B \) increases, the weight assigned to lower-performing models vanishes under the thresholding scheme in Equation~\eqref{eq:wsingle}, leading the ensemble to concentrate on top-performing classifiers. This effectively implements a form of regularized extreme value selection. The convergence rate of the bound depends on the amount of probability mass that \( F_z^\ve \) places near \( r_{z,1} \); see~\cite{leadbetter2012extremes} for related results from extreme value theory. The third term in the bound reflects the remaining gap between the best achievable accuracy \( r_{z,1} \) and the Bayes-optimal level \( |\eta(z) - \tfrac{1}{2}| + \tfrac{1}{2} \). Compared to Theorem~\ref{thm:ermr}, this bound is tighter when high-quality classifiers exist and can be identified reliably through evaluation. These insights are consistent with our empirical findings in Section~\ref{sec:exp}, where model averaging improves performance, especially in high-noise regimes.

While the above results focus on a single-server setting, they naturally extend to multi-server scenarios introduced in Section~\ref{ssec:multi}. In a heterogeneous multi-server environment, each server can treat other servers’ classifiers as additional weak classifiers and apply model reversal and model averaging locally. The same type of bounds hold, subject to analogous conditions for data partition and accuracy estimation.

Deciding how to split data between the training and evaluation sets, as well as how many weak classifiers to train, is crucial. The theoretical results above provide guidance: a sufficiently large training sample size \( n_0 \) is necessary to ensure that each weak classifier is adequately learned, achieving lower excess risk \( L_{P^{\ve}}(f^{(b)}, h^{*\ve}) \); a larger evaluation sample size \( n_1 \) improves the precision of utility estimates, thereby strengthening model reversal; and a larger number of weak classifiers \( B \) leads to a more stable aggregated classifier, enhancing the effect of model averaging. Meanwhile, our empirical findings (see Appendix~\ref{appsim:paras}) suggest that in high-noise regimes, it may be more beneficial to allocate more samples to evaluation rather than training.

\section{Functional Data Classification under LDP}\label{sec:func}
Advances in sensing technologies have made it increasingly feasible to collect data densely sampled over temporal or spatial domains, giving rise to what is commonly referred to as functional data \citep{ramsay2005}. Examples include physiological signals such as heart rate and electrocardiogram readings from wearable devices, GPS traces from smartphones that capture mobility patterns, and sensor data from activity trackers monitoring daily routines or sleep cycles. 
These data streams, while informative, can also expose sensitive information related to an individual's health, behavior, and lifestyle. This makes privacy preservation particularly critical in biomedical and health-monitoring applications, where temporal or spatial patterns may reveal personal characteristics or preferences \citep{stisen2015smart}. 
Despite its importance, privacy-preserving learning with functional data remains largely unexplored. As a concrete demonstration of our approach, this section presents the full pipeline for constructing a classifier for functional data under LDP. To the best of our knowledge, this is the first framework for functional classification under LDP.

Suppose there is a single server with access to \(N\) clients. Each client holds a square-integrable functional covariate \(x(t)\), defined on a compact domain \(\gI=[0,1]\), and a corresponding binary label \(y\in\{-1,1\}\). Our goal is to learn a functional linear classifier 
\[
   f(x) \;=\;\alpha \;+\;\int_0^1 x(t)\,\beta(t)\,dt,
\]
such that we predict \(\hat{y}(x) = \text{sign}\{f(x)\}\). In accordance with \cite{yang2020local}, an LDP algorithm typically proceeds in four stages: encoding, perturbation, aggregation, and estimation. Below, we describe how each stage applies to functional data.

\subsection{Encoding and Perturbation}\label{sec:encoding}
For each client with a functional covariate \(x(t)\) and label \(y\) in the training set, the data reported to the server is transformed through the following steps.

\textbf{Dimensionality Reduction.}
Functional data are often observed on a fine grid of time points (or locations), leading to high-dimensional vectors. To reduce dimensionality in a privacy-preserving way, we project \(x(t)\) onto a finite set of basis functions \(\{\phi_k(t)\}_{k=1}^d\), which is corresponding to the representation process in Section \ref{ssec:nota}. Concretely,
\begin{equation}\label{eq:encoding}
    x(t) \;=\;
    \sum_{k=1}^d \;z_k\,\phi_k(t) \;+\;\xi(t),
\end{equation}
where \(\bz = (z_1,\dots,z_d)^{\top}\in\mathbb{R}^d\) are the basis coefficients and \(\xi(t)\) is a residual term excluded from the model. 
\blue{The choice of basis functions \(\Phi=(\phi_1,\dots,\phi_d)\) (e.g., B-splines, Fourier bases) and the dimension $d$ are specified by the server, where $d$ controls the representation fidelity and also affects the magnitude of the subsequent privatization for a fixed privacy budget.}
This projection serves two key purposes:
\begin{enumerate}
    \item Functional data fitting. In practice, \(x(t)\) is discretely observed at $T$ time points, \(\bm{x} = (x(t_1),\ldots,x(t_T))^\top\). For smooth families like B-splines of order \(m\), fitting accuracy improves at the optimal rate of \( T^{-2m/(2m+1)} \) as the number of knots grows like \( T^{1/(2m+1)} \) \citep{eubank1999nonparametric}, which indicates a relatively slow rate of increase for $d$.
    \item Privacy enhancement. Truncating at \( d \) dimensions mitigates overfitting to person-specific variations, thereby enhancing privacy. It also significantly reduces communication overhead compared to transmitting the full functional data.
\end{enumerate}
\blue{When the chosen basis cannot adequately represent the underlying signal, or when the classification rule is too complex to be well approximated by a basis expansion, the separate perturbation mechanism in Equation~\eqref{eq:eLDP} may lead to substantial information loss. This mechanism can be viewed as inducing a form of measurement error in both covariates and labels, and classification performance may deteriorate further in such settings \citep{li2024update}. More generally, Equation~\eqref{eq:eLDP} can be replaced by more sophisticated LDP mechanisms, together with the corresponding algorithms for constructing weak classifiers \citep{berrett2019classification,ma2024optimal}, while our proposed evaluation, model reversal, and model averaging procedures remain applicable.}

\blue{
\textbf{Rescaling.}
To bound sensitivity for LDP, we rescale each component of \(z\) into a fixed range, typically \([-1,1]\). We consider two simple examples:
\begin{align*}
    \text{(Tanh Transformation)} &\quad z_k^* \;=\;\tanh(z_k),\\
    \text{(Max-Abs Normalization)} &\quad z_k^* \;=\;
    \frac{z_k}{\max_{k'} \lvert{z_{k'}}\rvert}\, .
\end{align*}
Here, \(\bm{z}^*=(z_1^*,\dots,z_d^*)^\top\in[-1,1]^d\).
The tanh transformation is applied componentwise and introduces a nonlinear squashing effect, whereas the Max-Abs normalization is a linear rescaling that preserves coefficient signs and the relative magnitudes across coordinates. Other bounded transformations can also be used to map the feature vector into a fixed bounded range and thereby control sensitivity under the chosen LDP mechanism. In Appendix~\ref{appsim:single}, we compare these two rescaling choices and observe similar classification accuracy in our experimental setting. Therefore, for concreteness, the experiments in Section~6 report results based on the tanh transformation.
}

\textbf{Perturbation.} 
After rescaling, each client adds noise according to an LDP mechanism. Using a Laplace mechanism \citep{dwork2006calibrating}, they report
\begin{equation}\label{eq:xpert}
    \bz^{(\varepsilon_z)} \;=\; \bz^* \;+\; \boldsymbol{\delta}^{(\varepsilon_z)},
\end{equation}
where \(\boldsymbol{\delta}^{(\varepsilon_z)}\in\mathbb{R}^d\) has independent and identically distributed (i.i.d.) Laplace entries with scale \(\lambda = d\,\Delta / \vz\) and \(\Delta=2\) (since each \(z_k^*\in[-1,1]\)). Hence each coordinate’s privacy budget is effectively \(\vz/d\). 
To protect the binary label \(y\), we apply the randomized response mechanism \citep{warner1965randomized}:
\begin{equation}\label{eq:ypert}
    \sP\bigl(y^{(\varepsilon_y)} = y\bigr)
    \;=\;
    \frac{\,e^{\,\vy}\,}{\,1 + e^{\,\vy}\,},
\end{equation}
where \(\vy = \varepsilon - \vz\). In the experiment, we choose \(\vy = \varepsilon/(d+1) = \vz/d\). Theorem~\ref{thm:ldp1} shows that the entire encoding and perturbation procedure satisfies \(\varepsilon\)-LDP.

\begin{theorem}\label{thm:ldp1}
    Let \(\gM_f\) be the privacy mechanism taking \(\bigl(x(t),y\bigr)\) as input and returning \(\bigl(\bz^{(\varepsilon_z)},y^{(\varepsilon_y)}\bigr)\) via Equations \eqref{eq:xpert}--\eqref{eq:ypert}. Then \(\gM_f\) satisfies \(\varepsilon\)-LDP.
\end{theorem}

It is worth noting that we employ the Laplace mechanism here due to its simplicity and ease of interpretation, which facilitates our subsequent demonstrations of the proposed methods. As shown in \cite{duchi2018minimax}, adding Laplacian noise to (appropriately truncated) observations is an optimal privacy mechanism in one-dimensional mean estimation under LDP. However, this optimality does not directly extend to higher-dimensional settings. For additional insights on such mechanisms, we refer readers to \cite{nguyen2016collecting,duchi2018minimax,wang2019collecting,xiao2023geometry}.

\subsection{Classifier Construction with MRMA}\label{ssec:aggest}
Once the server collects $\{(\bz_i^\vez,y_i^\vey)\}_{i=1}^{N_0}$ from \(N_0\) training clients, it can train multiple {weak classifiers} as introduced in Section~\ref{ssec:eval}. Two general methods are outlined below.

\textbf{Method I.} It's straightforward to build a vector classifier by treating \(\bm{z}_i^{(\varepsilon_z)} \in \mathbb{R}^d\) as the input to a standard classification model (e.g., logistic regression, SVM). Denote the estimated intercept and coefficients by \(\hat\alpha\) and \(\hat{\bm{b}}=(\hat b_1,\dots,\hat b_d)^\top\in\mathbb{R}^d\), respectively. We then recover the slope function via $\widehat\beta(t)=\sum_{k=1}^d\hat b_k\phi_k(t)$, and define the corresponding functional classifier as
\begin{equation*}
   \hat{f}(x)
   \;=\;
   \hat{\alpha} \;+\;\int_0^1 x(t)\,\widehat{\beta}(t)\,dt.
\end{equation*}

\textbf{Method II.} Before building the classifier, we first reconstruct the function using the perturbed basis scores, i.e., treating $x^\vez_i(t)=\sum_{k=1}^d z^\vez_{i,k}\phi_k(t)$ as the perturbed functional covariate, where $z^\vez_{i,k}$ represents the $k$th element of $\bz^\vez_i$. Then, the classifier can be obtained using the functional conjugate gradient algorithm \citep[CG, ][]{kraus2019classification}, functional distance weighted discrimination \citep[DWD, ][]{sang2022reproducing}, or any other established functional classification method.

With these two ways of constructing classifiers, in the Appendix \ref{appsim:single}, we illustrate the effects of dimensionality reduction, rescaling, and perturbation on the misclassification rate of the classifiers. Experiment results indicate that the projection of functional data and rescaling of coefficient vectors have a small impact on the classifier's performance. And the performance of two types of transformations is similar in our context. Moreover, as the privacy budget $\varepsilon\to0$, the misclassification rates of the weak classifiers based on the perturbed data tend to 50\%.

\textbf{MRMA.} Based on the noised binary evaluations collected from the additional \( N_1 \) evaluation clients in Section~\ref{ssec:eval}, we apply the MRMA procedure introduced in Section~\ref{ssec:mrma}. Specifically, for a weak classifier \(f^{(b)}\), if its estimated accuracy \(\tilde{r}^{(b)} < 50\%\), we invert its parameters:
\[
  (\hat{\alpha}^{(b)},\,\widehat{\beta}^{(b)}(t))
  \;\longmapsto\;
  \bigl(-\hat{\alpha}^{(b)},\,-\,\widehat{\beta}^{(b)}(t)\bigr).
\]
Let \(\{f^{*(b)}\}_{b=1}^B\) be the {reversed} classifiers, and denote their weights by \(\{w_b\}_{b=1}^B\) as in Equation \eqref{eq:wsingle}. Our final averaged functional classifier becomes
\begin{equation*}
    f^\dagger(x)\;=\;\hat\alpha^\dagger\;+\;\int_0^1 x(t)\,\widehat\beta^\dagger(t)\,dt,
\end{equation*}
where $\hat\alpha^\dagger=\sum_{b=1}^Bw_b\,\hat\alpha^{*(b)}, \widehat\beta^\dagger(t)=\sum_{b=1}^Bw_b\,\widehat\beta^{*(b)}(t)$.

This example demonstrates how our LDP framework generalizes naturally to infinite-dimensional data via basis truncation, preserving both privacy and utility. It highlights the method’s flexibility in accommodating complex data structures.

\section{Experiments}\label{sec:exp}
In this section, we present experiments that primarily demonstrate the improvements in classification accuracy achieved by applying our proposed techniques, compared to baselines that do not use them. These results validate the effectiveness of our approach. Differences in performance across classifier types are due to their intrinsic properties and are not the focus of this study.

In each experimental trial, the functional covariate $X(\cdot)$ is generated by $X(t)=\sum_{j=1}^{50}\xi_{j}$ $\zeta_j\phi_j(t)$ for $t\in[0,1]$, where $\xi_{j}$'s are independently drawn from $U(-\sqrt{3},\sqrt{3})$, $\zeta_j=(-1)^{j+1}j^{-1},j=1,\dots,50$, $\phi_1(t)=1$ and $\phi_j(t)=\sqrt2\cos((j-1)\pi t)$ with $j\ge2$. 
The binary response variable $Y$, taking values $1$ or $-1$, is generated using the following logistic model:
\begin{align*}
    &f\left(X\right)=\alpha_0+\int_0^1 X(t) \beta(t) d t, \quad \operatorname{Pr}\left(Y=1\right)=\frac{\exp \left\{f\left(X\right)\right\}}{1+\exp \left\{f\left(X\right)\right\}},
\end{align*}
where $\alpha_0=0.1$, $\beta(t)$ is the slope function, and $f(X)$ is referred to as the classification function. 
Results in Appendix \ref{appsim:single} show that the performance with \( d = 4, 5, 6 \) cubic B-Spline is comparable, and the performance based on Tanh or Max-Abs transformation is close. Thus we present results with $d=4$ and employ the tanh transformation.

\subsection{Single-Server Classification with MRMA}\label{secexp:RMA}
In this section, we demonstrate the improvements in classification accuracy brought about by model reversal and model averaging. 
We generate data for a server using the slope function \( \beta(t) = \sum_{j=1}^{50} 4(-1)^{j+1} j^{-2} \phi_j(t) \). 
Assuming the server has a total of \(N = 3000\) clients, we allocate \(N_0 = 500\) for training and the remaining $N_1=2500$ for evaluation. To construct classifiers, we sequentially draw \(n_0 = 50\) samples from the training data set without replacement, repeating this procedure $B=50$ times. At the same time, we partition the evaluation set into $B=50$ subsets of equal size, enabling us to evaluate the classification accuracy of each classifier using $n_1=50$ distinct samples. To assess the performance of the classifiers, we randomly generate a testing data set comprising $500$ samples during each trial, repeating this procedure $500$ times.

Figure \ref{fig:rma} showcases the misclassification rates along with error bars, for various classifiers across different \( \varepsilon \) levels. In this figure, ``Weak" denotes the average misclassification rate of \( B = 50 \) weak classifiers obtained through sampling. ``MR" represents the average misclassification rate of \( B \) weak classifiers after model reversal under LDP. ``MA" signifies the results when using model averaging on weak classifiers with cutoff value $r_0=0.8$, while ``MRMA" illustrates the results of applying both model reversal and model averaging. 
To compare with classic ensemble methods, we train $B$ weak classifiers under LDP. Each classifier is trained with $N/B=60$ instances from the total sample. We then obtain the results through majority voting and equal-weight averaging, denoted as ``Voting" and ``Averaging", respectively. And ``All data" denotes the classifier trained with $N$ clients directly.

The classifier ``All data", even if it is trained with $3000$ clients, shows almost no improvement over the classifier ``Weak" when $\varepsilon$ is small (indicating substantial noise interference). And classifiers ``Voting" and ``Averaging" also perform similarly. 
However, our proposed techniques, both model reversal and model averaging, significantly improve the performance of all types of weak classifiers. And MRMA further enhances the performance of SVM and CG classifiers substantially. 
Figure \ref{fig:size} in Appendix \ref{appsim:size} demonstrates that even when allocating more clients for training weak classifiers, MR and MA can still significantly enhance the classifiers' performance. 
From the client's perspective, this means that to achieve a target classification accuracy, our method allows for a higher level of privacy protection, that is, a smaller value of $\varepsilon$ can be used without sacrificing performance.
For further discussions regarding distinctions among different types of classifiers, please refer to Appendix \ref{appsim:single}.

\begin{figure}[t!]
\centering
\subfig{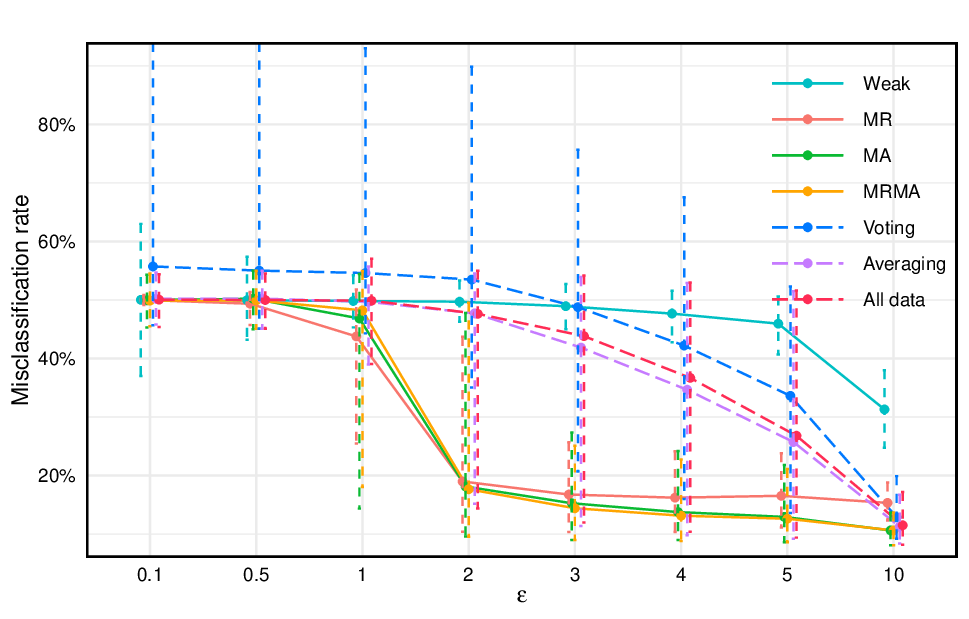}{a}{Logistic}{0.49}\hfill
\subfig{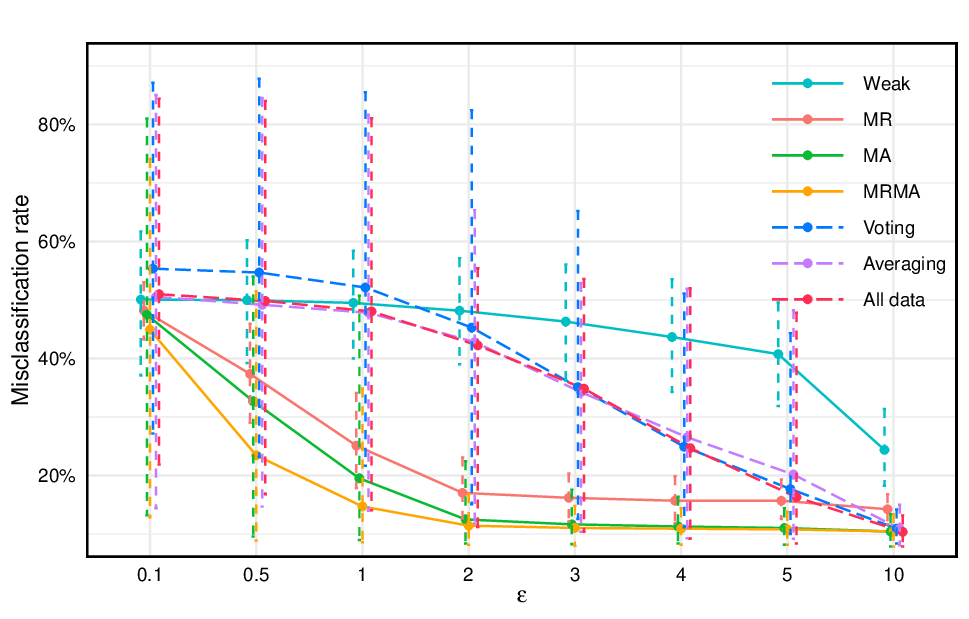}{b}{SVM}{0.49}
\subfig{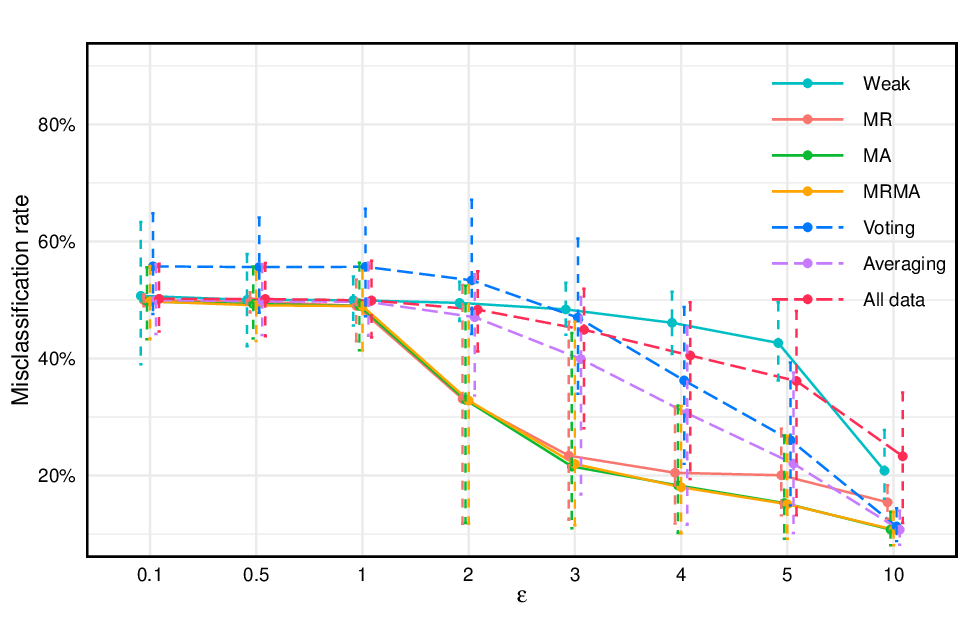}{c}{DWD}{0.49}\hfill
\subfig{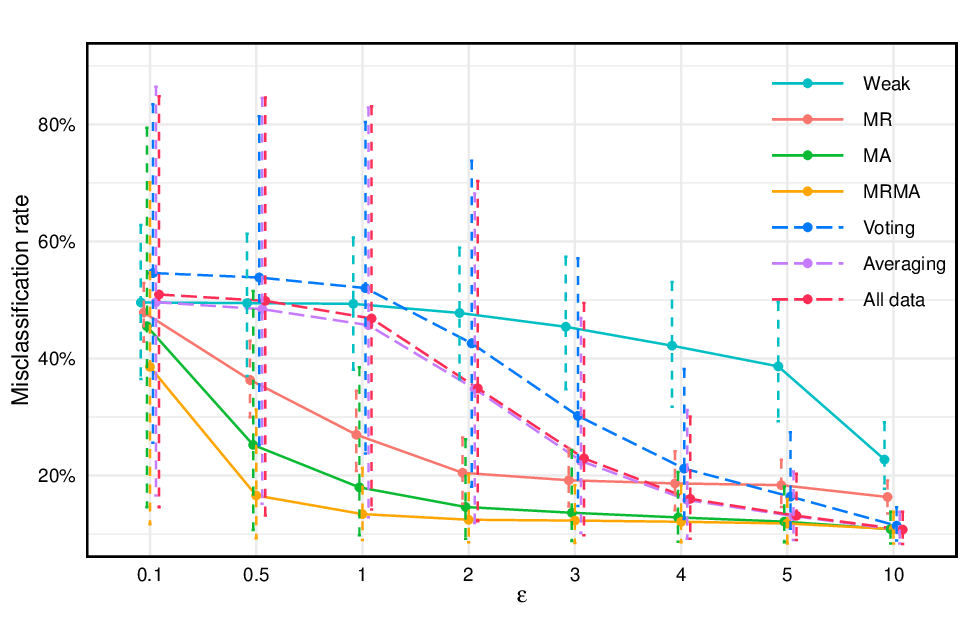}{d}{CG}{0.49}
\caption{The misclassification rates of classifiers with a single server under $\varepsilon$-LDP.}
\label{fig:rma}
\end{figure}

\subsection{Multi-Server Classification with MRMA under Heterogeneity}\label{secexp:FL}
In this section, we showcase the improvements in classification accuracy achieved through multi-server learning for individual servers. We consider a total of $K=25$ servers, divided into three groups, each characterized by a distinct slope function to introduce heterogeneity. Specifically,
\begin{itemize}
    \item Group 1: for $k=1,\dots,10$, \( \beta_k(t)=\sum_{j=1}^{50}\gamma_j(-1)^{j+1} j^{-2} \phi_j(t) \) with $\gamma_j\overset{i.i.d.}{\sim}U(-8,-2)$.
    \item Group 2: for $k=11,\dots,15$, $\beta_k(t)\sim\text{GP}(0,K(s,t))$ with $K(s,t)=\exp(-15|s-t|)$.
    \item Group 3: for $k=16,\dots,25$, \( \beta_k(t)=\sum_{j=1}^{50}\gamma_j(-1)^{j+1} j^{-2} \phi_j(t) \) with $\gamma_j\overset{i.i.d.}{\sim}U(2,8)$.
\end{itemize}
Here, \( \text{GP}(0, K(s,t)) \) represents a Gaussian process with zero mean and kernel \( K(s,t) \). It is essential to highlight that the slope functions of servers within the same group are not identical. Moreover, the directions of the slope functions in groups 1 and 3 are opposite, serving as a test to assess the potential negative transfer impact during multi-server learning. In group 2, the slope function is randomly generated, resulting in an approximate 50\% misclassification rate for classifiers built on servers in this group. This design is purposefully implemented to gauge its potential disruption to the multi-server learning process.

For each server, we generate \(N=3000\) clients, with an additional \(500\) clients designated for testing. Each server employs algorithms with the parameters \(N_0=500\), \(N_1=2500\), \(n_0=50\), and \(B=50\). 
Initially, each server independently obtains its own classifier using MRMA, with parameters set to \( \varepsilon_v = \varepsilon \) and \( r_0 = 0.8 \). Subsequently, all servers proceed to perform multi-server learning. 
The parameters for this phase are set as \( \varepsilon_v = \varepsilon_v^* = \varepsilon/2 \), with \( r_0^* = 0.8 \) to mitigate potential negative transfer effects. 
As expected, servers in group 2 exhibit misclassification rates around 50\%. The average misclassification rates for classifiers of servers in groups 1 and 3 under these two scenarios are illustrated in Figure \ref{fig:fl}. 
The results show that multi-server learning significantly improves the performance of both logistic and DWD classifiers, even with server heterogeneity. While SVM and CG classifiers already perform well within a single server setting, multi-server learning shows comparable or slightly better performance.

\begin{figure}[htbp]
\centering
\subfig{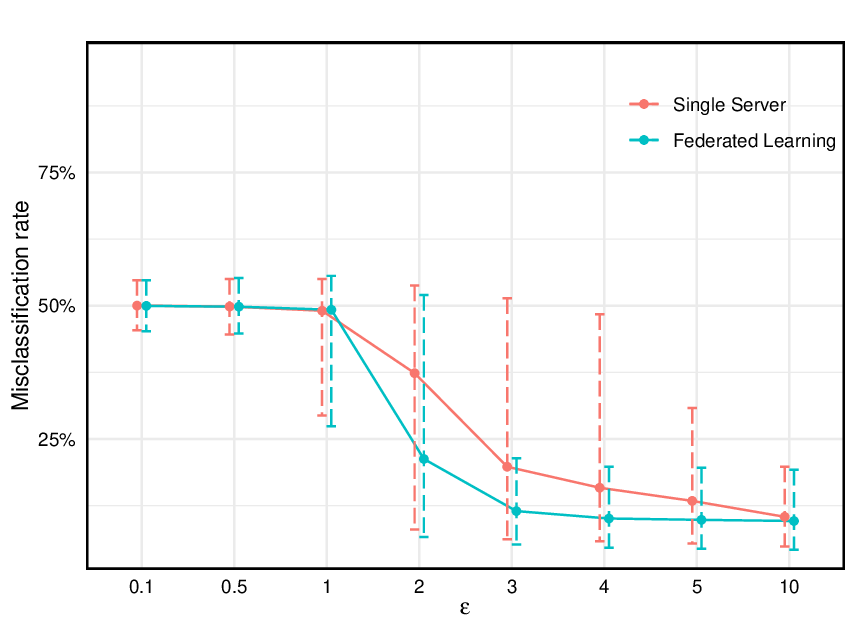}{a}{Logistic}{0.49}\hfill
\subfig{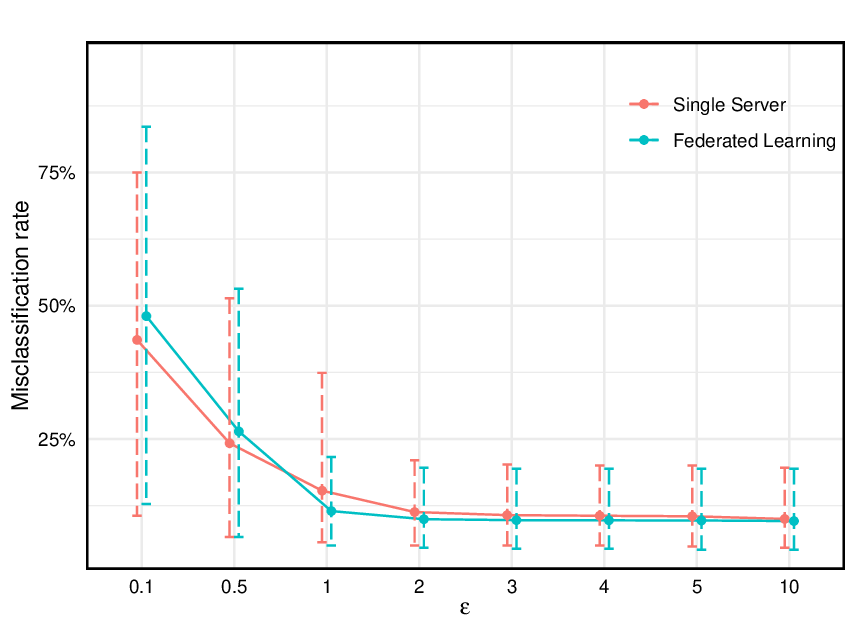}{b}{SVM}{0.49}
\subfig{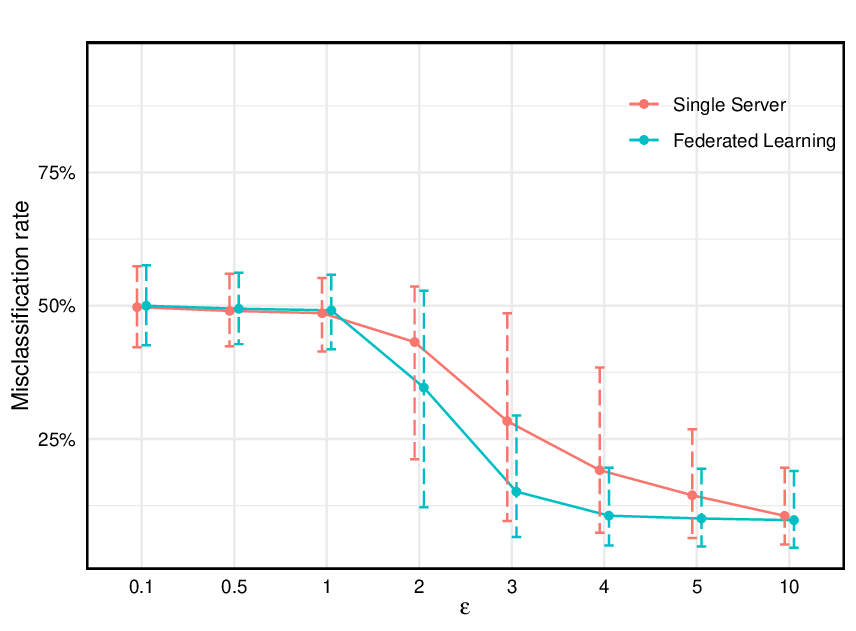}{c}{DWD}{0.49}\hfill
\subfig{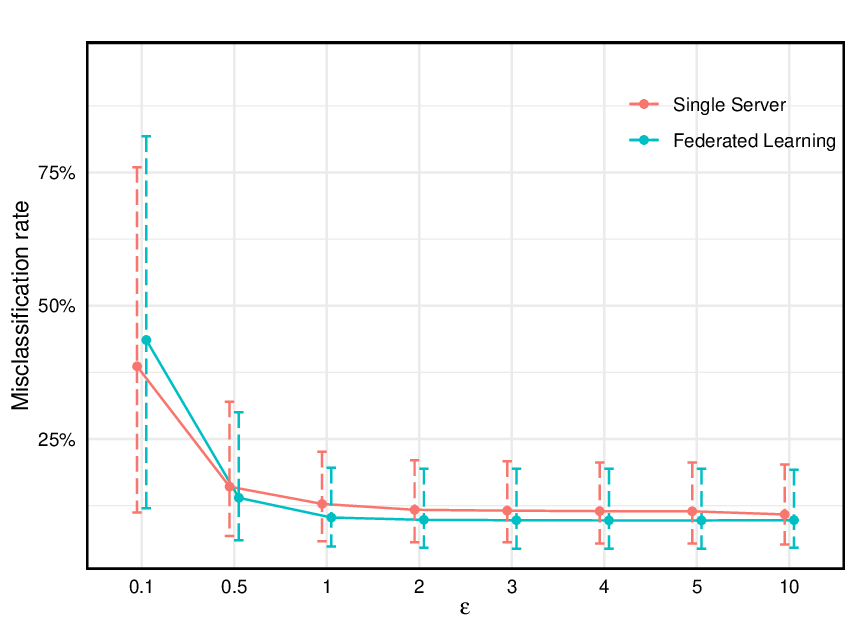}{d}{CG}{0.49}
\caption{The misclassification rates of classifiers with multi-server under $\varepsilon$-LDP.}
\label{fig:fl}
\end{figure}

\section{Real Application}\label{sec:real}
To demonstrate the practical utility of our proposed classification methods under LDP, we evaluate their performance on both vector-valued and functional datasets, each involving sensitive information that necessitates privacy protection. Section~\ref{ssec:real_vec} presents results from two real-world datasets with vector predictors related to diabetes and employee attrition, while Section~\ref{ssec:real_fun} focuses on applications involving functional predictors derived from physical activity and speech recordings.

\subsection{Vector-valued Predictors}\label{ssec:real_vec}
We first evaluate our methods on two publicly available datasets: a healthcare dataset related to diabetes risk and an employee attrition dataset. Both datasets contain sensitive information, making them suitable candidates for privacy-preserving classification. \blue{We compare our methods with the histogram-based classifier under LDP proposed by \citet{berrett2019classification}, referred to as ``Histogram''. We fix the number of splits to \(h^{-1}=4\) along each axis, since tuning \(h\) under LDP would require an additional privacy-preserving validation step and further complicate the procedure.}

The Diabetes dataset, published by \citet{pore2023diabetes}, comprises 2768 samples and includes eight health-related attributes used to predict whether an individual has diabetes. We randomly split the dataset into 2214 training samples (approximately 80\%) and 554 testing samples. To assess classifier performance under LDP, we repeat this process 500 times. In each repetition, we draw \(N_0 = 414\) samples from the training samples to serve as the training set, while the remaining \(N_1 = 1800\) samples are used for evaluation. Specifically, we randomly sample \(n_0 = 60\) observations (without replacement) from the training set to construct a classifier, repeating this procedure $B=30$ times to generate multiple weak classifiers. Then we partition the evaluation set into \(B = 30\) subsets of equal size, enabling us to evaluate the classification accuracy of each classifier using $n_1=60$ distinct samples.

The Employee dataset, introduced by \citet{elmetwally2023employee}, contains employment-related records for 4653 individuals, each described by eight variables, including education, job history, and workplace factors. The task is to predict whether an employee leaves the company. We split the dataset into 3722 training samples and 931 testing samples, again repeating the process 500 times. Within each repetition, we use \(N_0 = 722\) for training and \(N_1 = 3000\) for evaluation. Classifier construction and evaluation follow the same procedure as described for the Diabetes dataset, using \(n_0 = n_1 = 100\) and \(B = 30\).

Tables~\ref{tab:real_diabetes} and \ref{tab:real_employee} report the average misclassification rates and standard deviations across 500 repetitions under various privacy levels (\(\varepsilon\)) for both datasets. We set $r_0=0.7$ to ensure the existence of weak classifiers satisfying this criterion. Our MRMA-based classifier consistently achieves superior accuracy, particularly when \(\varepsilon\) is small, corresponding to stronger privacy guarantees. As \(\varepsilon\) increases, indicating weaker privacy constraints, the performance differences among our proposed methods, the ensemble strategies (``Voting'' and ``Averaging''), and the ``All data'' method become negligible, with all achieving similar accuracy and significantly outperforming the histogram-based approach. 
Notably, even when \(\varepsilon = 1000\), the histogram method remains less effective, likely due to the substantial Laplace noise introduced, whose scale is \(2^{d+1}/\varepsilon\), where \(d = 8\) is the number of predictors. This level of noise, although reduced for large \(\varepsilon\), still introduces enough distortion to degrade classification performance.

\begin{table}[h!]
\centering
\begin{tabular}{lcccccccc}
\toprule
$\varepsilon$ & Weak & MR & MA & MRMA & Voting & Averaging & All data & Histogram \\
\midrule
0.1   & 49.55 & 48.22 & 46.22 & \textbf{43.44} & 55.25 & 49.03 & 50.82 & 49.75 \\
      & (0.078) & (0.032) & (0.153) & (0.144) & (0.221) & (0.157) & (0.158) & (0.057) \\
0.5   & 49.23 & 40.23 & 37.32 & \textbf{34.49} & 53.93 & 45.61 & 46.33 & 49.75 \\
      & (0.066) & (0.031) & (0.095) & (0.030) & (0.226) & (0.152) & (0.153) & (0.057) \\
1     & 48.48 & 34.90 & 34.59 & \textbf{34.31} & 48.09 & 41.50 & 41.57 & 49.75 \\
      & (0.055) & (0.020) & (0.033) & (0.018) & (0.197) & (0.132) & (0.132) & (0.057) \\
5     & 42.95 & 34.16 & \textbf{34.31} & \textbf{34.31} & 34.38 & \textbf{34.31} & \textbf{34.31} & 49.64 \\
      & (0.047) & (0.017) & (0.018) & (0.018) & (0.019) & (0.018) & (0.018) & (0.057) \\
10    & 38.06 & \textbf{34.01} & 34.29 & 34.30 & 34.31 & 34.31 & 34.31 & 49.58 \\
      & (0.030) & (0.018) & (0.018) & (0.018) & (0.018) & (0.018) & (0.018) & (0.058) \\
1000  & 26.60 & 26.02 & 23.03 & 23.03 & 22.81 & 22.23 & \textbf{22.05} & 39.80 \\
      & (0.010) & (0.013) & (0.018) & (0.018) & (0.016) & (0.016) & (0.015) & (0.036) \\
\bottomrule
\end{tabular}
\caption{Mean and standard deviation (in parentheses) of misclassification rates for various classification methods, computed over 500 random sample splits of the Diabetes dataset and evaluated under different values of $\varepsilon$ for $\varepsilon$-LDP, with bold values indicating the best performance.}
\label{tab:real_diabetes}
\end{table}

\begin{table}[h!]
\centering
\small
\begin{tabular}{lcccccccc}
\toprule
$\varepsilon$ & Weak & MR & MA & MRMA & Voting & Averaging & All data & Histogram \\
\midrule
0.1 & 49.71 & 47.86 & 46.43 & \textbf{44.15} & 58.23 & 48.06 & 48.25 & 49.98 \\
 & (0.082) & (0.031) & (0.152) & (0.145) & (0.230) & (0.155) & (0.155) & (0.025) \\
0.5 & 49.02 & 37.73 & 35.61 & \textbf{34.43} & 52.22 & 43.50 & 43.63 & 49.98 \\
 & (0.062) & (0.027) & (0.061) & (0.014) & (0.207) & (0.141) & (0.142) & (0.025) \\
1 & 48.04 & 34.64 & 34.48 & \textbf{34.43} & 45.30 & 38.73 & 39.08 & 49.98 \\
 & (0.053) & (0.015) & (0.020) & (0.014) & (0.167) & (0.107) & (0.109) & (0.025) \\
5 & 41.19 & \textbf{34.36} & 34.39 & 34.42 & 34.44 & 34.43 & 34.43 & 49.97 \\
 & (0.040) & (0.017) & (0.015) & (0.015) & (0.014) & (0.014) & (0.014) & (0.025) \\
10 & 36.61 & \textbf{34.01} & 34.18 & 34.16 & 34.43 & 34.43 & 34.43 & 49.96 \\
 & (0.020) & (0.018) & (0.018) & (0.018) & (0.014) & (0.014) & (0.014) & (0.025) \\
1000 & 32.43 & 31.62 & 29.69 & 29.69 & 30.21 & 29.51 & \textbf{29.48} & 47.61 \\
 & (0.009) & (0.013) & (0.017) & (0.017) & (0.014) & (0.014) & (0.014) & (0.024) \\
\bottomrule
\end{tabular}
\caption{Mean and standard deviation (in parentheses) of misclassification rates for various classification methods, computed over 500 random sample splits of the Employee dataset and evaluated under different values of $\varepsilon$ for $\varepsilon$-LDP, with bold values indicating the best performance.}
\label{tab:real_employee}
\end{table}

\subsection{Functional Predictors}\label{ssec:real_fun}
We evaluate our LDP methods on two real-world classification tasks involving functional predictors. The first task uses wearable-sensor-based physical activity data to predict cardiovascular health status, while the second revisits the phoneme classification problem using frequency-domain speech features.

The first task draws on data from the National Health and Nutrition Examination Survey 2013--2014 cycle. We consider adult female participants aged 20 years or older, for whom both laboratory-measured high-density lipoprotein (HDL) cholesterol and minute-level physical activity recordings are available. HDL is known to be protective against cardiovascular disease, and its levels are influenced by physical activity. According to clinical guidelines, HDL is considered healthy for women when its concentration is at least 50 mg/dL \citep{turk2009cardiac}. Using this threshold, we define a binary classification task to distinguish participants with healthy versus unhealthy HDL levels. As the functional predictor, we use hour-level Monitor-Independent Movement Summary (MIMS) values averaged over seven days and projected onto a set of 6 cubic B-spline basis functions. The resulting dataset includes 2410 samples, with 1558 labeled as healthy and 852 as unhealthy. To evaluate classifier performance under LDP, we adopt a randomized split: 410 samples are held out as a test set, and the remaining 2000 are divided into a training set of size \(N_0 = 500\) and an evaluation set of size \(N_1 = 1500\). For each repetition, we subsample \(n_0 = n_1 = 50\) observations to train and evaluate classifiers across \(B = 30\) randomized subsets. This entire process is repeated 500 times.

\begin{figure}[h!]
\centering
\subfig{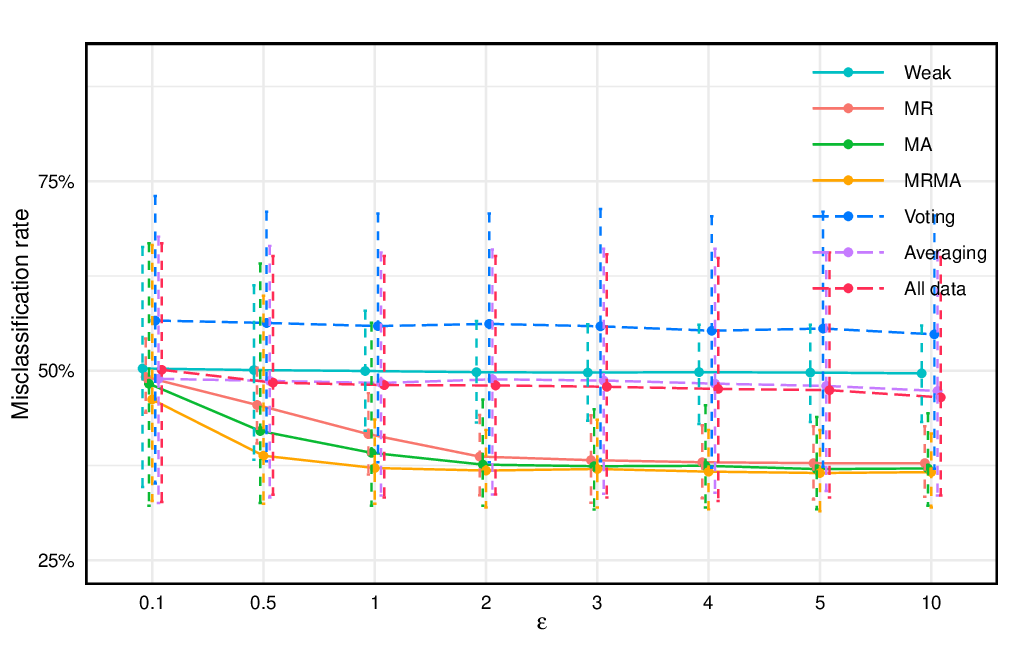}{a}{Logistic}{0.49}\hfill
\subfig{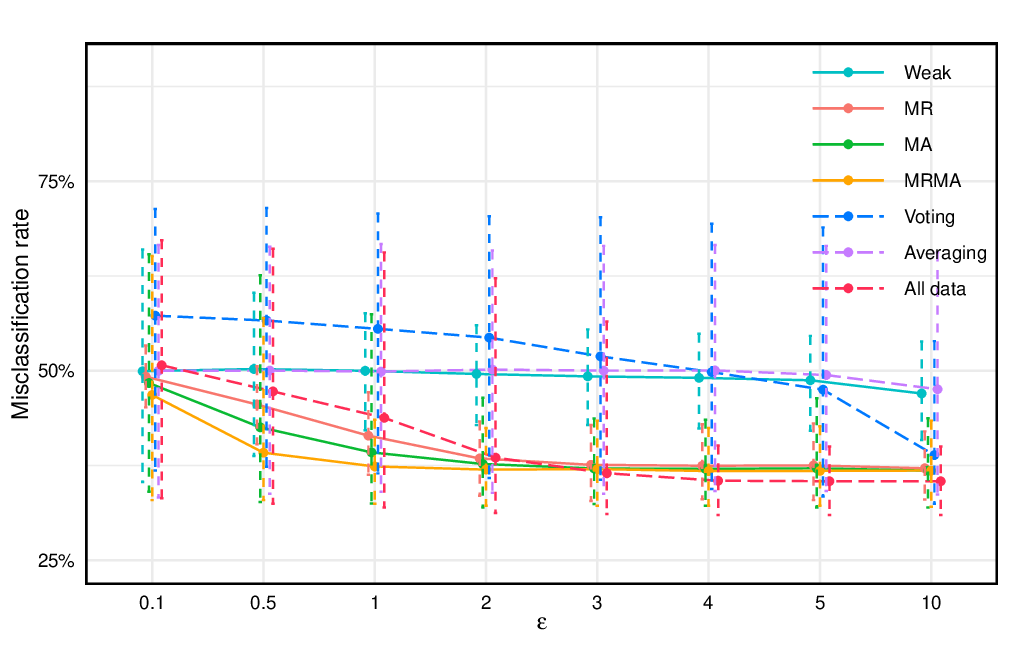}{b}{SVM}{0.49}
\subfig{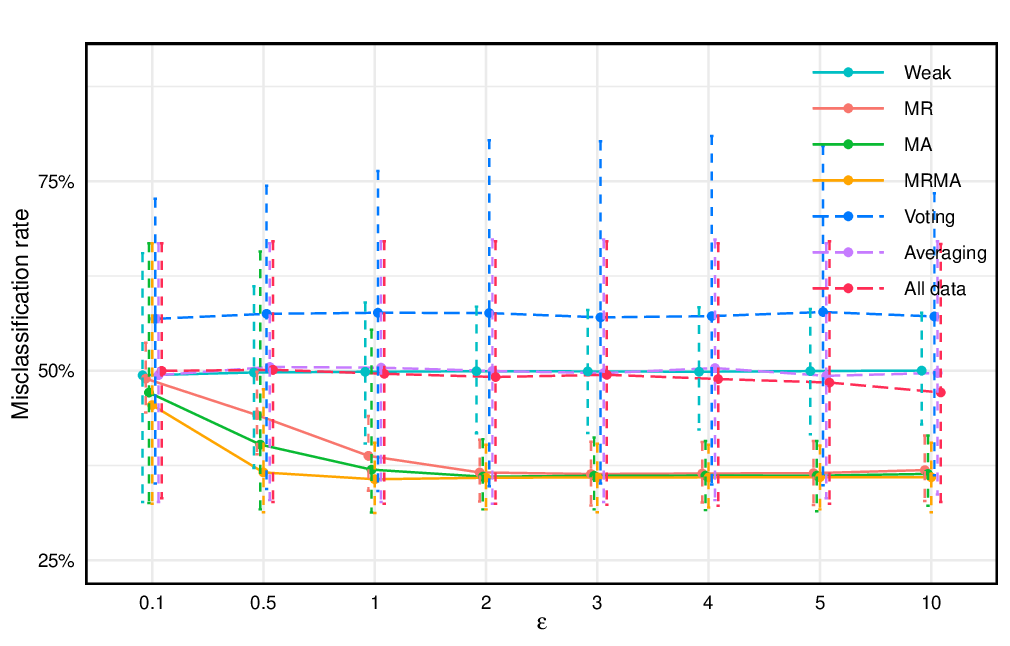}{c}{DWD}{0.49}\hfill
\subfig{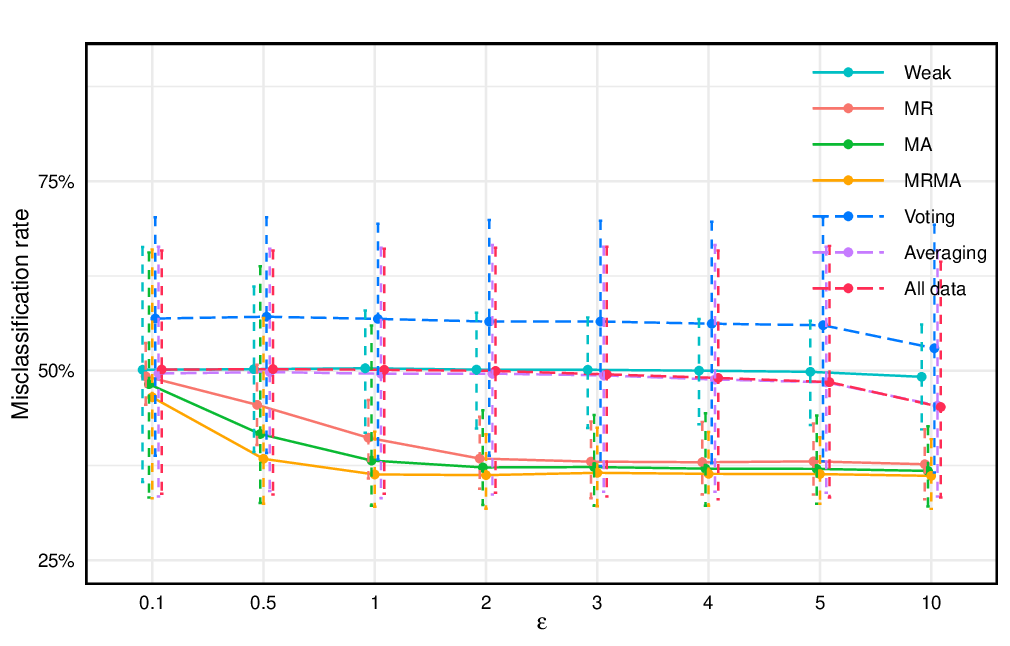}{d}{CG}{0.49}
\caption{The misclassification rates of classifiers with a single server under $\varepsilon$-LDP on the physical activity dataset.}
\label{fig:real_mims}
\end{figure}

The second task uses the Phonemes dataset derived from the TIMIT Acoustic-Phonetic Continuous Speech Corpus \citep{garofolo1993timit}. Speech frames in this dataset are extracted from the continuous speech of 50 male speakers, and log-periodograms are constructed from recordings sampled at 256 equispaced frequencies and projected onto a set of 4 Fourier basis functions. The classification task focuses on distinguishing between the phonemes ``sh'' and ``iy'', represented by 1163 and 872 curves, respectively. The log-periodogram functions are approximated using Fourier basis expansion. As in the MIMS example, we adopt a randomized split of 535 test samples, a training set of \(N_0 = 300\), and an evaluation set of \(N_1 = 1200\), with \(n_0 = n_1 = 50\), and \(B = 24\), repeating the entire process 500 times.

\begin{figure}[h!]
\centering
\subfig{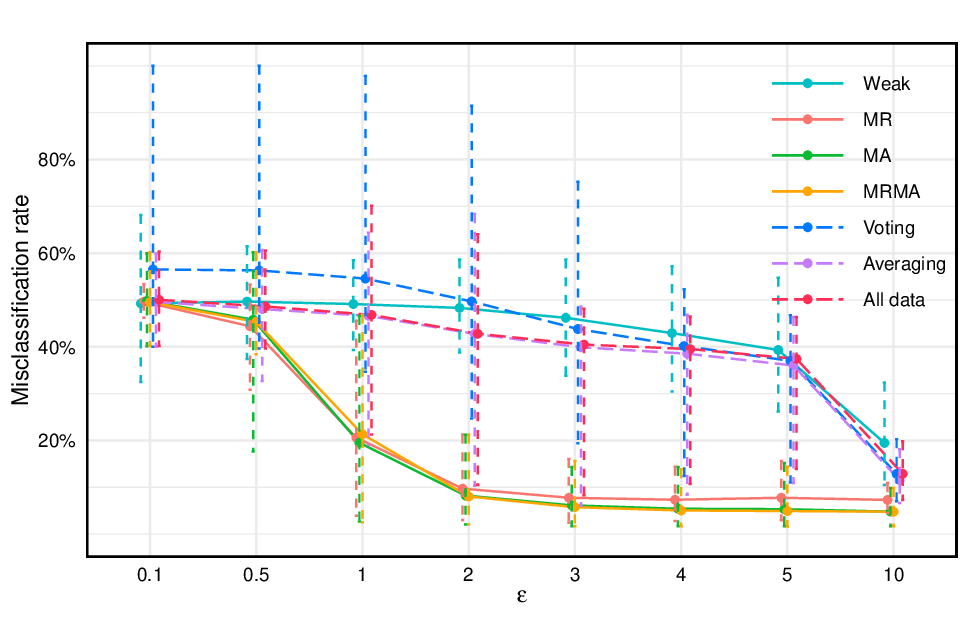}{a}{Logistic}{0.49}\hfill
\subfig{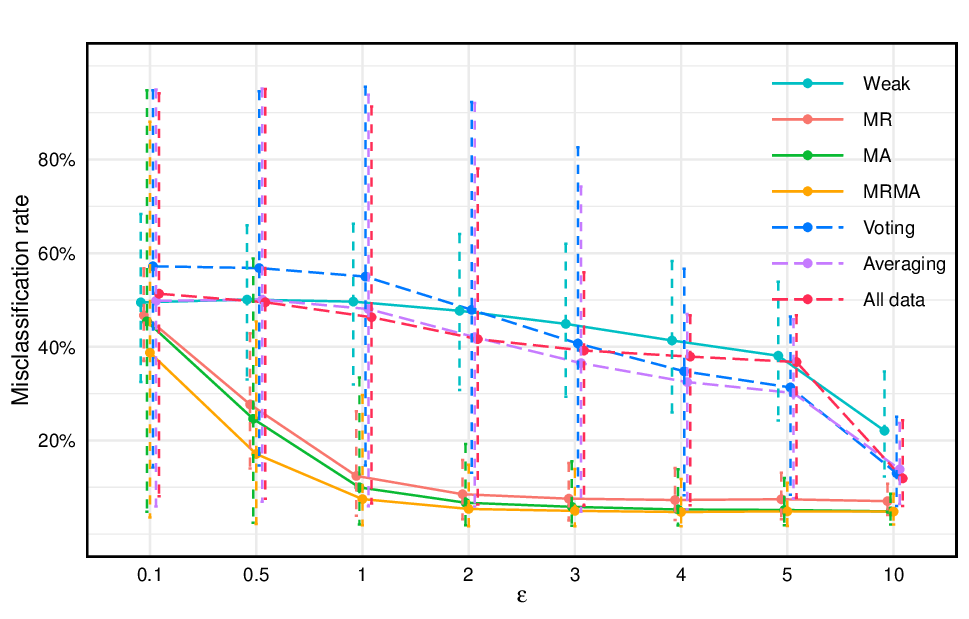}{b}{SVM}{0.49}
\subfig{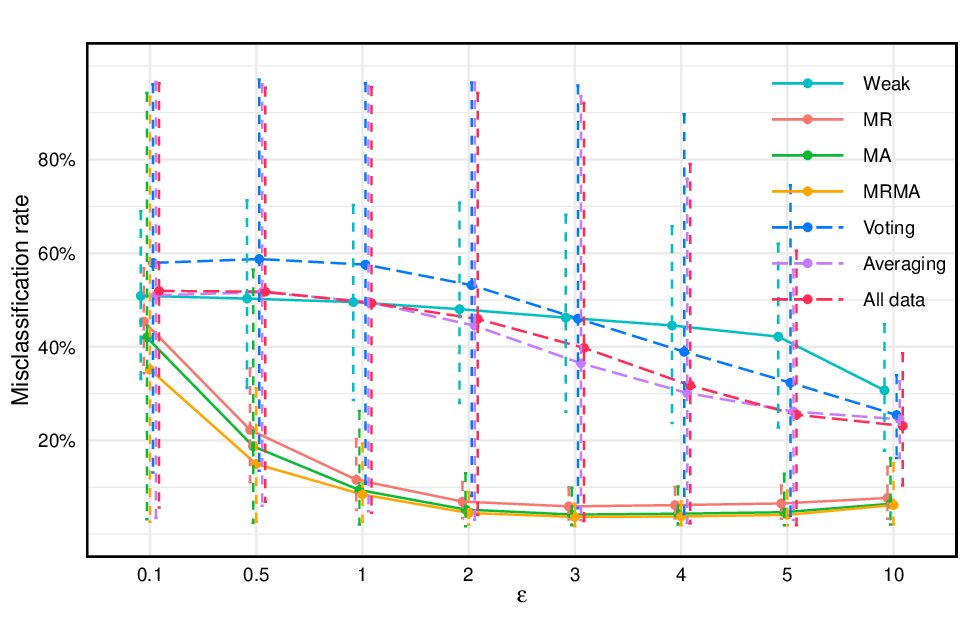}{c}{DWD}{0.49}\hfill
\subfig{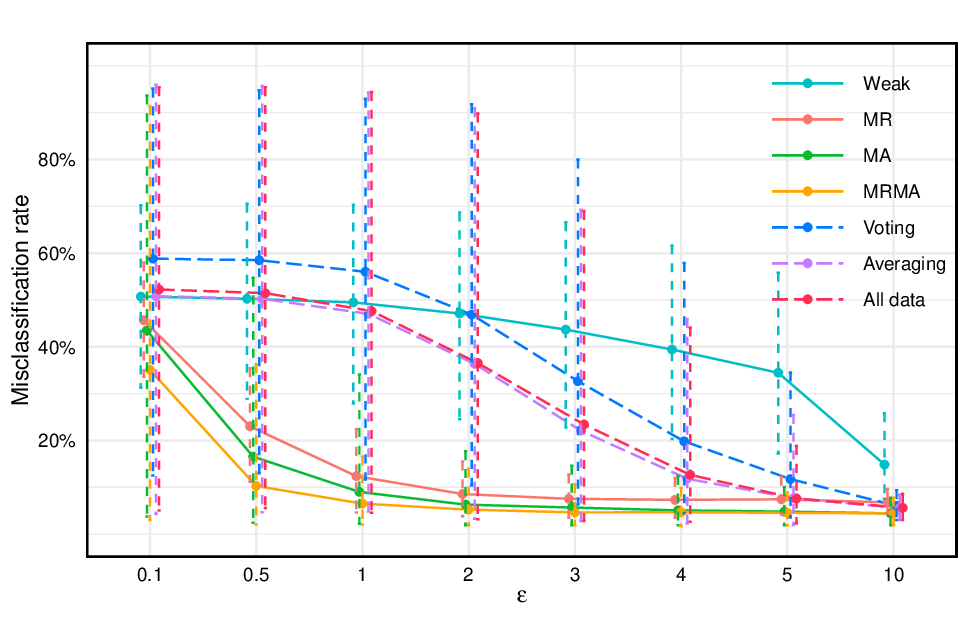}{d}{CG}{0.49}
\caption{The misclassification rates of classifiers with a single server under $\varepsilon$-LDP on the Phonemes dataset.}
\label{fig:real_speech}
\end{figure}

Visualizations of the functional predictors in these two applications are provided in Figures~\ref{fig:realline_mims}-\ref{fig:realline_speech} in the Appendix. These figures display both the raw functional observations and the curves reconstructed after finite basis projection and transformation. Despite some smoothing and loss of fine-grained detail, the figures demonstrate that the key temporal patterns and population-level group differences are well preserved. Figures~\ref{fig:real_mims}-\ref{fig:real_speech} report their misclassification rates. The parameter \( r_0 \) is set to 0.7 for the physical activity dataset and 0.9 for the Phonemes dataset. In both examples, classifiers based on model reversal and model averaging achieve substantial performance gains compared to other classifiers.

\section{Conclusion}\label{sec:conc}
We proposed a novel framework for classification under LDP, addressing key challenges in data utility and model performance. By reinterpreting private learning as a transfer learning problem, we introduced a utility evaluation mechanism based on privatized binary feedback, enabling accurate assessment of classifier performance without accessing unperturbed data. 
Building on this, we developed two techniques: model reversal, which rescues underperforming classifiers by inverting decision boundaries, and model averaging, which combines multiple weak classifiers using utility-based weights. Together, these methods form a robust strategy, MRMA, that significantly enhances classification accuracy under LDP constraints. 
We provided theoretical excess risk bounds demonstrating the effectiveness of these techniques, and showed that the framework generalizes naturally to functional data settings through basis projection. Empirical evaluations on both simulated and real-world datasets confirmed substantial improvements over baseline approaches. 
Future directions include further theoretical analysis of specific classifiers under particular LDP mechanisms, as well as extending the framework to multi-class classification, regression, and other structured data modalities under LDP constraints.





\appendix
\renewcommand{\thefigure}{A.\arabic{figure}}
\renewcommand{\thetable}{A.\arabic{table}}
\setcounter{figure}{0}
\setcounter{table}{0}

\section{Additional Results}
In this section, we provide additional results from our experiments to demonstrate the impacts of encoding and perturbation, compare various sample allocation strategies, and confirm that the performance of weak classifiers improves slightly with an increase in training sample size.

\subsection{Encoding and Perturbation}\label{appsim:single}
In this section, we discuss the effects of dimensionality reduction, rescaling, and perturbation introduced in Section \ref{sec:encoding}, on the misclassification rates of different types of classifiers.

\begin{table}[b!]
\begin{center}
\begin{tabular}{c|c|ccc|ccc|ccc}
& & & $d=4$ & & & $d=5$ & & & $d=6$ & \\
& INI & Coefs & Tanh & MAN & Coefs & Tanh & MAN & Coefs & Tanh & MAN \\ 
\hline \\
Logistic & - & 17.81 & 12.32 & 11.68 & 20.66 & 13.55 & 12.81 & 23.19 & 15.21 & 13.86 \\ 
SVM & - & 10.85 & 11.20 & 11.03 & 11.29 & 11.67 & 11.27 & 11.69 & 12.07 & 11.48 \\ 
DWD & 11.01 & 11.02 & 10.64 & 10.61 & 11.01 & 10.60 & 10.58 & 11.01 & 10.61 & 10.59 \\ 
CG & 15.42 & 11.02 & 11.47 & 11.31 & 11.34 & 11.86 & 11.43 & 11.65 & 12.00 & 11.64 \\
\end{tabular}
\caption{The misclassification rate of classifiers based on actual data (INI), coefficients obtained after dimensionality reduction (Coefs), and coefficients rescaled by either the Tanh transformation (Tanh) or Max-Abs normalization (MAN).}\label{tbl:encoding}
\end{center}
\end{table}

To investigate the impacts of dimensionality reduction and rescaling, we generate data based on the model described in Section \ref{sec:exp}, and the sample sizes of the training and testing data sets are $50$ and $500$, respectively. During the dimensionality reduction, cubic B-spline with equidistant knots are employed, and we apply the methods for varying numbers of basis functions, \(d = 4, 5, 6\). Table \ref{tbl:encoding} showcases the misclassification rates of classifiers based on the actual data (INI), coefficients obtained after dimensionality reduction (Coefs), and coefficients rescaled by either the Tanh transformation (Tanh) or Max-Abs normalization (MAN). This is based on the results from $500$ repeated experiments.

Table \ref{tbl:encoding} shows that there are slight variations in the performance of different types of classifiers. Overall, classifiers with $d=4,5,6$ exhibit comparable results, and the impact of dimension reduction and rescaling on classifier performance is small. Also, the performances based on Tanh and Max-Abs transformations are very similar. Furthermore, the CG classifier based on actual data and the logistic classifier based on coefficients obtained after dimensionality reduction perform relatively poorer. This may be related to the inherent characteristics of the classifiers, which is beyond the scope of this paper. Our primary focus is on the changes in classifier performance before and after consider LDP and the improvements brought about by different techniques.

\begin{figure}[b!]
\centering
\subfig{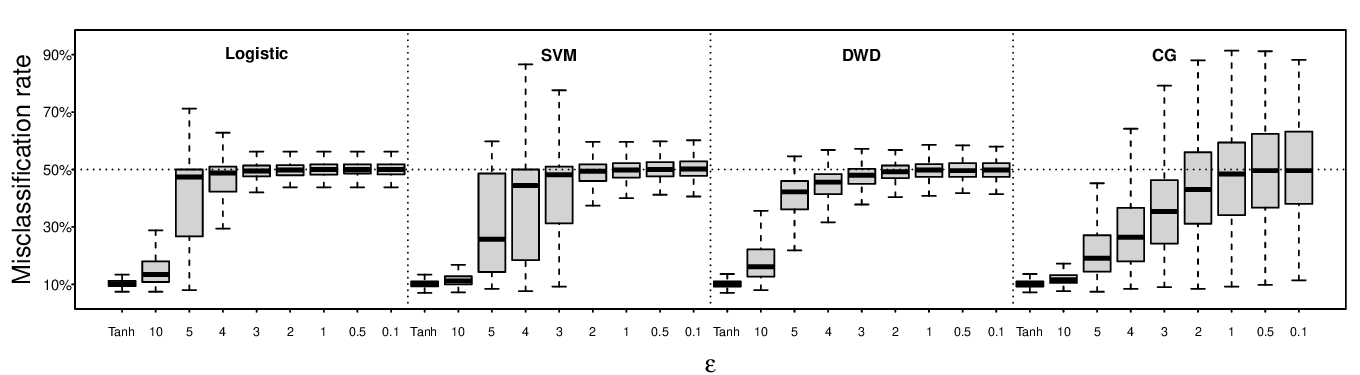}{a}{Tanh transformation}{1}
\subfig{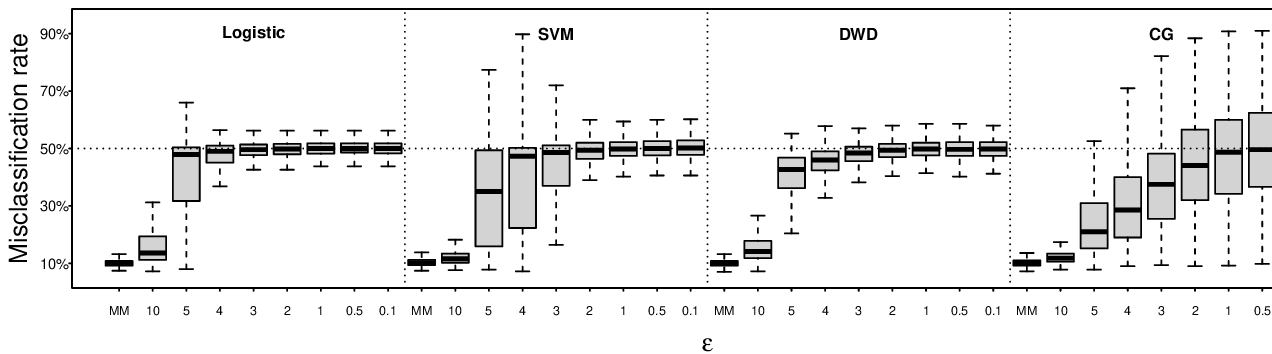}{b}{Max-Abs transformation}{1}
\caption{The boxplot of the misclassification rates of classifiers with Tanh and Max-Abs transformations under $\varepsilon$-LDP.}
\label{fig:pert}
\end{figure}

To further explore the impacts of perturbation, we generate data in accordance with the model in Section \ref{sec:exp}, and both the training and testing data set sample sizes are $500$. Table \ref{tbl:encoding} demonstrates that classifiers with \(d=4,5,6\) exhibit comparable results.Therefore, we employ $d=4$ cubic B-spline with equidistant knots to introduce as little noise as possible. 
We consider eight distinct privacy budget levels, specifically, \( \epsilon = 0.1,0.5,1,\dots,5,10 \). Figure \ref{fig:pert} displays the misclassification rates of the four types of classifiers, based on both un-perturbed rescaled coefficients (i.e., ``Tanh", ``MAN") and perturbed rescaled coefficients across various levels of \( \epsilon \).

\begin{figure}[b!]
\centering
\subfig{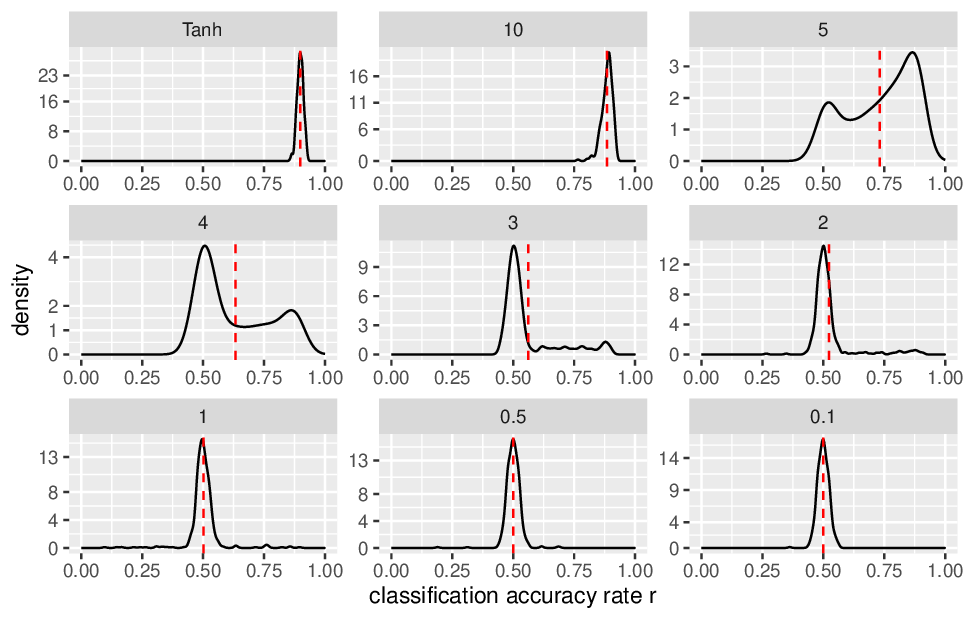}{a}{Logistic}{0.49}\hfill
\subfig{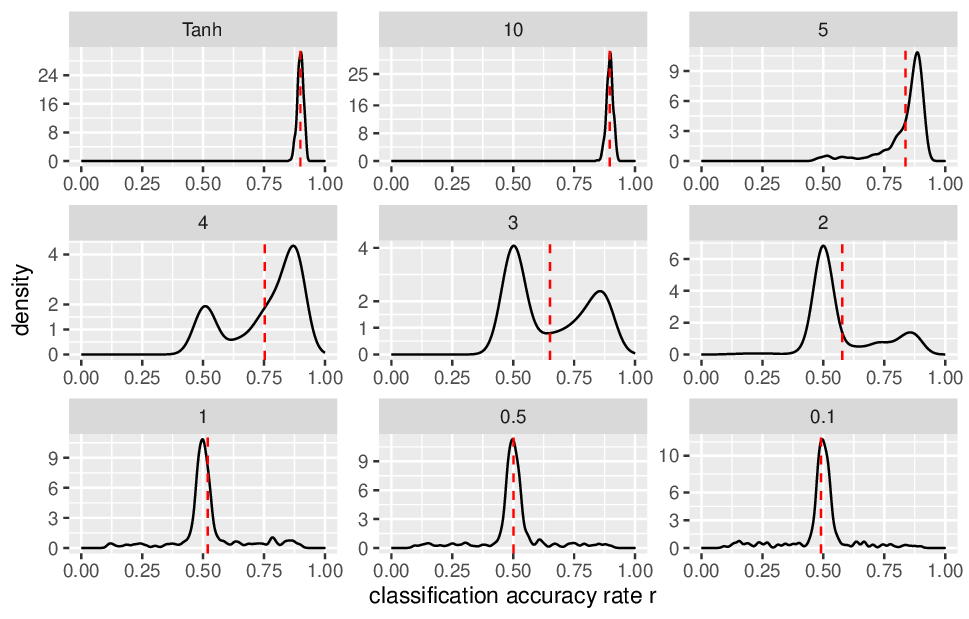}{b}{SVM}{0.49}
\subfig{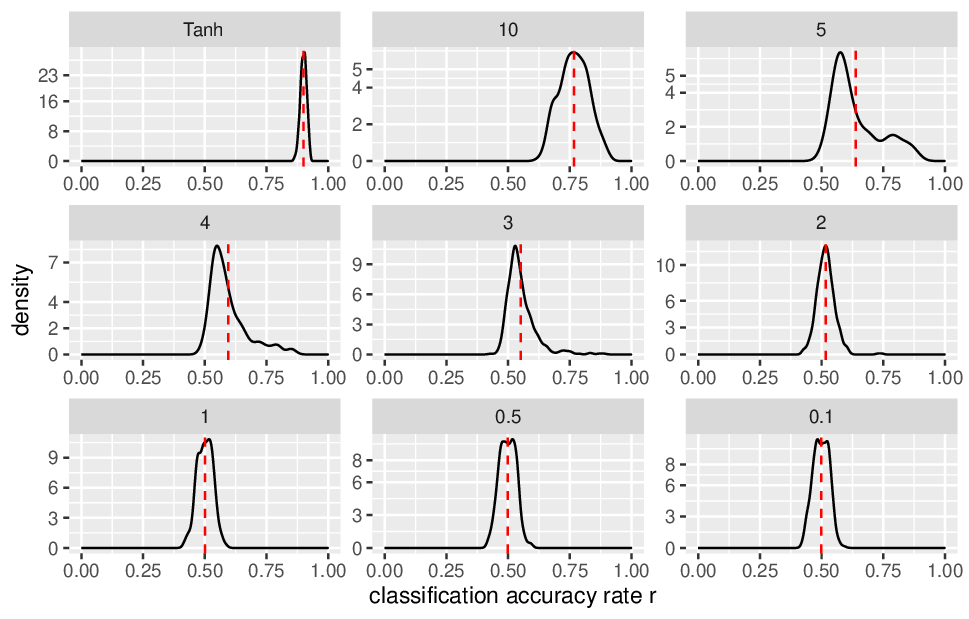}{c}{DWD}{0.49}\hfill
\subfig{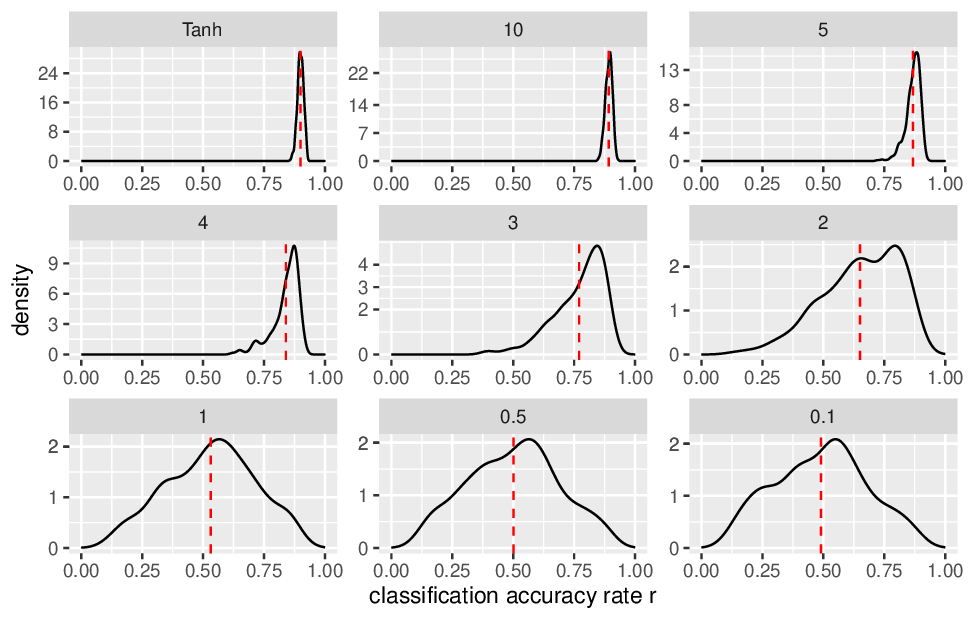}{d}{CG}{0.49}
\caption{The empirical distributions of classification accuracy for classifiers trained with $N=3000$ clients and perturbed by Tanh transformation, where \( \epsilon = 0.1,0.5,1,\dots,5,10 \), and the red dashed line represents the mean accuracy rate.}
\label{fig:density}
\end{figure}

In Figure \ref{fig:pert}, with decreasing $\varepsilon$, the misclassification rates of different types of classifiers tend to 50\%. At the same time, there is a variation in the performance of classifiers under the influence of noise. Notably, the misclassification rate of CG classifier remains highly volatile at smaller $\varepsilon$ values, ranging between 10\% and 90\%, instead of converging around 50\% as other methods do. This indicates that the performance can be enhanced through model reversal and model averaging. Additionally, it can be observed from Figure \ref{fig:pert} that classifiers based on both Tanh and Max-Abs transformations exhibit very similar performances. Only results based on the Tanh transformation will be presented hereafter.

In Theorem~\ref{thm:ermr}, we quantify how model reversal improves the excess risk bound of a classifier. In Figure~\ref{fig:density}, we show the empirical distributions of classification accuracy for the different classifiers under our experimental settings.
Figure \ref{fig:density} illustrates that different classifiers exhibit varying degrees of sensitivity to noise. Among them, the DWD classifier is the most affected, followed by logistic and SVM classifiers. In contrast, the CG classifier is relatively less impacted by noise interference. As \(\varepsilon\) decreases, which corresponds to increased noise, the classification accuracy distributions for logistic, SVM, and DWD classifiers gradually converge around $0.5$. However, the distribution for the CG classifier remains more dispersed, indicating greater potential for improvement through model reversal.

\subsection{Sample Size Balancing}\label{appsim:paras}
In this section, we assess the performance of classifiers over varying values of the parameters \(N, N_0, N_1, n_0, n_1, B\). We generate data based on the settings presented in Section \ref{sec:exp}. Specifically, we consider five distinct combinations of these parameter values, which are listed in Table \ref{tbl:cases}. Figures \ref{fig:ss} and \ref{fig:ssrev} display the misclassification rates of model-averaged classifiers using the cutoff value $r_0=0.6$, with and without model reversal, respectively, for the different parameter combinations.

\begin{table}[h!]
\begin{center}
\begin{tabular}{ccccccc}
Case & $N$ & $N_0$ & $N_1$ & $n_0$ & $n_1$ & $B$  \\ 
\hline \\
1 & 3000 & 500 & 2500 & 100 & 100 & 25 \\
2 & 5500 & 500 & 5000 & 100 & 100 & 50 \\
3 & 5500 & 500 & 5000 & 50 & 100 & 50 \\
4 & 3000 & 500 & 2500 & 100 & 50 & 50 \\
5 & 3000 & 500 & 2500 & 50 & 50 & 50 \\
\end{tabular}
\caption{Five different combinations of parameters \(N, N_0, N_1, n_0, n_1, B\)}\label{tbl:cases}
\end{center}
\end{table}

\begin{figure}[t!]
\centering
\subfig{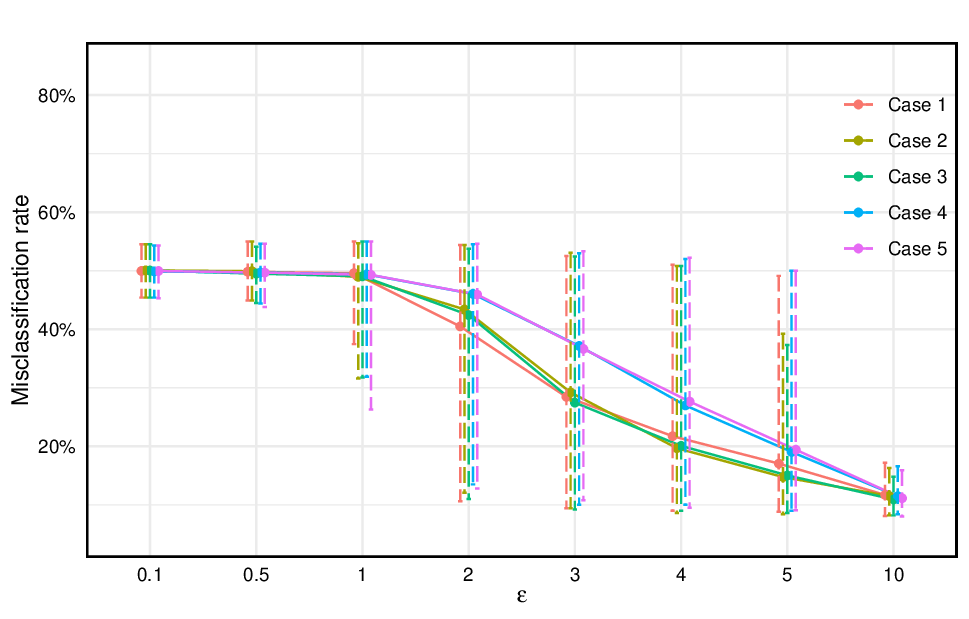}{a}{Logistic}{0.49}\hfill
\subfig{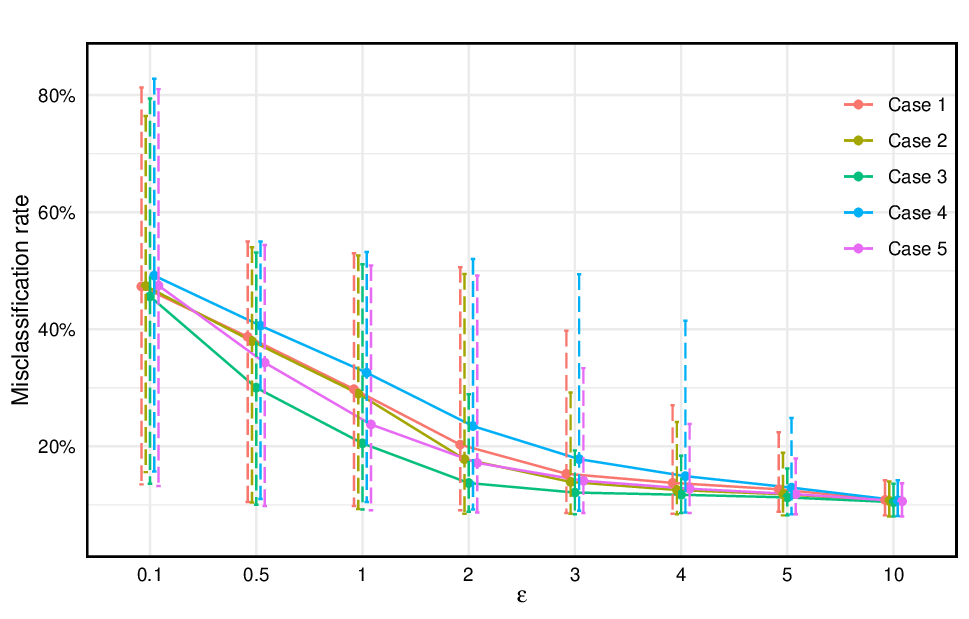}{b}{SVM}{0.49}
\subfig{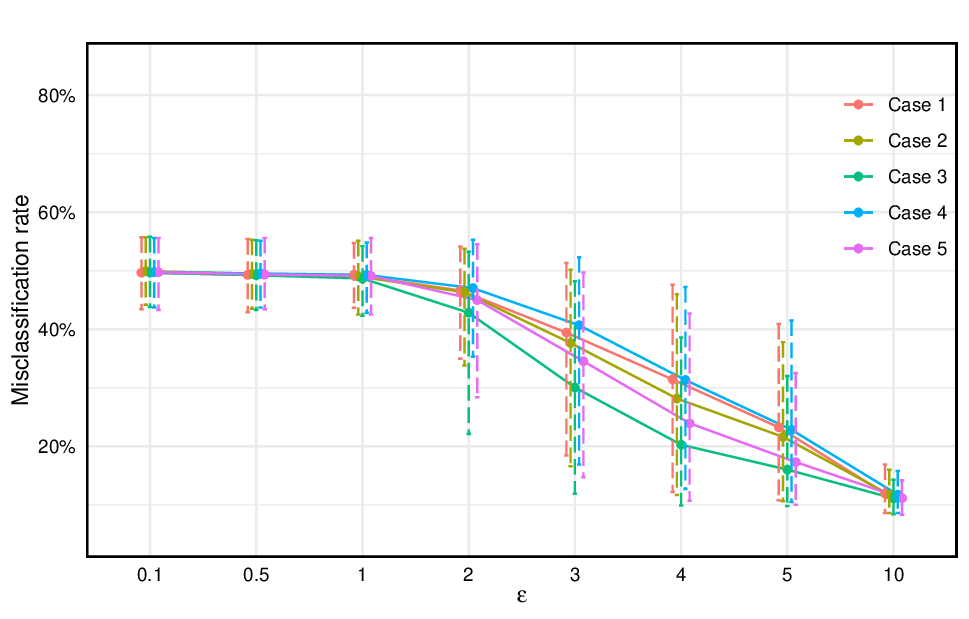}{c}{DWD}{0.49}\hfill
\subfig{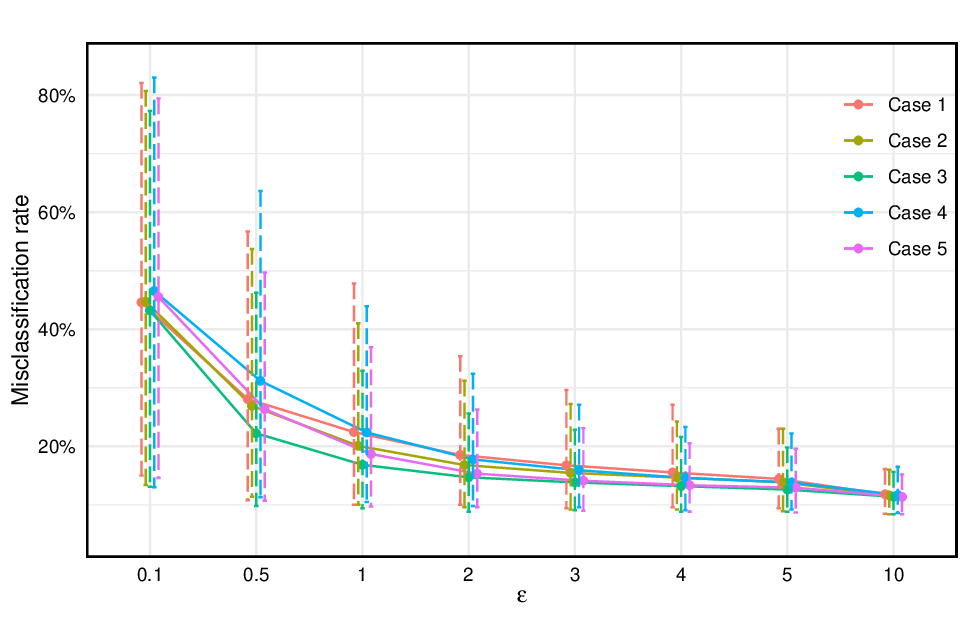}{d}{CG}{0.49}
\caption{The misclassification rates classifiers with model averaging under $\varepsilon$-LDP.}
\label{fig:ss}
\end{figure}

In Figure \ref{fig:ss}, the differences arising from various parameter combinations on different types of classifiers manifest across distinct intervals of $\varepsilon$. Overall, Case 2 performs slightly better than Case 1, indicating that increasing the number of weak classifiers \( B \) and the sample size of the evaluation set \( N_1 \) can enhance the performance. The preference for a relatively smaller \( n_0 \) is evident as Case 3 outperforms Case 2, and Case 5 outperforms Case 4. Additionally, Case 2 is distinctly superior to Case 4, and Case 3 is markedly better than Case 5, indicating a preference for a larger \( n_1 \). Importantly, by comparing Case 1 and Case 4, we discern that, given \( N_0 \) and \( N_1 \), the allocation should favor a relatively smaller \( B \) and a larger \( n_1 \). 
This result is consistent with expectations, as a larger \( n_1 \) aids in more accurately estimating the performance of individual weak classifiers, thereby facilitating a more accurate model averaging.

Compared to Figure \ref{fig:ss}, Figure \ref{fig:ssrev} incorporates model reversal prior to model averaging. It can be observed that the results of the various parameter combinations in Figure \ref{fig:ssrev} are consistent with those in Figure \ref{fig:ss}. Notably, the introduction of model reversal has significantly enhanced the performance of both SVM and CG classifiers, particularly under more stringent privacy protection levels, characterized by smaller intervals of \( \varepsilon \).

\begin{figure}[h!]
\centering
\subfig{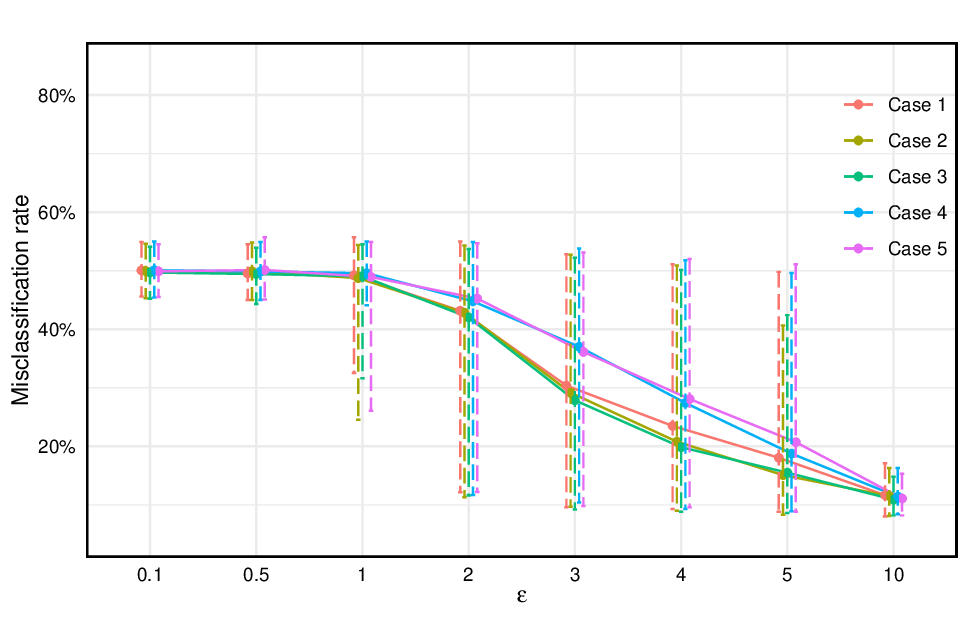}{a}{Logistic}{0.49}\hfill
\subfig{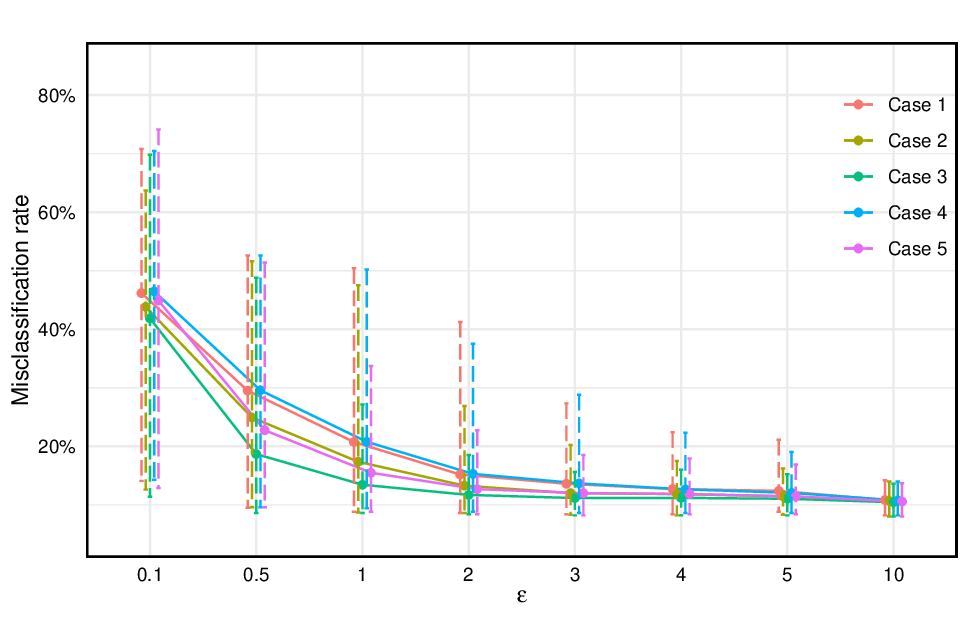}{b}{SVM}{0.49}
\subfig{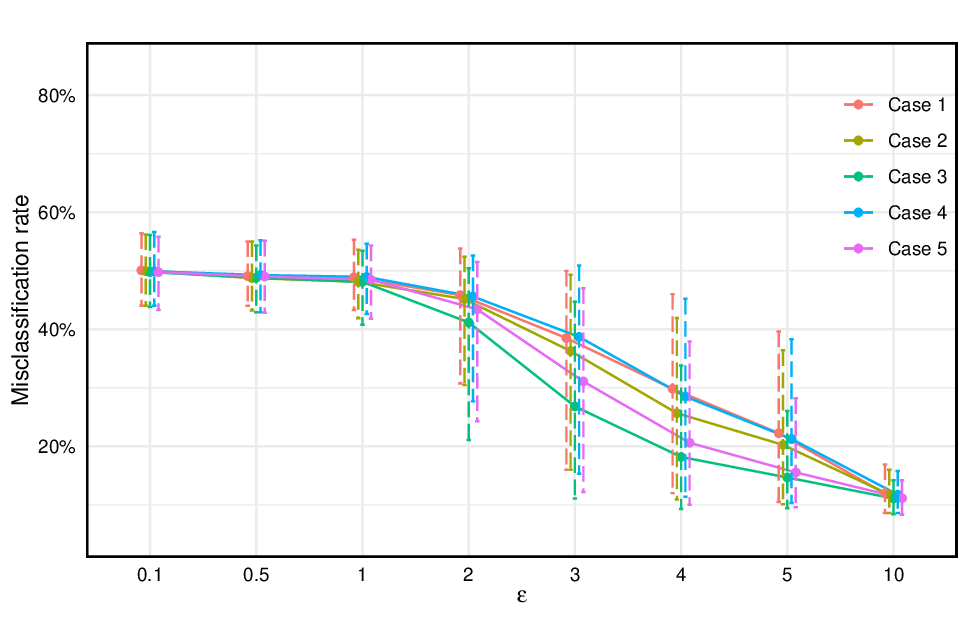}{c}{DWD}{0.49}\hfill
\subfig{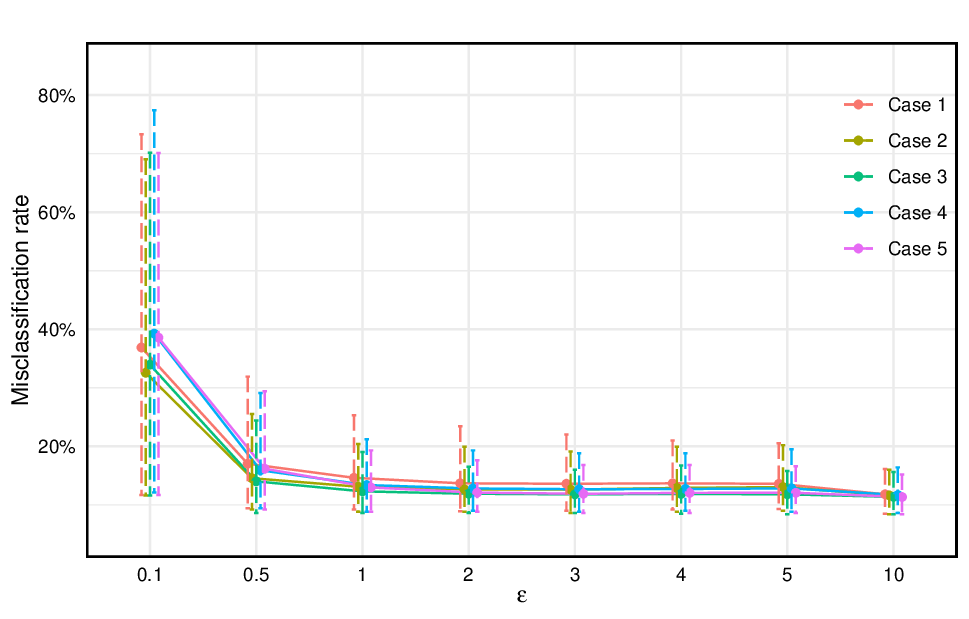}{d}{CG}{0.49}
\caption{The misclassification rates of classifiers with model reversal and model averaging under $\varepsilon$-LDP.}
\label{fig:ssrev}
\end{figure}

\subsection{Sample Size per Weak Classifier}\label{appsim:size}
To demonstrate how the number of clients used to train weak classifiers influences the efficacy of MRMA in improving classifier performance, we consider four distinct settings, which are listed in Table \ref{tbl:cases2}. Figure \ref{fig:size} shows how the misclassification rates of the classifier CG vary with different parameter settings.

We can see from the results of cases 5, 6, and 7 in Figure \ref{fig:size} that using more clients to train weak classifiers improves their performance when \(\varepsilon\) is large, implying low noise levels. However, when \(\varepsilon\) is small, which is our primary concern, increasing the sample size has little effect on the weak classifiers. In contrast, the classifiers based on MRMA consistently achieve significant improvements under different cases. Moreover, when we compare cases 7 and 8 in Figure \ref{fig:size}, where the training and evaluation data sets have different proportions, we find that allocating more data for evaluation, i.e., for MRMA, leads to better results than using it to enhance weak classifiers.

\begin{table}[h!]
\begin{center}
\begin{tabular}{ccccccc}
Case & $N$ & $N_0$ & $N_1$ & $n_0$ & $n_1$ & $B$  \\ 
\hline \\
5 & 3000 & 500 & 2500 & 50 & 50 & 50 \\
6 & 5000 & 2500 & 2500 & 250 & 50 & 50 \\
7 & 7500 & 5000 & 2500 & 500 & 50 & 50 \\
8 & 7500 & 2500 & 5000 & 250 & 100 & 50 \\
\end{tabular}
\caption{Four different combinations of parameters \(N, N_0, N_1, n_0, n_1, B\)}\label{tbl:cases2}
\end{center}
\end{table}

\begin{figure}[h!]
\centering
\includegraphics[width=0.7\textwidth]{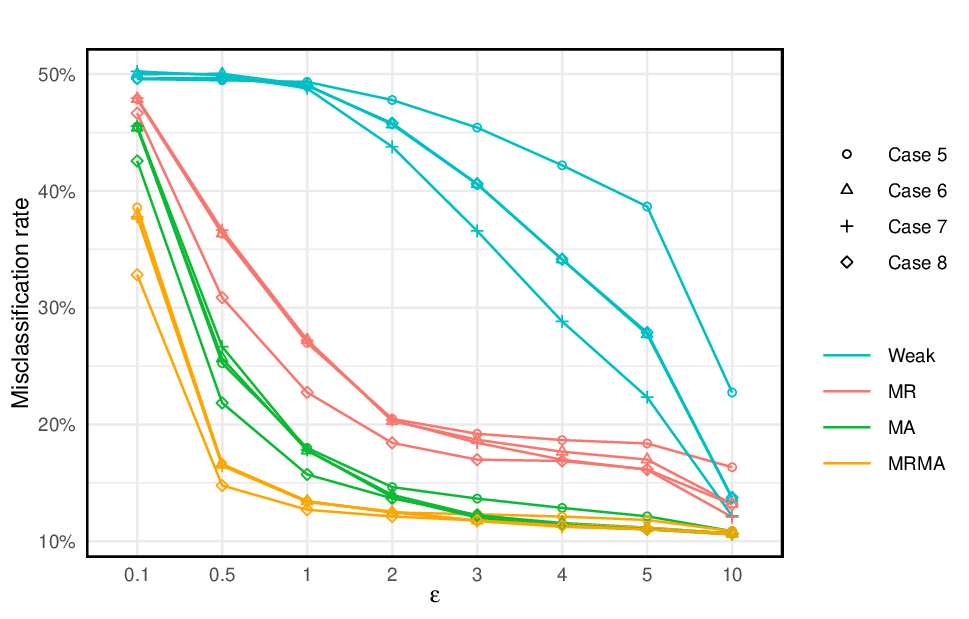}
\caption{The misclassification rates of classifier CG with a single server under $\varepsilon$-LDP.}
\label{fig:size}
\end{figure}

\subsection{Numerical Illustration of the Total Variation Distance}
\label{appsim:dtv_visualization}

The excess risk bounds in Theorems~\ref{thm:erini}-\ref{thm:erma} involve the term $d_{\mathrm{TV}}(P_Z,P_Z^{\ve})$, which measures the discrepancy between the marginal distributions of $Z$ and its privatized version $Z^{\vez}$. 
Closed-form expressions for this total variation distance are generally unavailable, even under simple additive-noise mechanisms. To provide intuition regarding the behavior of this term, we present a simple numerical illustration.

We consider $Z \sim \mathrm{Unif}([-1,1]^d)$ and apply the Laplace mechanism \citep{dwork2006calibrating}, which adds i.i.d.\ noise $\delta_k \sim \mathrm{Laplace}(0,\lambda_z)$ to each coordinate, where $\lambda_z = d(b-a)/\varepsilon_z$ with $[a,b]=[-1,1]$. The distribution of $Z^{\vez} = Z + \delta$ typically has no simple closed form, so we approximate $d_{\mathrm{TV}}(P_Z,P_Z^{\ve})$ via Monte Carlo: the empirical distributions are estimated on a common grid, and the total variation distance is computed as
\[
d_{\mathrm{TV}}(P,Q) 
\approx 
\frac{1}{2}\sum_{k} |\hat{P}_k - \hat{Q}_k|.
\]

Figure~\ref{fig:dtv_laplace} reports the estimated values for several dimensions $d \in \{1,2,5,10\}$ and a range of privacy levels $\varepsilon_z$. The results reveal two intuitive patterns. 
First, as $\varepsilon_z \to 0$ (i.e., stronger privacy and heavier noise), the discrepancy between $P_Z$ and $P_Z^{\ve}$ increases, and the total variation distance approaches~1. This corresponds to the case where the privatized distribution becomes almost disjoint from the original distribution.  
Second, for fixed~$\varepsilon_z$, the distance increases with the dimension~$d$, because the per-coordinate privacy budget effectively decreases as $d$ grows. This is reflected in the scale parameter $\lambda_z = d(b-a)/\varepsilon_z$, which increases with $d$ when the total privacy level $\varepsilon_z$ is fixed, and thus induces larger Laplace noise in each coordinate.

\begin{figure}[h!]
    \centering
    \includegraphics[width=0.75\linewidth]{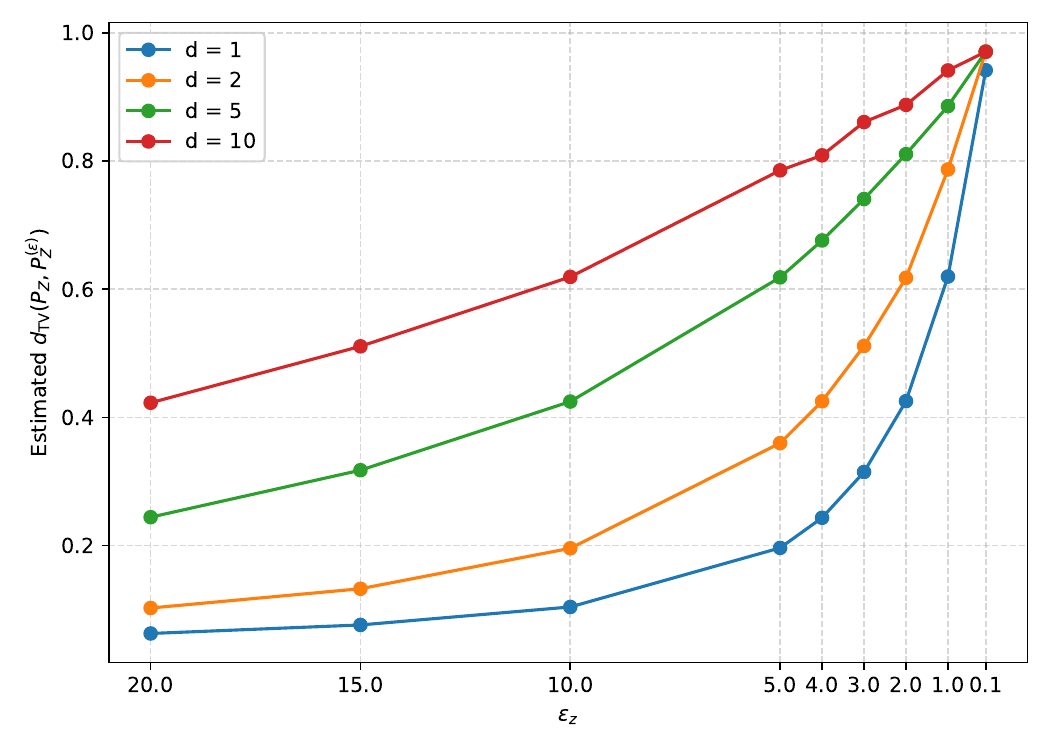}
    \caption{Estimated total variation distance $d_{\mathrm{TV}}(P_Z,P_Z^{\ve})$ under the Laplace mechanism for $Z \sim \mathrm{Unif}([-1,1]^d)$, across dimensions $d \in \{1,2,5,10\}$ and a range of privacy levels~$\varepsilon_z$.}
    \label{fig:dtv_laplace}
\end{figure}

\subsection{Visualization of Functional Predictors}\label{app:real_vis}
Figures~\ref{fig:realline_mims}-\ref{fig:realline_speech} provide visualizations of the functional predictors used in the physical activity and speech datasets, respectively, under different stages of processing: raw functional data, basis projection, and two types of transformations. 
In Figure~\ref{fig:realline_mims}, we display the group mean MIMS curves for participants with healthy (HDL~$\geq$~50) and unhealthy (HDL~$<$~50) cholesterol levels, along with 10 randomly selected individual curves shown in gray. We limit the number of individual curves to reduce visual clutter, as the raw MIMS functions exhibit substantial variability across participants. 
Figure~\ref{fig:realline_speech} presents 200 randomly selected individual log-periodogram curves from the speech dataset along with the group means. Despite some smoothing and loss of fine-scale detail, these visualizations demonstrate that both temporal fluctuations and population-level differences are preserved throughout the functional representation pipeline.

\begin{figure}[h!]
\centering
\subfig{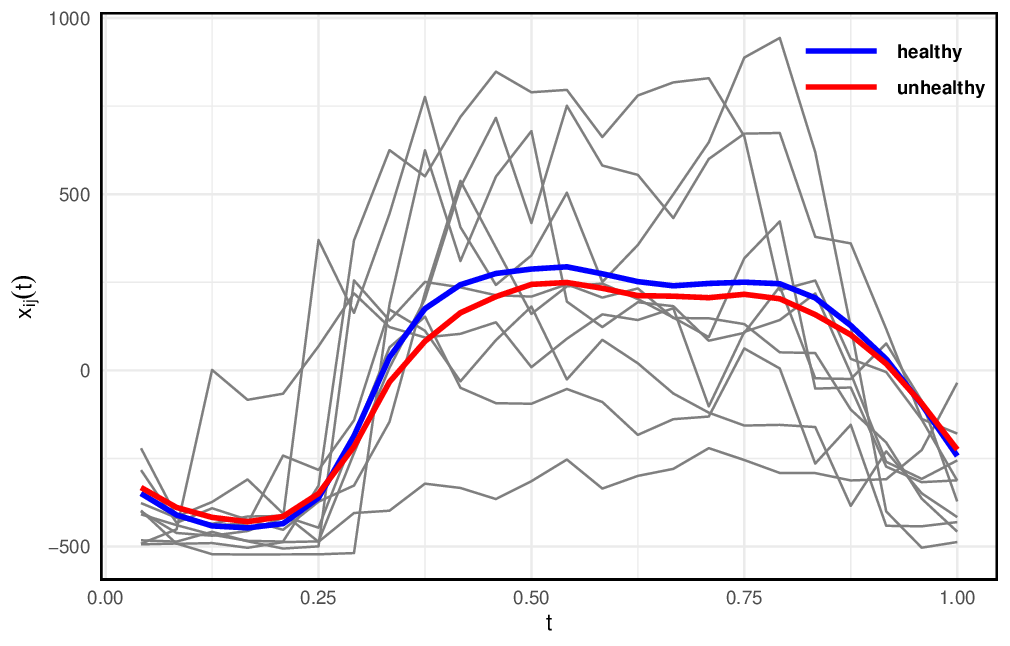}{a}{Raw functional observations}{0.49}\hfill
\subfig{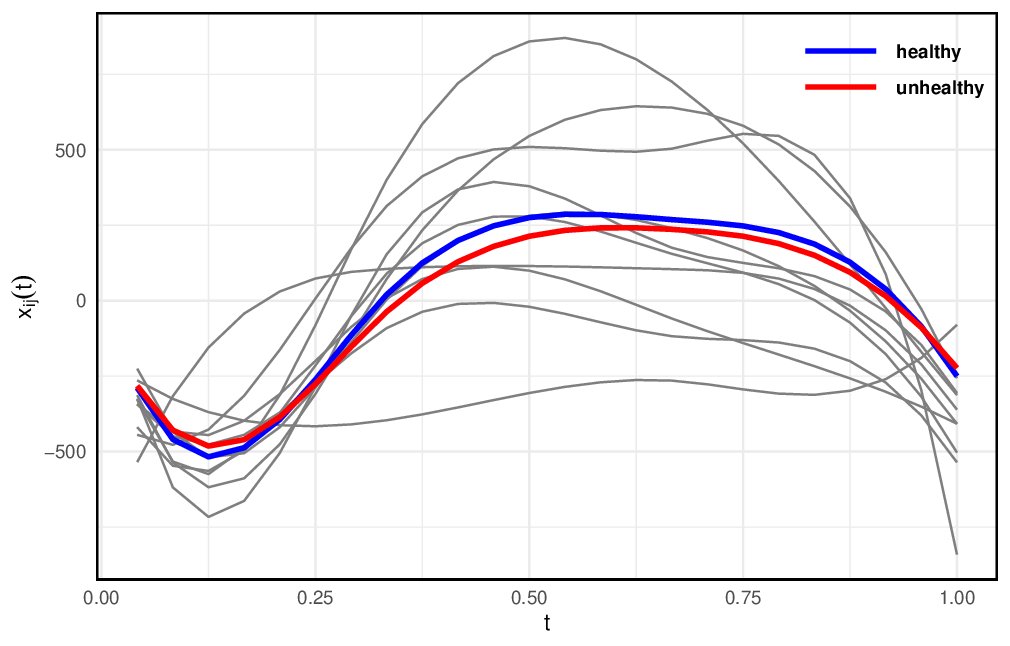}{b}{Reconstruction after projection}{0.49}
\subfig{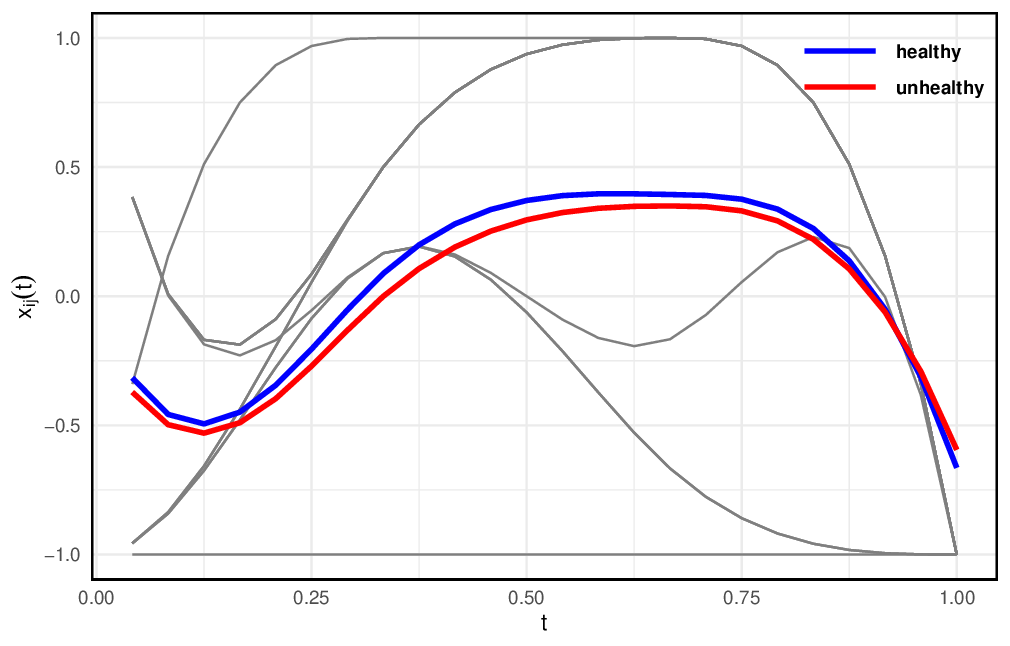}{d}{Reconstruction after projection and Tanh transformation}{0.485}
\subfig{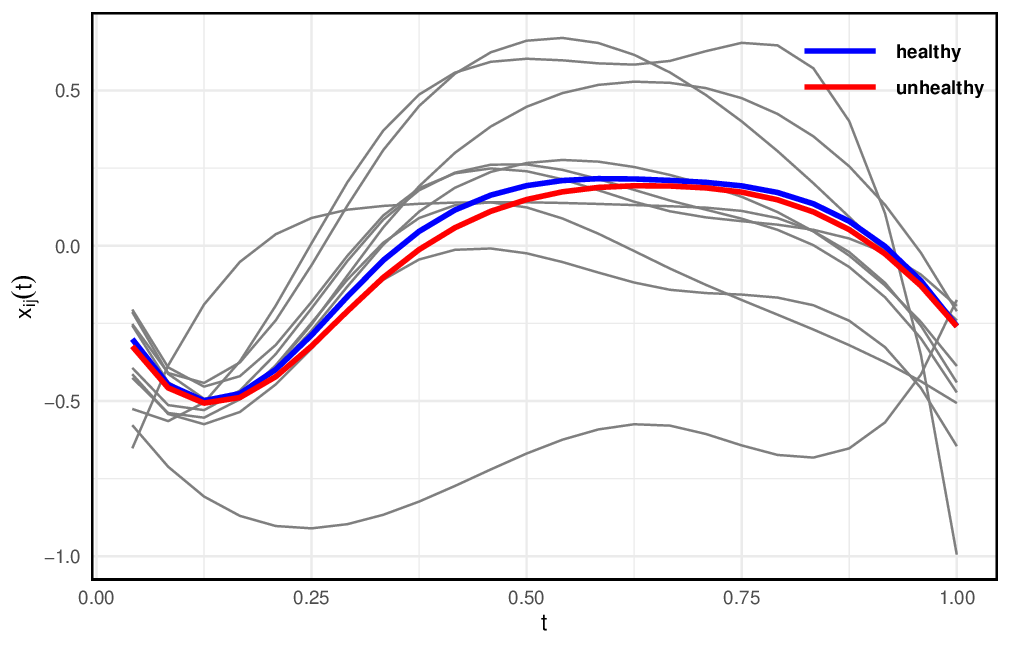}{c}{Reconstruction after projection and Max-Abs transformation}{0.485}\hfill
\caption{Visualization of the mean MIMS curves for the ``healthy'' (HDL~$\geq$~50) and ``unhealthy'' (HDL~$<$~50) groups in the physical activity dataset, with 10 randomly selected individual curves overlaid in gray. Each subplot shows the curves under different stages of processing the raw functional data.}
\label{fig:realline_mims}
\end{figure}

\begin{figure}[h!]
\centering
\subfig{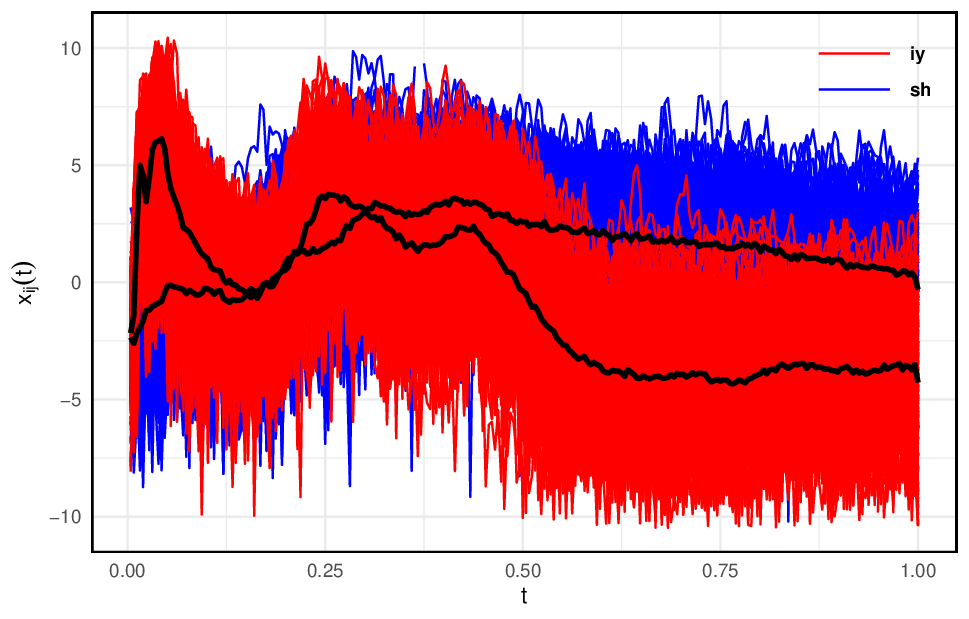}{a}{Raw functional observations}{0.49}\hfill
\subfig{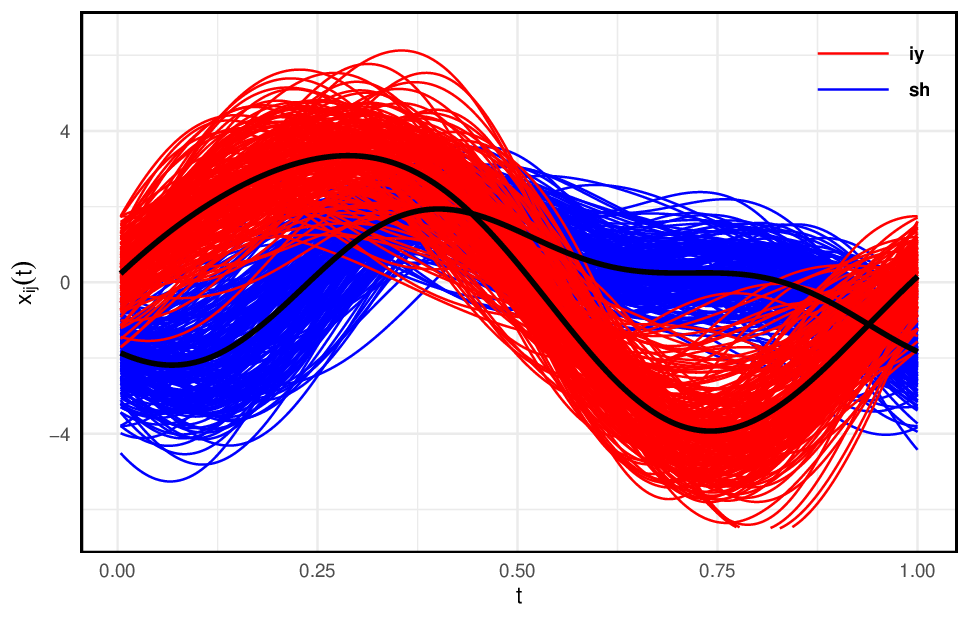}{b}{Reconstruction after projection}{0.49}
\subfig{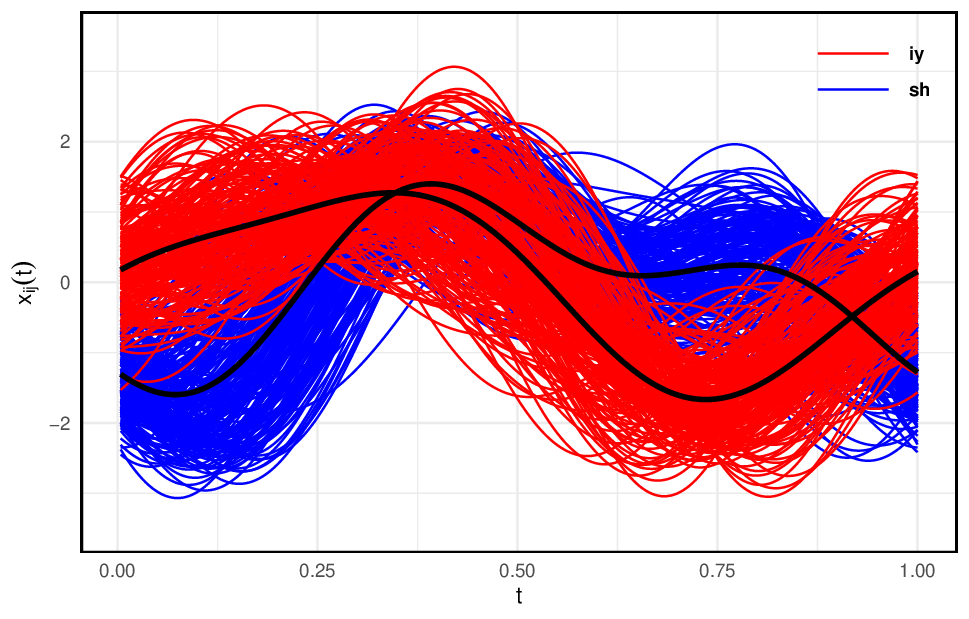}{d}{Reconstruction after projection and Tanh transformation}{0.485}
\subfig{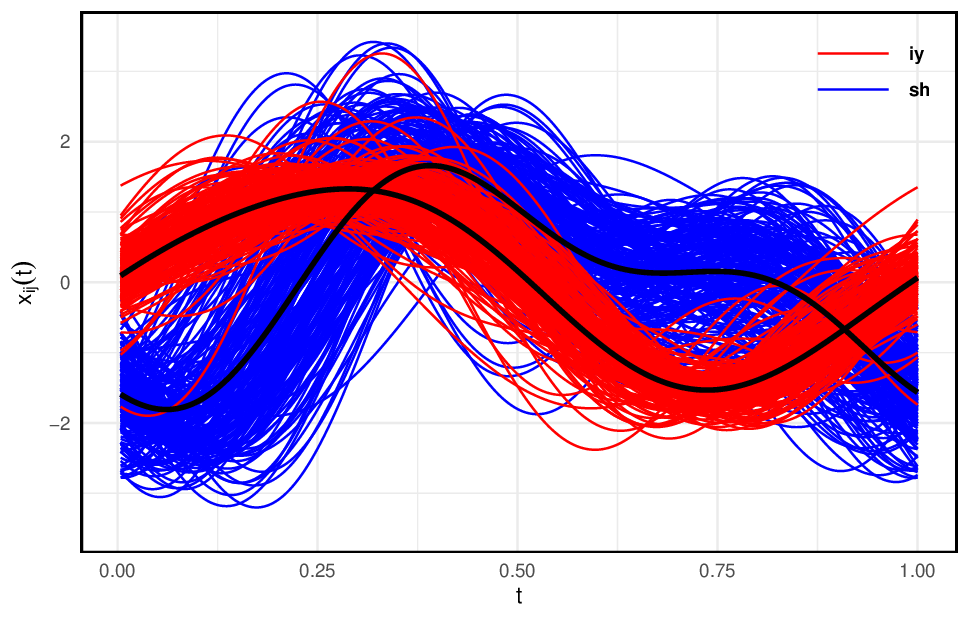}{c}{Reconstruction after projection and Max-Abs transformation}{0.485}\hfill
\caption{Visualization of 200 randomly selected functional observations from groups ``sh" and ``iy" in the phonemes data set. The two black lines in each figure represent the mean functions of the two groups of curves.}
\label{fig:realline_speech}
\end{figure}

\section{Theoretical Analysis}
\begin{proof}[of Proposition \ref{prop:gfun}]
We begin by noting that $g^{\ve}(z_0) = \eta(z_0)\eta^{\ve}(z_0) + (1 - \eta(z_0))(1 - \eta^{\ve}(z_0))$, which yields the first expression in Equation~\eqref{eq:gfun}. 

To compute $\eta^{\ve}(z_0) - \frac{1}{2}$, we have that
\begin{equation*}
    \eta^{\ve}(z_0) = \frac{P^{\ve}(Z^{\vez} = z_0, Y^{\vey} = 1)}{P_{Z}^{\ve}(Z^{\vez} = z_0)},
\end{equation*}
where
\begin{align*}
    P_{Z}^{\ve}(Z^{\vez} = z_0) &= \sE_{P_Z} \left(D^{\vez}(\delta_z^{\vez} = z_0 - z)\right), \\
    P^{\ve}(Z^{\vez} = z_0, Y^{\vey} = 1) &= \sE_{P_Z} \left(D^{\vez}(\delta_z^{\vez} = z_0 - z) \left(q^{\vey}\eta(z) + (1 - q^{\vey})(1 - \eta(z))\right)\right).
\end{align*}
After some algebraic manipulation, we arrive at the expression given in Equation~\eqref{eq:etainfo}.
\end{proof}

\begin{proof}[of Theorem \ref{thm:ldp2}]
    Let $u=(x,y)$ represent the possible observation of a client, and $w=r'$ represent the possible report of a client. Let $f_{w|u}(w|u)$ be the conditional density of $w$ given $u$. Then for any $w\in\{0,1\}$, we have
\begin{align*}
    \frac{f_{w|u}(w|u_1)}{f_{w|u}(w|u_2)}=\frac{P(r'|r_1)}{P(r'|r_2)}\le e^{\varepsilon_v}.
\end{align*}
Thus $\gM$ satisfies $\varepsilon_v$-local differential privacy. Note that $r^{(b)}$ is the classification accuracy of the classifier $f^{(b)}$, then
\begin{equation*}
    \sE(\hat r^{(b)})=qr^{(b)}+(1-q)(1-r^{(b)}),
\end{equation*}
where $q=e^{\varepsilon_v}/(1+e^{\varepsilon_v})$. Thus we have $\sE(\tilde r^{(b)})=r^{(b)}$, and $\text{Var}(\tilde r^{(b)})=\text{Var}(\hat r^{(b)})/(2q-1)^2 \le ((e^{\varepsilon_v}+1)/(e^{\varepsilon_v}-1))^2/(4n_1^{(b)})$. This concludes the proof.
\end{proof}

\begin{proof}[of Theorem \ref{thm:erini}] 
By definition, we have
\begin{align*}
L_P(f^{(b)}, h^*) 
=\, & \sE_P\{\ell(f^{(b)}) - \ell(h^*)\} \\
=\, & \{\sE_P \ell(f^{(b)}) - \sE_{P^{\ve}} \ell(f^{(b)})\} 
+ \sE_{P^{\ve}}\{\ell(f^{(b)}) - \ell(h^{*\ve})\} 
+ \{\sE_{P^{\ve}} \ell(h^{*\ve}) - \sE_P \ell(h^*)\} \\
\triangleq\, & I_1 + I_2 + I_3.
\end{align*}

To analyze \( I_1 = \sE_P \ell(f^{(b)}) - \sE_{P^{\ve}} \ell(f^{(b)}) \), let \( p(z) \) and \( p^{\ve}(z) \) denote the density functions of \( Z \) and \( Z^{\ve} \), respectively. Define
\[
p_0(z) = \min\{p(z), p^{\ve}(z)\}, \quad
p_1(z) = p(z) - p_0(z), \quad
p^{\ve}_1(z) = p^{\ve}(z) - p_0(z).
\]
Then,
\begin{align*}
I_1 =\, & \sE_{P_Z}\left\{\eta(z)\ell(f^{(b)} \mid y = 1) + (1 - \eta(z))\ell(f^{(b)} \mid y = -1)\right\} \\
& - \sE_{P^{\vez}_Z}\left\{\eta^{\ve}(z)\ell(f^{(b)} \mid y = 1) + (1 - \eta^{\ve}(z))\ell(f^{(b)} \mid y = -1)\right\} \\
=\, & \int p_0(z) \left\{\eta(z) - \eta^{\ve}(z)\right\} 
\left\{\ell(f^{(b)} \mid y = 1) - \ell(f^{(b)} \mid y = -1)\right\} dz \\
& + \int p_1(z) \ell(f^{(b)}) dz - \int p^{\ve}_1(z) \ell(f^{(b)}) dz \\
\le\, & \sE_{P_Z}\left\{|\eta(z) - \eta^{\ve}(z)|\right\} + \int p_1(z) dz.
\end{align*}

Next, consider \( I_3 = \sE_{P^{\ve}} \ell(h^{*\ve}) - \sE_P \ell(h^*) \). Note that
\begin{align*}
\sE_P \ell(h^*) 
=\, & \sE_{P_Z}\left\{I\{\eta(z) \ge 1/2\} P_{Y|Z}(y = -1) + I\{\eta(z) < 1/2\} P_{Y|Z}(y = 1)\right\} \\
=\, & \sE_{P_Z}\left\{I\{\eta(z) \ge 1/2\} (1 - \eta(z)) + I\{\eta(z) < 1/2\} \eta(z)\right\} \\
=\, & \sE_{P_Z}\left\{\frac{1}{2} - \left|\eta(z) - \frac{1}{2}\right|\right\},
\end{align*}
and similarly,
\[
\sE_{P^{\ve}} \ell(h^{*\ve}) = \sE_{P^{\ve}_Z}\left\{\frac{1}{2} - \left|\eta^{\ve}(z) - \frac{1}{2}\right|\right\}.
\]
Thus,
\begin{align*}
I_3 =\, & \sE_{P_Z}\left\{|\eta(z) - \tfrac{1}{2}|\right\} - \sE_{P^{\ve}_Z}\left\{|\eta^{\ve}(z) - \tfrac{1}{2}|\right\} \\
=\, & \int p_0(z) \left\{|\eta(z) - \tfrac{1}{2}| - |\eta^{\ve}(z) - \tfrac{1}{2}|\right\} dz \\
& + \int p_1(z) |\eta(z) - \tfrac{1}{2}| dz - \int p^{\ve}_1(z) |\eta^{\ve}(z) - \tfrac{1}{2}| dz \\
\le\, & \sE_{P_Z}\left\{|\eta(z) - \tfrac{1}{2}| - |\eta^{\ve}(z) - \tfrac{1}{2}|\right\} + \frac{1}{2} \int p_1(z) dz.
\end{align*}

By the definition of total variation distance, we have 
\[
d_{\text{TV}}(P_Z, P^{\ve}_Z) = 2 \int p_1(z) dz.
\]
Combining the bounds for \( I_1 \) and \( I_3 \), we obtain the result stated in the theorem.
\end{proof}

\begin{proof}[of Theorem \ref{thm:ermr}]
Consider the noised dataset $\{(z_i^{\vez}, y_i^{\vey})\}_{i=1}^{n_0}$ and the classifier $f^{(b)}$ trained on it. We define a pseudo-noised dataset $\{(z_i^{\vez}, \tilde y_i^{\vey})\}_{i=1}^{n_0}$ where $\tilde y_i^{\vey} = -y_i^{\vey}$, and correspondingly define $\tilde f^{(b)} = -f^{(b)}$. Then $\tilde f^{(b)}$ can be interpreted as the classifier trained on this pseudo-noised dataset.

Let $\tilde P^{\ve}$ denote the joint distribution of $(Z^{\vez}, \tilde Y^{\vey})$. Then the marginal distribution of $Z^{\vez}$ under $\tilde P^{\ve}$ remains the same as that under $P^{\ve}$, i.e., $\tilde P^{\ve}_Z = P^{\ve}_Z$. The conditional distribution satisfies $\tilde P^{\ve}_{Y|Z} = 1 - P^{\ve}_{Y|Z}$, and the corresponding Bayes classifier becomes $\tilde h^{*\ve}(z) = \text{sign}\{\tilde\eta^{\ve}(z) \ge 1/2\} = -h^{*\ve}(z)$, where $\tilde\eta^{\ve}(z)=1-\eta^{\ve}(z)$.

Applying Theorem~\ref{thm:erini}, we obtain
\begin{align*}
L_P(-f^{(b)}, h^*) = L_P(\tilde f^{(b)}, h^*) \le\ & L_{\tilde P^{\ve}}(\tilde f^{(b)}, \tilde h^{*\ve}) 
+ \frac{3}{4} d_{\text{TV}}(P_Z, \tilde P^{\ve}_Z) 
+ \sE_{P_Z} \left(|\eta(z) - \tilde \eta^{\ve}(z)|\right) \\
& + \sE_{P_Z} \left(|\eta(z) - \tfrac{1}{2}| - |\tilde \eta^{\ve}(z) - \tfrac{1}{2}|\right) \\
=\ & L_{P^{\ve}}(f^{(b)}, h^{*\ve}) 
+ \frac{3}{4} d_{\text{TV}}(P_Z, P^{\ve}_Z) 
+ \sE_{P_Z} \left(|\eta(z) - \tilde \eta^{\ve}(z)|\right) \\
& + \sE_{P_Z} \left(|\eta(z) - \tfrac{1}{2}| - |\eta^{\ve}(z) - \tfrac{1}{2}|\right).
\end{align*}

Furthermore, we note that
\[
\min\left\{
\sE_{P_Z} \left(|\eta(z) - \tilde \eta^{\ve}(z)|\right), 
\sE_{P_Z} \left(|\eta(z) - \eta^{\ve}(z)|\right)
\right\} 
= \sE_{P_Z} \left(|\eta(z) - \tfrac{1}{2}| - |\eta^{\ve}(z) - \tfrac{1}{2}|\right).
\]
This implies
\begin{align*}
\min\left\{L_P(f^{(b)}, h^*), L_P(-f^{(b)}, h^*)\right\} \le\ &
L_{P^{\ve}}(f^{(b)}, h^{*\ve}) 
+ \frac{3}{4} d_{\text{TV}}(P_Z, P^{\ve}_Z) \\
& + 2 \sE_{P_Z} \left(|\eta(z) - \tfrac{1}{2}| - |\eta^{\ve}(z) - \tfrac{1}{2}|\right).
\end{align*}

Let \( r^{(b)} \) denote the true classification accuracy of the Bayes classifier \( h^{*\ve} \) under the distribution \( P \), and let \( \tilde r^{(b)} \) denote the estimated classification accuracy of \( f^{(b)} \) based on \( n_1 \) evaluation samples. Define the selected classifier \( f^{*(b)} \in \{f^{(b)}, -f^{(b)}\} \) as the one with higher estimated accuracy: specifically, \( f^{*(b)} = f^{(b)} \) if \( \hat{c}_{n_1} \ge 1/2 \), and \( f^{*(b)} = -f^{(b)} \) otherwise. 
Then, the probability that model reversal correctly selects the better classifier satisfies
\begin{align*}
\sP\left[L_P(f^{*(b)}, h^*) = \min\left\{L_P(f^{(b)}, h^*), L_P(-f^{(b)}, h^*)\right\}\right] 
=\, & \sP\left\{\left(\tilde r^{(b)} - \tfrac{1}{2}\right)\left(r^{(b)} - \tfrac{1}{2}\right) \ge 0\right\} \\
\approx\, & \Phi\left\{ \sqrt{n_1} \frac{|r^{(b)} - 1/2|}{\sqrt{r^{(b)}(1 - r^{(b)})}} \right\},
\end{align*}
where \( \Phi \) is the cumulative distribution function of the standard normal distribution.
\end{proof}

\begin{proof}[of Theorem \ref{thm:erma}]
Let \( f^\dagger(z) = \sum_{b=1}^B w_b f^{*(b)}(z) \) be the model-averaged classifier, where each \( f^{*(b)} \in \{f^{(b)}, -f^{(b)}\} \) is selected via model reversal to maximize estimated accuracy, and the weights \( \{w_b\}_{b=1}^B \) sum to 1. 
Let \( \tilde r_z^{*(b)} = \sP\{\text{sign}[f^{*(b)}(z)] = y\} \) denote the classification accuracy of the \( b \)-th classifier at point \( z \). Let \( F_z^\ve \) be the distribution of \( \tilde r_z^{*(b)} \) under random training and evaluation sets, with support \([r_{z,0}, r_{z,1}]\). Define the aggregated mapping
\[
\eta^{\dagger\ve}(z) := \sum_{b=1}^B w_b \eta^{*(b)}(z),
\]
where \( \eta^{*(b)}(z) := \sP\{Y = 1 \mid Z=z, f^{*(b)}\} \) is the latent regression function implied by classifier \( f^{*(b)} \). Under standard assumptions, the classification accuracy of \( f^{*(b)}(z) \) is
\[
\tilde r_z^{*(b)} = \frac{1}{2} + \left| \eta^{*(b)}(z) - \frac{1}{2} \right|.
\]

As \( B \to \infty \), the empirical average of classifiers concentrates around those with higher accuracy, due to the weight truncation scheme that sets \( w_b = 0 \) if \( \tilde r_z^{*(b)} < r_0 \). Since the condition \( B_0/B \to 0 \) ensures that only a vanishing fraction of classifiers fall below the cutoff \( r_0 \), the weighted average function \( \eta^{\dagger\ve}(z) \) satisfies
\[
\left| \eta^{\dagger\ve}(z) - \frac{1}{2} \right| \;\overset{\sP}{\longrightarrow}\; r_{z,1} - \frac{1}{2}.
\]

Now, recall from Theorem~\ref{thm:ermr} that for each classifier \( f^{*(b)} \), the excess risk is bounded by
\[
L_P(f^{*(b)}, h^*) \le L_{P^{\ve}}(f^{*(b)}, h^{*\ve}) + \frac{3}{4} d_{\text{TV}}(P_Z, P_Z^{\ve}) + 2\,\sE_{P_Z}\left( \left| \eta(z) - \tfrac{1}{2} \right| - \left| \eta^{*(b)}(z) - \tfrac{1}{2} \right| \right).
\]

Taking a convex combination over \( b \), and using the fact that the minimum excess risk over \(\{f^{*(b)}\}\) is preserved in expectation under model averaging, we obtain
\begin{align*}
L_P(f^\dagger, h^*) 
&\le L_{P^{\ve}}(f^\dagger, h^{*\ve}) + \frac{3}{4} d_{\text{TV}}(P_Z, P_Z^{\ve}) \\
&\quad + 2\,\sE_{P_Z}\left( \left| \eta(z) - \tfrac{1}{2} \right| - \left| \eta^{\dagger,\ve}(z) - \tfrac{1}{2} \right| \right).
\end{align*}

As discussed, \( \left| \eta^{\dagger,\ve}(z) - \tfrac{1}{2} \right| \to r_{z,1} - \tfrac{1}{2} \) in probability, which establishes the theorem.
\end{proof}

\begin{proof}[of Theorem \ref{thm:ldp1}]
Let $u=(x(t),y)$ represent the possible observation of a client, and $v=(\bz^\vez,y^\vey)$ represent the possible report of a client. Let $f_{v|u}(v|u)$ be the conditional density of $v$ given $u$. Then for any possible output  $v$ of $\gM_f(u)$, by the sequential composition theorem \citep{mcsherry2007mechanism}, we have
\begin{align*}
    \frac{f_{v|u}(v|u_1)}{f_{v|u}(v|u_2)}=&\frac{P(y^\vey|y_1)}{P(y^\vey|y_2)}\prod_{k=1}^d\frac{f(z_k^\vez-z_{1,k}^*)}{f(z_k^\vez-z_{2,k}^*)}\\
    \le&\max(1,\frac{q}{1-q},\frac{1-q}{q})\prod_{k=1}^d\exp(\frac{\vz}{d\Delta}(|z_k^\vez-z_{1,k}^*|-|z_k^\vez-z_{2,k}^*|))\\
    \le&e^{\vy}\prod_{k=1}^d \exp(\frac{\vz}{d})=e^{\vz+\vy}=e^{\varepsilon},
\end{align*}
where $q=e^{\vy}/(1+e^{\vy})$. Thus $\gM_f$ satisfies $\varepsilon$-local differential privacy.
\end{proof}

\vskip 0.2in
\bibliography{refs}

@inproceedings{truex2020ldp,
  title={LDP-Fed: Federated learning with local differential privacy},
  author={Truex, Stacey and Liu, Ling and Chow, Ka-Ho and Gursoy, Mehmet Emre and Wei, Wenqi},
  booktitle={Proceedings of the third ACM international workshop on edge systems, analytics and networking},
  pages={61--66},
  year={2020}
}

@inproceedings{sun2020ldp,
  title={LDP-FL: Practical private aggregation in federated learning with local differential privacy},
  author={Sun, Lichao and Qian, Jianwei and Chen, Xun},
  booktitle = {Proceedings of the Thirtieth International Joint Conference on
               Artificial Intelligence, {IJCAI-21}},
  publisher = {International Joint Conferences on Artificial Intelligence Organization},
  editor    = {Zhi-Hua Zhou},
  pages     = {1571--1578},
  year      = {2021}
}

@article{li2024update,
  title={An update on measurement error modeling},
  author={Li, Mushan and Ma, Yanyuan},
  journal={Annual Review of Statistics and Its Application},
  volume={11},
  year={2024},
  publisher={Annual Reviews}
}

@article{svm2008,
  title={Statistical performance of support vector machines},
  author={Blanchard, Gilles and Bousquet, Olivier and Massart, Pascal},
  journal={The Annals of Statistics},
  volume={36},
  number={2},
  pages={489--531},
  year={2008},
  publisher={Institute of Mathematical Statistics}
}

@article{hall2001functional,
  title={A functional data—analytic approach to signal discrimination},
  author={Hall, Peter and Poskitt, Donald S and Presnell, Brett},
  journal={Technometrics},
  volume={43},
  number={1},
  pages={1--9},
  year={2001},
  publisher={Taylor \& Francis}
}

@article{dai2017optimal,
  title={Optimal Bayes classifiers for functional data and density ratios},
  author={Dai, Xiongtao and M{\"u}ller, Hans-Georg and Yao, Fang},
  journal={Biometrika},
  volume={104},
  number={3},
  pages={545--560},
  year={2017},
  publisher={Oxford University Press}
}

@article{leng2006classification,
  title={Classification using functional data analysis for temporal gene expression data},
  author={Leng, Xiaoyan and M{\"u}ller, Hans-Georg},
  journal={Bioinformatics},
  volume={22},
  number={1},
  pages={68--76},
  year={2006},
  publisher={Oxford University Press}
}

@article{saravanan2014review,
  title={Review on classification based on artificial neural networks},
  author={Saravanan, Kl and Sasithra, S},
  journal={Int J Ambient Syst Appl},
  volume={2},
  number={4},
  pages={11--18},
  year={2014}
}

@article{kotsiantis2007supervised,
  title={Supervised machine learning: A review of classification techniques},
  author={Kotsiantis, Sotiris B and Zaharakis, Ioannis and Pintelas, P and others},
  journal={Emerging artificial intelligence applications in computer engineering},
  volume={160},
  number={1},
  pages={3--24},
  year={2007},
  publisher={Amsterdam}
}

@article{kumari2017machine,
  title={Machine learning: A review on binary classification},
  author={Kumari, Roshan and Srivastava, Saurabh Kr},
  journal={International Journal of Computer Applications},
  volume={160},
  number={7},
  year={2017},
  publisher={Foundation of Computer Science}
}

@inproceedings{yilmaz2020naive,
  title={Naive Bayes classification under local differential privacy},
  author={Yilmaz, Emre and Al-Rubaie, Mohammad and Chang, J Morris},
  booktitle={2020 IEEE 7th International Conference on Data Science and Advanced Analytics (DSAA)},
  pages={709--718},
  year={2020},
  organization={IEEE}
}

@article{berrett2019classification,
  title={Classification under local differential privacy},
  author={Berrett, Thomas and Butucea, Cristina},
  journal={arXiv preprint arXiv:1912.04629},
  year={2019}
}

@inproceedings{blitzer2006domain,
  title={Domain adaptation with structural correspondence learning},
  author={Blitzer, John and McDonald, Ryan and Pereira, Fernando},
  booktitle={Proceedings of the 2006 conference on empirical methods in natural language processing},
  pages={120--128},
  year={2006}
}

@article{pan2009survey,
  title={A survey on transfer learning},
  author={Pan, Sinno Jialin and Yang, Qiang},
  journal={IEEE Transactions on knowledge and data engineering},
  volume={22},
  number={10},
  pages={1345--1359},
  year={2009},
  publisher={IEEE}
}

@article{qin2024adaptive,
    author = {Qin, Caihong and Xie, Jinhan and Li, Ting and Bai, Yang},
    title = {An Adaptive Transfer Learning Framework for Functional Classification},
    journal = {Journal of the American Statistical Association},
    volume = {0},
    number = {ja},
    pages = {1--22},
    year = {2024},
    doi = {10.1080/01621459.2024.2403788}
}

@article{ma2024optimal,
  title={Optimal Locally Private Nonparametric Classification with Public Data},
  author={Ma, Yuheng and Yang, Hanfang},
  journal={Journal of Machine Learning Research},
  volume={25},
  number={167},
  pages={1--62},
  year={2024}
}

@inproceedings{dwork2006calibrating,
  title={Calibrating noise to sensitivity in private data analysis},
  author={Dwork, Cynthia and McSherry, Frank and Nissim, Kobbi and Smith, Adam},
  booktitle={Theory of Cryptography: Third Theory of Cryptography Conference, TCC 2006, New York, NY, USA, March 4-7, 2006. Proceedings 3},
  pages={265--284},
  year={2006},
  organization={Springer}
}

@article{warner1965randomized,
  title={Randomized response: A survey technique for eliminating evasive answer bias},
  author={Warner, Stanley L},
  journal={Journal of the American Statistical Association},
  volume={60},
  number={309},
  pages={63--69},
  year={1965},
  publisher={Taylor \& Francis}
}

@article{kasiviswanathan2011can,
  title={What can we learn privately?},
  author={Kasiviswanathan, Shiva Prasad and Lee, Homin K and Nissim, Kobbi and Raskhodnikova, Sofya and Smith, Adam},
  journal={SIAM Journal on Computing},
  volume={40},
  number={3},
  pages={793--826},
  year={2011},
  publisher={SIAM}
}

@book{ramsay2005,
  title={Functional data analysis},
  author={Ramsay, James O and Silverman, Bernhard W},
  year={2005},
  publisher={Springer, New York}
}

@inproceedings{mcsherry2007mechanism,
  title={Mechanism design via differential privacy},
  author={McSherry, Frank and Talwar, Kunal},
  booktitle={48th Annual IEEE Symposium on Foundations of Computer Science (FOCS'07)},
  pages={94--103},
  year={2007},
  organization={IEEE}
}

@article{li2022transfer,
  title={Transfer learning for high-dimensional linear regression: Prediction, estimation and minimax optimality},
  author={Li, Sai and Cai, T Tony and Li, Hongzhe},
  journal={Journal of the Royal Statistical Society Series B: Statistical Methodology},
  volume={84},
  number={1},
  pages={149--173},
  year={2022},
  publisher={Oxford University Press}
}

@article{kraus2019classification,
  title={Classification of functional fragments by regularized linear classifiers with domain selection},
  author={Kraus, David and Stefanucci, Marco},
  journal={Biometrika},
  volume={106},
  number={1},
  pages={161--180},
  year={2019},
  publisher={Oxford University Press}
}

@article{sang2022reproducing,
  title={A Reproducing Kernel Hilbert Space Framework for Functional Classification},
  author={Sang, Peijun and Kashlak, Adam B and Kong, Linglong},
  journal={Journal of Computational and Graphical Statistics},
  pages={1--9},
  year={2022},
  publisher={Taylor \& Francis}
}

@article{yang2020local,
  title={Local differential privacy and its applications: A comprehensive survey},
  author={Yang, Mengmeng and Lyu, Lingjuan and Zhao, Jun and Zhu, Tianqing and Lam, Kwok-Yan},
  journal={arXiv preprint arXiv:2008.03686},
  year={2020}
}

@inproceedings{ye2020local,
  title={Local differential privacy: Tools, challenges, and opportunities},
  author={Ye, Qingqing and Hu, Haibo},
  booktitle={International conference on web information systems engineering},
  pages={13--23},
  year={2020},
  organization={Springer}
}

@article{duchi2018minimax,
  title={Minimax optimal procedures for locally private estimation},
  author={Duchi, John C and Jordan, Michael I and Wainwright, Martin J},
  journal={Journal of the American Statistical Association},
  volume={113},
  number={521},
  pages={182--201},
  year={2018},
  publisher={Taylor \& Francis}
}

@inproceedings{liu2020fedsel,
  title={Fedsel: Federated sgd under local differential privacy with top-k dimension selection},
  author={Liu, Ruixuan and Cao, Yang and Yoshikawa, Masatoshi and Chen, Hong},
  booktitle={Database Systems for Advanced Applications: 25th International Conference, DASFAA 2020, Jeju, South Korea, September 24--27, 2020, Proceedings, Part I 25},
  pages={485--501},
  year={2020},
  organization={Springer}
}

@inproceedings{wang2019collecting,
  title={Collecting and analyzing multidimensional data with local differential privacy},
  author={Wang, Ning and Xiao, Xiaokui and Yang, Yin and Zhao, Jun and Hui, Siu Cheung and Shin, Hyejin and Shin, Junbum and Yu, Ge},
  booktitle={2019 IEEE 35th International Conference on Data Engineering (ICDE)},
  pages={638--649},
  year={2019},
  organization={IEEE}
}

@article{lin2022transfer,
  title={Transfer learning for functional linear regression with structural interpretability},
  author={Lin, Haotian and Reimherr, Matthew},
  journal={arXiv preprint arXiv:2206.04277},
  year={2022}
}

@article{tian2022transfer,
  title={Transfer learning under high-dimensional generalized linear models},
  author={Tian, Ye and Feng, Yang},
  journal={Journal of the American Statistical Association},
  volume={0},
  number={0},
  pages={1--14},
  year={2022},
  publisher={Taylor \& Francis}
}

@incollection{torrey2010transfer,
  title={Transfer learning},
  author={Torrey, Lisa and Shavlik, Jude},
  booktitle={Handbook of research on machine learning applications and trends: algorithms, methods, and techniques},
  pages={242--264},
  year={2010},
  publisher={IGI global}
}

@article{cai2021transfer,
author = {Cai, T Tony and Wei, Hongji},
title = {Transfer learning for nonparametric classification: Minimax rate and adaptive classifier},
volume = {49},
journal = {The Annals of Statistics},
number = {1},
publisher = {Institute of Mathematical Statistics},
pages = {100 -- 128},
year = {2021}
}

@article{weiss2016survey,
  title={A survey of transfer learning},
  author={Weiss, Karl and Khoshgoftaar, Taghi M and Wang, DingDing},
  journal={Journal of Big data},
  volume={3},
  number={1},
  pages={1--40},
  year={2016},
  publisher={SpringerOpen}
}

@article{wang2019locally,
  title={Locally private high-dimensional crowdsourced data release based on copula functions},
  author={Wang, Teng and Yang, Xinyu and Ren, Xuebin and Yu, Wei and Yang, Shusen},
  journal={IEEE Transactions on Services Computing},
  volume={15},
  number={2},
  pages={778--792},
  year={2019},
  publisher={IEEE}
}

@article{nguyen2016collecting,
  title={Collecting and analyzing data from smart device users with local differential privacy},
  author={Nguy{\^e}n, Th{\^o}ng T and Xiao, Xiaokui and Yang, Yin and Hui, Siu Cheung and Shin, Hyejin and Shin, Junbum},
  journal={arXiv preprint arXiv:1606.05053},
  year={2016}
}

@inproceedings{arcolezi2021random,
  title={Random sampling plus fake data: Multidimensional frequency estimates with local differential privacy},
  author={Arcolezi, H{\'e}ber H and Couchot, Jean-Fran{\c{c}}ois and Al Bouna, Bechara and Xiao, Xiaokui},
  booktitle={Proceedings of the 30th ACM International Conference on Information \& Knowledge Management},
  pages={47--57},
  year={2021}
}

@article{arcolezi2022improving,
  title={Improving the utility of locally differentially private protocols for longitudinal and multidimensional frequency estimates},
  author={Arcolezi, H{\'e}ber H and Couchot, Jean-Fran{\c{c}}ois and Al Bouna, Bechara and Xiao, Xiaokui},
  journal={Digital Communications and Networks},
  year={2022},
  publisher={Elsevier}
}

@online{appleprivacy,
    author    = {Differential Privacy Team, Apple},
    title     = {Learning with Privacy at Scale},
    year      = {2017},
    url       = {https://docs-assets.developer.apple.com/ml-research/papers/learning-with-privacy-at-scale.pdf},
    note      = {Accessed: 20 Sep 2023}
}

@inproceedings{erlingsson2014rappor,
  title={Rappor: Randomized aggregatable privacy-preserving ordinal response},
  author={Erlingsson, {\'U}lfar and Pihur, Vasyl and Korolova, Aleksandra},
  booktitle={Proceedings of the 2014 ACM SIGSAC conference on computer and communications security},
  pages={1054--1067},
  year={2014}
}

@article{ding2017collecting,
  title={Collecting telemetry data privately},
  author={Ding, Bolin and Kulkarni, Janardhan and Yekhanin, Sergey},
  journal={Advances in Neural Information Processing Systems},
  volume={30},
  year={2017}
}

@article{horvath2015,
  title={An introduction to functional data analysis and a principal component approach for testing the equality of mean curves},
  author={Horv{\'a}th, Lajos and Rice, Gregory},
  journal={Revista Matem{\'a}tica Complutense},
  volume={28},
  number={3},
  pages={505--548},
  year={2015},
  publisher={Springer}
}

@article{pomann2016,
  title={A two sample distribution-free test for functional data with application to a diffusion tensor imaging study of multiple sclerosis},
  author={Pomann, Gina-Maria and Staicu, Ana-Maria and Ghosh, Sujit},
  journal={Journal of the Royal Statistical Society. Series C, Applied statistics},
  volume={65},
  number={3},
  pages={395},
  year={2016},
  publisher={NIH Public Access}
}

@book{eubank1999nonparametric,
  title={Nonparametric regression and spline smoothing},
  author={Eubank, Randall L},
  year={1999},
  publisher={CRC press}
}

@article{garofolo1993timit,
  title={Timit acoustic phonetic continuous speech corpus},
  author={Garofolo, John S},
  journal={Linguistic Data Consortium, 1993},
  year={1993}
}

@article{hall2013differential,
  title={Differential privacy for functions and functional data},
  author={Hall, Rob and Rinaldo, Alessandro and Wasserman, Larry},
  journal={The Journal of Machine Learning Research},
  volume={14},
  number={1},
  pages={703--727},
  year={2013},
  publisher={JMLR. org}
}

@inproceedings{mirshani2019formal,
  title={Formal privacy for functional data with gaussian perturbations},
  author={Mirshani, Ardalan and Reimherr, Matthew and Slavkovi{\'c}, Aleksandra},
  booktitle={International Conference on Machine Learning},
  pages={4595--4604},
  year={2019},
  organization={PMLR}
}

@inproceedings{jiang2023functional,
  title={Functional Renyi Differential Privacy for Generative Modeling},
  author={Jiang, Dihong and Sun, Sun and Yu, Yaoliang},
  booktitle={Thirty-seventh Conference on Neural Information Processing Systems},
  year={2023}
}

@article{lin2023differentially,
  title={Differentially Private Functional Summaries via the Independent Component Laplace Process},
  author={Lin, Haotian and Reimherr, Matthew},
  journal={arXiv preprint arXiv:2309.00125},
  year={2023}
}

@book{horvath2012inference,
  title={Inference for functional data with applications},
  author={Horv{\'a}th, Lajos and Kokoszka, Piotr},
  volume={200},
  year={2012},
  publisher={Springer Science \& Business Media}
}

@inproceedings{stisen2015smart,
  title={Smart devices are different: Assessing and mitigatingmobile sensing heterogeneities for activity recognition},
  author={Stisen, Allan and Blunck, Henrik and Bhattacharya, Sourav and Prentow, Thor Siiger and Kj{\ae}rgaard, Mikkel Baun and Dey, Anind and Sonne, Tobias and Jensen, Mads M{\o}ller},
  booktitle={Proceedings of the 13th ACM conference on embedded networked sensor systems},
  pages={127--140},
  year={2015}
}

@article{Bai2023,
  author    = {Bai, Yang and Qin, Caihong and Zhu, Huichen},
  title     = {Functional Two-Sample Test Based on Projection},
  journal   = {Statistica Sinica Preprint},
  year      = {2023},
  volume    = {},
  number    = {},
  pages     = {},
  doi       = {10.5705/ss.202023.0272},
  url       = {http://www.stat.sinica.edu.tw/statistica/}
}

@article{ben2006analysis,
  title={Analysis of representations for domain adaptation},
  author={Ben-David, Shai and Blitzer, John and Crammer, Koby and Pereira, Fernando},
  journal={Advances in neural information processing systems},
  volume={19},
  year={2006}
}

@misc{ko2023excessriskconvergencerates,
      title={On Excess Risk Convergence Rates of Neural Network Classifiers}, 
      author={Hyunouk Ko and Namjoon Suh and Xiaoming Huo},
      year={2023},
      eprint={2309.15075},
      archivePrefix={arXiv},
      primaryClass={stat.ML},
      url={https://arxiv.org/abs/2309.15075}, 
}

@article{liu2022gdp,
  title={Gdp vs. ldp: A survey from the perspective of information-theoretic channel},
  author={Liu, Hai and Peng, Changgen and Tian, Youliang and Long, Shigong and Tian, Feng and Wu, Zhenqiang},
  journal={Entropy},
  volume={24},
  number={3},
  pages={430},
  year={2022},
  publisher={MDPI}
}

@article{salami2023cryptographic,
  title={Cryptographic algorithms: A review of the literature, weaknesses and open challenges},
  author={Salami, Yashar and Khajevand, Vahid and Zeinali, Esmaeil},
  journal={J. Comput. Robot},
  volume={16},
  number={2},
  pages={46--56},
  year={2023}
}

@article{marcolla2022survey,
  title={Survey on fully homomorphic encryption, theory, and applications},
  author={Marcolla, Chiara and Sucasas, Victor and Manzano, Marc and Bassoli, Riccardo and Fitzek, Frank HP and Aaraj, Najwa},
  journal={Proceedings of the IEEE},
  volume={110},
  number={10},
  pages={1572--1609},
  year={2022},
  publisher={IEEE}
}

@article{gong2024practical,
  title={Practical solutions in fully homomorphic encryption: a survey analyzing existing acceleration methods},
  author={Gong, Yanwei and Chang, Xiaolin and Mi{\v{s}}i{\'c}, Jelena and Mi{\v{s}}i{\'c}, Vojislav B and Wang, Jianhua and Zhu, Haoran},
  journal={Cybersecurity},
  volume={7},
  number={1},
  pages={5},
  year={2024},
  publisher={Springer}
}

@article{peng2019danger,
  title={Danger of using fully homomorphic encryption: A look at Microsoft SEAL},
  author={Peng, Zhiniang},
  journal={arXiv preprint arXiv:1906.07127},
  year={2019}
}

@article{erlingsson2020encode,
  title={Encode, shuffle, analyze privacy revisited: Formalizations and empirical evaluation},
  author={Erlingsson, {\'U}lfar and Feldman, Vitaly and Mironov, Ilya and Raghunathan, Ananth and Song, Shuang and Talwar, Kunal and Thakurta, Abhradeep},
  journal={arXiv preprint arXiv:2001.03618},
  year={2020}
}

@article{huang2022lemmas,
  title={Lemmas of differential privacy},
  author={Huang, Yiyang and Canonne, Cl{\'e}ment L},
  journal={arXiv preprint arXiv:2211.11189},
  year={2022}
}

@inproceedings{geumlek2019profile,
  title={Profile-based privacy for locally private computations},
  author={Geumlek, Joseph and Chaudhuri, Kamalika},
  booktitle={2019 IEEE International Symposium on Information Theory (ISIT)},
  pages={537--541},
  year={2019},
  organization={IEEE}
}

@article{qin2023survey,
  title={A Survey of Local Differential Privacy and Its Variants},
  author={Qin, Likun and Wang, Nan and Qiu, Tianshuo},
  journal={arXiv preprint arXiv:2309.00861},
  year={2023}
}

@inproceedings{wang2017locally,
  title={Locally differentially private protocols for frequency estimation},
  author={Wang, Tianhao and Blocki, Jeremiah and Li, Ninghui and Jha, Somesh},
  booktitle={26th USENIX Security Symposium (USENIX Security 17)},
  pages={729--745},
  year={2017}
}

@article{hosna2022transfer,
  title={Transfer learning: a friendly introduction},
  author={Hosna, Asmaul and Merry, Ethel and Gyalmo, Jigmey and Alom, Zulfikar and Aung, Zeyar and Azim, Mohammad Abdul},
  journal={Journal of Big Data},
  volume={9},
  number={1},
  pages={102},
  year={2022},
  publisher={Springer}
}

@article{araveeporn2021higher,
  title={The higher-order of adaptive lasso and elastic net methods for classification on high dimensional data},
  author={Araveeporn, Autcha},
  journal={Mathematics},
  volume={9},
  number={10},
  pages={1091},
  year={2021},
  publisher={MDPI}
}

@article{rawat2017deep,
  title={Deep convolutional neural networks for image classification: A comprehensive review},
  author={Rawat, Waseem and Wang, Zenghui},
  journal={Neural computation},
  volume={29},
  number={9},
  pages={2352--2449},
  year={2017},
  publisher={MIT Press}
}

@inproceedings{xiao2023geometry,
  title={Geometry of sensitivity: twice sampling and hybrid clipping in differential privacy with optimal gaussian noise and application to deep learning},
  author={Xiao, Hanshen and Wan, Jun and Devadas, Srinivas},
  booktitle={Proceedings of the 2023 ACM SIGSAC Conference on Computer and Communications Security},
  pages={2636--2650},
  year={2023}
}

@misc{pore2023diabetes,
  author = {Nandita Pore},
  title = {Healthcare Diabetes Dataset},
  year = {2023},
  howpublished = {\url{https://www.kaggle.com/datasets/nanditapore/healthcare-diabetes/data}}
}

@misc{elmetwally2023employee,
  author = {Tawfik Elmetwally},
  title = {Employee Dataset},
  year = {2023},
  howpublished = {\url{https://www.kaggle.com/datasets/tawfikelmetwally/employee-dataset/data}}
}

@article{turk2009cardiac,
  title={Cardiac health: primary prevention of heart disease in women},
  author={Turk, Melanie Warziski and Tuite, Patricia K and Burke, Lora E},
  journal={The Nursing clinics of North America},
  volume={44},
  number={3},
  pages={315},
  year={2009}
}

@book{leadbetter2012extremes,
  title={Extremes and related properties of random sequences and processes},
  author={Leadbetter, Malcolm R and Lindgren, Georg and Rootz{\'e}n, Holger},
  year={2012},
  publisher={Springer Science \& Business Media}
}

@inproceedings{meehan2022privacy,
  title={Privacy implications of shuffling},
  author={Meehan, Casey and Chowdhury, Amrita Roy and Chaudhuri, Kamalika and Jha, Somesh},
  booktitle={International Conference on Learning Representations},
  year={2022}
}
\end{document}